\documentclass[10pt,journal,compsoc]{IEEEtran}
\ifCLASSOPTIONcompsoc
  \usepackage[nocompress]{cite}
\else
  \usepackage{cite}
\fi
\ifCLASSINFOpdf
\else
\fi
\usepackage{latexsym}
\usepackage{graphicx}
\usepackage{tikz}
\usepackage{amsthm,algorithmic}
\usepackage{epsfig,psfrag,amstext,amsfonts,amssymb,subeqnarray,color}
\usepackage{hyperref,epstopdf,bm,url,multirow}
\usepackage{amsmath}
\usepackage{array,makecell}  
\usepackage[caption=false,font=footnotesize,subrefformat=parens,labelformat=parens]{subfig}
\usepackage[framemethod=tikz]{mdframed}
\usepackage[lined,boxed,commentsnumbered,ruled]{algorithm2e}
\usepackage{microtype}
\usepackage{booktabs}
\usepackage{ragged2e, xcolor}
\usepackage{cuted}
\usepackage[font=footnotesize]{caption}
\usepackage{datetime}
\usepackage{soul}
\usepackage{threeparttable}
\usepackage{adjustbox}  
\usepackage{cancel,extarrows}
\usepackage{enumitem}  

\newtheorem{theorem}{Theorem}
\newtheorem{lemma}{Lemma}
\newtheorem{corollary}{Corollary}
\newtheorem{remark}{Remark}
\newtheorem{definition}{Definition}
\newtheorem{proposition}{Proposition}
\newtheorem{assumption}{Assumption}

\def\x{{\mathbf x}}
\def\y{{\mathbf y}}
\def\z{{\mathbf z}}
\def\f{{\mathbf{f}}}
\def\u{{\mathbf u}}
\def\m{{\mathbf m}}

\def\vx{{\mathbf X}}
\def\vy{{\mathbf Y}}
\def\vz{{\mathbf Z}}
\def\vf{{\mathbf F}}
\def\vu{{\mathbf U}}
\def\tu{{\tilde{\mathbf u}}}
\def\tf{{\tilde{\mathbf f}}}
\def\tU{{\tilde{\mathbf U}}}
\def\tF{{\tilde{\mathbf F}}}

\def\J{{\mathbf J}}
\def\G{{\mathbb G}}
\def\E{{\mathbb{E}}}

\def\cX{{\cal X}}

\def\cD{{\cal D}}

\def\cN{{\cal N}}

\def\bK{{\boldsymbol{K}}}

\def\bzeta{{\bm{\zeta}}}
\def\btheta{{{\bm{\theta}}}}
\def\bphi{{\bm{\phi}}}

\definecolor{orange}{RGB}{200,0,100}

\begin{document}
%
\title{Towards Flexibility and Interpretability of Gaussian Process State-Space Model}

\author{
Zhidi Lin,~\IEEEmembership{Student Member,~IEEE}, Feng  Yin,~\IEEEmembership{Senior Member,~IEEE}, and  Juan Maro\~{n}as%
\IEEEcompsocitemizethanks{\IEEEcompsocthanksitem Z. Lin is with the School of Science and Engineering, and the Future Network of Intelligence Institute (FNii), The Chinese University of Hong Kong, Shenzhen 518172, China (E-mail: zhidilin@link.cuhk.edu.cn).
\IEEEcompsocthanksitem F. Yin is with the School of Science and Engineering, The Chinese University of Hong Kong, Shenzhen 518172, China. F. Yin is the corresponding author (E-mail: yinfeng@cuhk.edu.cn).
\IEEEcompsocthanksitem J. Maro\~{n}as is with the Machine Learning Group, Universidad Aut\'{o}noma de Madrid, Madrid 28049, Spain, and also with Cognizant AI (E-mail: juan.maronnas@uam.es).
}
}

\IEEEtitleabstractindextext{%
\begin{abstract}
    The Gaussian process state-space model (GPSSM) has garnered considerable attention over the past decade. However, the standard GP with a preliminary kernel, such as the squared exponential kernel or Mat\'{e}rn kernel, that is commonly used in GPSSM studies, limits the model's representation power and substantially restricts its applicability to complex scenarios. To address this issue, we propose a new class of probabilistic state-space models called TGPSSMs, which leverage a parametric normalizing flow to enrich the GP priors in the standard GPSSM, enabling greater flexibility and expressivity. Additionally, we present a scalable variational inference algorithm that offers a flexible and optimal structure for the variational distribution of latent states. The proposed algorithm is interpretable and computationally efficient due to the sparse GP representation and the bijective nature of normalizing flow. Moreover, we incorporate a constrained optimization framework into the algorithm to enhance the state-space representation capabilities and optimize the hyperparameters, leading to superior learning and inference performance. Experimental results on synthetic and real datasets corroborate that the proposed TGPSSM outperforms several state-of-the-art methods. The accompanying source code is available at \url{https://github.com/zhidilin/TGPSSM}.
\end{abstract}

}

\maketitle

\IEEEdisplaynontitleabstractindextext

%
\IEEEpeerreviewmaketitle

\IEEEraisesectionheading{ \section{Introduction} \label{sec:intro} }
%
%
\IEEEPARstart{B}{ecause} of the superiority in modeling dynamical systems, state-space models (SSMs) have been successfully applied in various fields of engineering, statistics, computer science, and economics \cite{sarkka2013bayesian}.
A generic SSM describes the underlying system dynamics and the dependence between the latent states $\x_t \in \mathbb{R}^{d_x}$ and the observations $\y_t \in \mathbb{R}^{d_y}$. Mathematically, it can be written as
\begin{subequations}
	\label{eq:SSM}
	\begin{align}
		& \x_{t+1} = h(\x_t) + \mathbf{v}_t, \\
		& \y_{t} = g(\x_t) + \mathbf{e}_t,	
	\end{align}
\end{subequations}
where $h(\cdot): \mathbb{R}^{d_x} \mapsto \mathbb{R}^{d_x}$ and $g(\cdot): \mathbb{R}^{d_x} \mapsto \mathbb{R}^{d_y}$ are the transition function and emission function, respectively, while $\mathbf{v}_t$ and $\mathbf{e}_t$ are additive noise terms.

The two major tasks of SSMs are learning and inference. The SSM learning task, also known as system identification, is about finding optimal model parameters so that the underlying dynamical system can be accurately represented, while the SSM inference task refers to optimally estimating the latent states of interest using the observed sequential data \cite{sarkka2013bayesian}.
A plethora of learning and inference methods for SSMs have been developed in recent decades.
For example, when the system dynamics are exactly known, the Kalman filter (KF), extended Kalman filter (EKF), unscented Kalman filter (UKF), and particle filter (PF) can be used to estimate the latent states \cite{sarkka2013bayesian}. However, for some complex and harsh scenarios, such as model-based reinforcement learning \cite{yan2020gaussian} and disease epidemic propagation \cite{alaa2019attentive}, the underlying system dynamics are difficult to determine \textit{a priori} \cite{revach2022kalmannet}. Thus, the dynamics need to be learned from the observed noisy measurements, leading to the emergence of data-driven state-space models. The representative classes of data-driven state-space models include deep state-space models (DSSMs) \cite{karl2017deep, krishnan2017structured} 
and Gaussian process state-space models (GPSSMs) \cite{frigola2015bayesian},
which employ deep neural networks \cite{theodoridis2020machine} or Gaussian processes (GPs) \cite{williams2006gaussian} as the core data-driven module to represent the underlying complex system dynamics.

Because of the powerful model representation capability of deep neural networks, DSSMs are particularly suitable for modeling complex and high-dimensional dynamical systems and have shown good results in some real-world applications, such as disease progression analysis \cite{alaa2019attentive}, medication effect analysis \cite{krishnan2017structured}, and nonlinear system identification \cite{gedon2020deep}. However, several vital limitations exist in DSSMs:
1) The transition and/or emission functions in DSSMs are modeled by deep neural networks that typically require a substantial amount of data to tune a large number of model parameters. 
2) DSSMs built on deep neural networks are black-box models, which makes them risky to use in safety-critical applications \cite{kullberg2021online,yin2020fedloc}. 3) 
It is difficult, if not impossible, to give an explicit uncertainty quantification for the transition/emission functions in DSSMs and their predictions \cite{doerr2018probabilistic}.

In contrast, the class of GPSSMs is able to naturally address the abovementioned limitations. First, GPSSMs are popular variants of probabilistic SSMs that are well interpretable and can naturally account for model uncertainty in a fully Bayesian way \cite{frigola2015bayesian}. Second, the nonparametric GP model adopted in a GPSSM, known as a data-efficient model, is directly parameterized by the data and can automatically adapt the model complexity to the observations \cite{theodoridis2020machine}; thus, it is suitable for learning with small datasets \cite{williams2006gaussian}.

Due to the appealing properties of GPSSMs, much promising progress has been made over the past decade \cite{zhao2019cramer}. Early studies on GPSSMs proposed various learning and inference methods \cite{ko2011learning,turner2010state} for robotics control \cite{deisenroth2013gaussian,deisenroth2011robust}, human motion modeling \cite{wang2007gaussian}, and so on. However, these methods either estimate the latent states by performing \textit{maximum a posteriori} estimation, or learn the transition function by assuming the observable latent states, resulting in limited use cases of the GPSSM. The first Bayesian treatment of learning and inference in the GPSSM was proposed in \cite{frigola2013bayesian} using Monte-Carlo sampling methods, which entails high computational complexity and is prohibitive to high-dimensional latent states. To this end, attention was later shifted to variational approximation-based methods \cite{frigola2014variational, frigola2015bayesian, mchutchon2015nonlinear,eleftheriadis2017identification, doerr2018probabilistic,ialongo2019overcoming,curi2020structured,lindinger2022laplace,lin2022output}, which can be divided into two classes: the mean-field (MF) class \cite{frigola2014variational, frigola2015bayesian, mchutchon2015nonlinear,eleftheriadis2017identification} and the nonmean-field (NMF) class \cite{doerr2018probabilistic,ialongo2019overcoming,curi2020structured,lindinger2022laplace,lin2022output}, according to whether the independence assumption is made or not between the assumed GP-based transition function and the latent states in designing the variational distribution. More specifically, the first MF variational algorithm that integrated a particle filter was proposed in \cite{frigola2014variational}. The authors in \cite{eleftheriadis2017identification} assumed a linear Markov Gaussian variational distribution for the latent states and additionally employed an inference network to overcome the issue of linear growth of the number of variational parameters over time.
The first NMF variational algorithm was proposed in \cite{doerr2018probabilistic}, where the posterior of latent states is directly approximated using the prior distribution. Later, different nonlinear and parametric Markov Gaussian distributions were employed to approximate the posterior of the latent states, hence leading to different NMF algorithms \cite{ialongo2019overcoming, curi2020structured, lindinger2022laplace,lin2022output}. 
In general, the MF class is simpler and less computationally expensive than the NMF class. Recently, online learning schemes for GPSSMs that utilize sampling methods were proposed \cite{liu2020gpssm,liu2022inference,zhao2022streaming}. However, the computational complexity of these schemes remains as an issue, and they are still unable to effectively handle high-dimensional latent states \cite{dowlingreal}.
While the design of learning and inference algorithms plays a crucial role in GPSSM, the limited modeling capacity of GPSSM has been overlooked for years. Specifically, the current GPSSM employs standard GPs with elementary kernels, such as squared exponential (SE) kernel or Matern kernel  \cite{williams2006gaussian}, which prevents it from modeling complex and harsh nonlinear dynamical systems (e.g., systems with sharp transition dynamics). Consequently, this paper aims to enhance the modeling capacity of GPSSM to enable it to model complex dynamical systems.
 There are two paths to improve the model representation power: (1) transforming the standard GP with elementary kernels into a more flexible stochastic process \cite{damianou2013deep, wilson2010copula, rios2020contributions, maronas2021transforming, maronas2022efficient}, and (2) using universal/optimal kernels in the GP model \cite{wilson2013gaussian, yin2020linear, suwandi2022gaussian, wilson2016deep, dai2020interpretable}. While the second path enhances the flexibility of the GP by adopting an optimal kernel function, the Gaussianity in the GP may still be limited and unsuitable for modeling complex systems \cite{damianou2013deep}. Furthermore, optimal kernels may require a large number of kernel hyperparameters, making the model training/optimization process computationally expensive and sensitive to initialization \cite{wilson2013gaussian,dai2020interpretable}. Thus, this paper focuses on the first path, which leverages normalizing flow techniques \cite{papamakarios2021normalizing} to transform the standard GP prior into a more flexible stochastic process. The newly introduced process maintains the elegance and interpretability of GPs while requiring fewer model hyperparameters to tune compared to the GPs using an optimal kernel. As a result, the model representation power of the GPSSM is improved, thereby enhancing the learning and inference performance in harsh and complex scenarios. 
The main contributions of this paper are detailed as follows:
\begin{itemize}
    \item We propose a flexible and expressive class of probabilistic SSMs, called transformed Gaussian process state-space models (TGPSSMs), by leveraging the normalizing flow technique \cite{kobyzev2020normalizing,papamakarios2021normalizing} to transform the GP priors in GPSSMs. Theoretically, the proposed TGPSSM can be regarded as placing a flexible stochastic process prior over the traditional SSMs, which provides a unified framework for standard probabilistic SSMs, including the GPSSM.
    \item We design a scalable learning and inference algorithm for the TGPSSM based on the variational inference framework. Under theoretical analysis, the variational distribution of latent states in our algorithm is flexible and optimal in structure. Furthermore, due to the sparse representation of GP \cite{titsias2009variational} and the bijective nature of normalizing flow, the new proposed variational algorithm is both interpretable and minimally increases computational complexity.
    \item  To further enhance the learning and inference performance of the proposed algorithm, we then utilize a constrained optimization framework to enable learning an informed latent state space representation.  This framework allows automatic optimization of the hyperparameters in the objective function without requiring manual setting as in $\beta$-VAE \cite{higgins2016beta}.
\item Extensive numerical tests based on various synthesized- and real-world datasets verify that the proposed flexible TGPSSM empowered by the variational learning algorithm provides improved model learning and inference performance compared with several competing methods.
\end{itemize}

The structure of this paper is organized as follows. Section \ref{sec:preliminaries} presents some preliminaries related to GPSSMs. In Section \ref{sec:proposed-model}, we introduce a novel flexible and expressive TGPSSM. The proposed variational learning and inference algorithms are described in Section \ref{sec:vari_infer}. Section \ref{sec:experimental-results} presents the experimental results. We conclude this paper in Section \ref{sec:conclusion}. Technical proofs and derivations are provided in Section \ref{sec:proofs}, while additional experimental details and supportive results are included in the Appendix as supplementary materials.

\section{Preliminaries}
\label{sec:preliminaries}
This section provides some preliminaries for GPSSM. Specifically, Section \ref{subsec:GP} briefly reviews the Gaussian process for machine learning, especially the regression task. Section \ref{subsec:GPSSM} introduces the GPSSM and its fundamental properties. 

\subsection{Gaussian Process (GP)}\label{subsec:GP}
GP defines a collection of random variables indexed by $\bm{x} \in \cX$, such that any finite collection of these variables follows a joint Gaussian distribution \cite{williams2006gaussian}.  Mathematically,  a real scalar-valued GP $f(\bm{x})$ can be written as 
\begin{equation}
	f(\bm{x}) \sim \mathcal{GP}\left(\mu(\bm{x}), \ k(\bm{x}, \bm{x}^\prime);  \ \bm{\theta}_{gp}\right),
\end{equation}
where $\mu({\bm{x}})$ is a mean function typically set to zero in practice, and $k(\bm{x}, \bm{x}^\prime)$ is the covariance function, also known as (a.k.a.) kernel function, which is interpretable and can provide insights about the nature of the underlying function \cite{williams2006gaussian}, and $\bm{\theta}_{gp}$ is a set of hyperparameters that needs to be tuned for model selection. 
Let us consider a general regression model, 
\begin{equation}
	y = f(\bm{x}) + {e},  \quad 	{e} \sim \cN(0, \sigma_{e}^2), \quad y \in \mathbb{R}.
	\label{eq:reg_model}
\end{equation}
By placing a GP prior over the function $f(\cdot): \cX \mapsto \mathbb{R}$, we get the salient Gaussian process regression (GPR) model. The task in GPR model is to infer the mapping function $f(\cdot)$ using an observed dataset  $\mathcal{D}\triangleq\{\bm{x}_i, y_i\}_{i = 1}^{n} \triangleq \{\bm{X}, \bm{y}\}$ consisting of $n$ samples, or alternatively input-output pairs. Conditioning on the observed data, the posterior distribution of the mapping function,  $p(f(\bm{x}_*) \vert \bm{x}_*, \cD)$, at any test input $\bm{x}_* \in \cX$, is Gaussian, fully characterized by the posterior mean $\xi$ and the posterior variance $\Xi$.  Concretely, 
\begin{subequations}
    \label{eq:GP_posterior}
    \begin{align}
        &\!\! \xi(\bm{x}_*)  \!=\! \bm{K}_{\bm{x}_*, \bm{X}} \left(\boldsymbol{K}_{\bm{X},\bm{X}}+ \sigma_{e}^{2} \boldsymbol{I}_{n} \right)^{-1} { \bm{y}}, 
        \label{eq:post_mean}\\
        &\!\! \Xi(\bm{x}_*) \!=\! k(\bm{x}_*, \bm{x}_*)  \!-\! \bm{K}_{\bm{x}_*, \bm{X}} \left( \boldsymbol{K}_{\bm{X},\bm{X}}\!+\! \sigma_{e}^{2} \boldsymbol{I}_{n} \right)^{-1} \bm{K}_{\bm{x}_*, \bm{X}}^\top,
        \label{eq:post_cov}
    \end{align}
\end{subequations}
where $\boldsymbol{K}_{\bm{X},\bm{X}}$ denotes the covariance matrix evaluated on the training input $\bm{X}$, and each entry is $[\boldsymbol{K}_{\bm{X},\bm{X}}]_{i,j} = k({\bm{x}}_i, \bm{x}_j)$; $\bm{K}_{\bm{x}_*, \bm{X}}$ denotes the cross covariance matrix between the test input $\bm{x}_*$ and the training input $\bm{X}$;  
the zero-mean GP prior is assumed here and will be used in the rest of this paper if there is no further specification.  
Note that the posterior distribution $p(f(\bm{x}_*) \vert \bm{x}_*, \cD)$ gives not only a point estimate, i.e., the posterior mean, but also an uncertainty region of such estimate quantified by the posterior variance.
It should also be noted that here we denote the variables in the GPR model using mathematical mode italics, such as $\bm{x}_i$ and $y_i$; these variables should not be confused with the latent state $\x_t$ and observation $\y_t$ in SSM (cf. Eq.~(\ref{eq:SSM})).


\subsection{Gaussian Process State-Space Model (GPSSM)}\label{subsec:GPSSM}
Placing GP priors over both  transition function $h(\cdot)$ and emission function  $g(\cdot)$ in SSM (cf. Eq.~(\ref{eq:SSM})) leads to the well-known GPSSM.  However, such GPSSM \textit{with transition and emission GPs} incurs severe nonidentifiability issue between $h(\cdot)$ and $g(\cdot)$ \cite{frigola2014variational}.  To address this issue,  GPSSM \textit{with GP transition and parametric emission} is considered in the literature, as it keeps the same model capacity as the original one only at the cost of introducing higher dimensional latent states \cite{frigola2015bayesian}. The result is summarized in the following theorem.
\begin{theorem}
	\label{theorem:gpssmidentifi}
	For the GPSSM with transition and emission GPs,  
	\begin{subequations}
		\label{eq:gpssm_trans_emis}
		\begin{align}
			& \x_{t+1} = h(\x_t) + \mathbf{v}_t,   \quad 	 h(\cdot) \sim \mathcal{GP},\\
			& \y_{t} = g(\x_t) + \mathbf{e}_t,   \qquad ~ 	 g(\cdot)  \sim \mathcal{GP},
		\end{align}
	\end{subequations}
by defining the augmented state $\bar{\x}_{t} \triangleq [\x_{t+1}, g(\x_{t})]^\top$, Gaussian process $f(\bar{\x}_t ) \triangleq \left[h(\x_{t+1}), g(\x_{t+1}) \right]^\top$, and the augmented process noise $\mathbf{w}_t \triangleq [\mathbf{v}_{t+1}, \mathbf{0}]^\top$,  the original model in Eq.~(\ref{eq:gpssm_trans_emis}) can be reformulated to  a GPSSM with a GP transition and a simple parametric emission, 
\begin{subequations}
	\label{eq:gpssm_trans_emis2}
	\begin{align}
		& \bar{\x}_{t+1} = f(\bar{\x}_t) + \mathbf{w}_t,    	\qquad    f(\cdot) \sim \mathcal{GP},\\
		& \y_{t} = [\bm{0}, \bm{I}] \ \bar{\x}_{t} + \mathbf{e}_t,   
	\end{align} 
 \label{eq:gpssm_original}
\end{subequations}
so as to eliminate/alleviate the severe nonidentifiability issue. 
\end{theorem}
Interested reader can refer to \cite{frigola2015bayesian} (Section 3.2.1) for more details. 
Therefore, in keeping with the previous literature and without loss of generality, we mainly consider the GPSSM with GP transition and parametric emission in this paper. More specifically, the GPSSM we consider is depicted in Fig.~\ref{fig:graphical_model} and expressed by the following equations:
\begin{subequations}
	\label{eq:gpssm}
	\begin{gather}
		 f(\cdot)  \sim \mathcal{G} \mathcal{P}\left(\mu(\cdot), k(\cdot, \cdot); \btheta_{gp} \right) \\
		 \mathbf{x}_{0} \sim p\left(\mathbf{x}_{0} \right) \\
		 {\f}_{t} =f\left(\mathbf{x}_{t-1}\right) \\
		 \mathbf{x}_{t} \mid {\f}_{t}  \sim \mathcal{N}\left( \x_t \mid {\f}_{t}, \mathbf{Q}\right) \\
		 \mathbf{y}_{t} \mid  \mathbf{x}_{t} \sim \cN \left(\mathbf{y}_{t} \mid \bm{C} \mathbf{x}_{t}, \boldsymbol{R}\right)
	\end{gather}
\end{subequations}
where the latent states form a Markov chain, that is, for any time instance $t$, the future state $\x_{t+1}$ is generated by conditioning on only $\x_t$ and the GP transition $f(\cdot)$.  According to Theorem \ref{theorem:gpssmidentifi},  the parametric emission model is assumed to be known and restricted to be a linear mapping with a coefficient matrix $\bm{C} \in \mathbb{R}^{d_y \times d_x}$.  Both the state transitions and observations are corrupted by zero-mean Gaussian noise with covariance matrices $\bm{Q}$ and $\bm{R}$, respectively.

\begin{figure}[t!]
	\centering
	\footnotesize
	\begin{tikzpicture}[align = center, latent/.style={circle, draw, text width = 0.45cm}, observed/.style={circle, draw, fill=gray!20, text width = 0.45cm}, transparent/.style={circle, text width = 0.45cm}, node distance=1.2cm]
		\node[latent](x0) {${\x}_0$};
		\node[latent, right of=x0, node distance=1.4cm](x1) {${\x}_{1}$};
		\node[transparent, right of=x1](x2) {$\cdots$};
		\node[latent, right of=x2](xt-1) {$\!\!{\x}_{t-1}\!\!$};
		\node[latent, right of=xt-1, node distance=1.4cm](xt) {${\x}_{t}$};
		\node[transparent, right of=xt](xinf) {$\cdots$};
		\node[transparent, above of=x0](f0) {$\cdots$};
		\node[latent, above of=x1](f1) {${\f}_{1}$};
		\node[transparent, right of=f1](f2) {$\cdots$};
		\node[latent, above of=xt-1](ft-1) {$\!\!{\f}_{t-1}\!\!$};
		\node[latent, above of=xt](ft) {${\f}_{t}$};
		\node[transparent, right of=ft](finf) {$\cdots$};
		\node[observed, below of=x1](y1) {${\y}_{1}$};
		\node[transparent, below of=x2](y2) {$\cdots$};
		\node[observed, below of=xt-1](yt-1) {$\!\!{\y}_{t-1}\!\!$};
		\node[observed, below of=xt](yt) {${\y}_{t}$};
		\node[transparent, right of=yt](yinf) {$\cdots$};
		\draw[-latex] (x0) -- (f1);
		\draw[-latex] (f1) -- (x1);
		\draw[-latex] (x1) -- (f2);
		\draw[-latex] (ft-1) -- (xt-1);
		\draw[-latex] (xt-1) -- (ft);
		\draw[-latex] (x2) -- (ft-1);
		\draw[-latex] (ft) -- (xt);
		\draw[-latex] (xt) -- (finf);
		\draw[-latex] (x1) -- (y1);
		\draw[-latex] (xt-1) -- (yt-1);
		\draw[-latex] (xt) -- (yt);
		\draw[ultra thick]
		(f0) -- (f1)
		(f1) -- (f2)
		(f2) -- (ft-1)
		(ft-1) -- (ft)
		(ft) -- (finf)
		;
	\end{tikzpicture}
	\caption{Graphical model of a GPSSM with GP transition and parametric emission. The white circles indicate that the variables are latent, while the gray circles represent the observable variables. The thick horizontal bar represents a set of fully connected nodes, i.e., the GP.}
	\label{fig:graphical_model}
\end{figure}
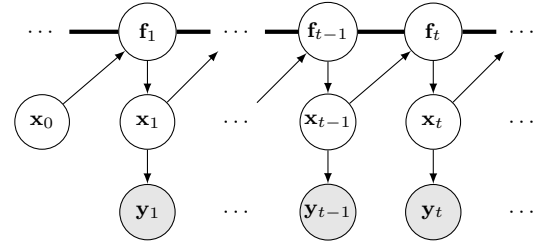

If the state is multidimensional, i.e., $d_x>1$,  the system transition function $f(\cdot): \mathbb{R}^{d_x} \mapsto  \mathbb{R}^{d_x}$, is typically modeled by using $d_x$ mutually independent GPs. More concretely,  each output dimension-specific function, $f_d(\cdot): \mathbb{R}^{d_x} \mapsto \mathbb{R}$, is independently modeled by a scalar-valued GP, and we denote
\begin{equation}
	\f_t = f(\x_{t-1}) \triangleq \{f_d(\x_{t-1})\}_{d=1}^{d_x}, 
	\label{eq:multivariate_GP}
\end{equation}
where each independent GP, $f_d(\cdot)$, has its own mean function, $\mu_d(\cdot)$,  and kernel function, $k_d(\cdot, \cdot)$. 
Finally, note that the GPSSM depicted in Fig.~\ref{fig:graphical_model} can be easily extended to a control system with deterministic control input $\bm{c}_t \in \mathbb{R}^{d_c}$ through an augmented latent state $[\x_t, \bm{c}_t] \in \mathbb{R}^{d_x + d_c}$, but for notation brevity, we omit $\bm{c}_t$ throughout this paper. 

To ease our notation in the rest of discussions, we introduce the following short-hand notations. Let an observation sequence of length $T$ be $\vy \triangleq \y_{1:T} =  \{\y_t\}_{t=1}^T$, latent function variables $\vf \triangleq \f_{1:T}= \{\f_t\}_{t=1}^T$, and  latent states $\vx \triangleq \x_{0:T}= \{\x_t\}_{t=0}^T$. Based on the aforementioned model settings, the joint density function of the GPSSM depicted in Fig. \ref{fig:graphical_model} can be written as:
\begin{equation}
	\begin{aligned}
		p(\vf, {\vx}, \vy \vert \bm{\theta}) = p(\mathbf{x}_{0}) p(\f_{1:T}) \prod_{t=1}^{T} p(\mathbf{y}_{t} \vert \mathbf{x}_{t})  p(\mathbf{x}_{t} \vert {\f}_{t}), 
	\end{aligned}
	\label{eq:joint_density}
\end{equation}
where 
$p(\f_{1:T}) \!\!=\!\! p(f(\x_{0:T-1})) \!\!=\!\! \prod_{t=1}^{T} p(\f_t \vert \f_{1:t-1}, \x_{0:t-1})$ corresponds to a finite dimensional GP distribution \cite{frigola2015bayesian}. The model parameters $\bm{\theta}$ includes the noise and GP hyper-parameters, i.e., $\bm{\theta} \!\!=\!\! [\bm{Q}, \bm{R}, \bm{\theta}_{gp}]$. One of the most challenging tasks in GPSSM is to learn $\btheta$, and simultaneously infer the latent states of interest, which usually involves the marginal distribution $p(\vy|\btheta)$. Due to the nonlinearity of the GPs, however, a closed-form analytical solution for $p(\vy|\btheta)$ is unavailable. Thus, approximation methods need to be employed. We shall defer the discussion of this until Section \ref{sec:vari_infer}.

\section{Proposed Model}
\label{sec:proposed-model}
This section aims to tackle the issue of limited model representation power mentioned in Section \ref{sec:intro}, thereby improving the learning and inference performance of the GPSSM. To accomplish this, we propose a flexible probabilistic SSM that leverages the parametric normalizing flow technique \cite{papamakarios2021normalizing}. 
More specifically, Section \ref{subsec:TGP_prior} introduces a more flexible function prior, namely the transformed Gaussian process (TGP) \cite{maronas2021transforming}, by exploiting the parametric normalizing flow. Different normalizing flows for TGP construction under practical usage considerations are detailed in Section \ref{subsec:NF_used}.  Lastly,  in Section \ref{subsec:tgpssm}, we introduce our proposed probabilistic SSM that uses the TGP.


\subsection{Transformed Gaussian Process}
\label{subsec:TGP_prior}
 Normalizing flow was originally proposed to transform a simple random variable into a more complex one by applying a sequence of invertible and differentiable  transformations (i.e., the \textit{diffeomorphisms}) \cite{papamakarios2021normalizing}. Specifically, let $\x$ be a $d_x$-dimensional continuous random vector, and $p(\x)$ be the corresponding probability density. Normalizing flow can help construct a desired, often more complex and possibly multi-modal distribution by pushing $\x$  through a series of transformations, $\mathbb{G}_{\btheta_{F}}(\cdot)=\mathbb{G}_{\theta_{0}} \circ \mathbb{G}_{\theta_{1}} \circ \cdots \circ \mathbb{G}_{\theta_{J-1}}$, i.e., 
\begin{equation}
	\tilde{\x} = \mathbb{G}_{\theta_{J-1}}\left( \G_{\theta_{J-2}}\left( ... \G_{\theta_0}(\x)...  \right) \right),
\end{equation}
where 
the set of transformations, $  \{ \G_{\theta_j}(\cdot): \mathbb{R}^{d_x} \mapsto \mathbb{R}^{d_x} \}_{j=0}^{J-1}$, parameterized by $\bm{\theta}_F \triangleq [\theta_{0}, \theta_1, ..., \theta_{J-1}]$, has to be invertible and differentiable \cite{kobyzev2020normalizing}. Under these conditions, the probability density of the induced random vector, $\pi(\tilde{\x})$,  is well-defined and can be obtained by the  ``change of variables'' formula, yielding:
\begin{equation}
		\pi(\tilde{\x})  \!=\!p({\x}) \prod_{j=1}^{J-1}\left|\operatorname{det} \frac{\partial \ \mathbb{G}_{\theta_{j}}( \G_{\theta_{j-1}}( ... \G_{\theta_0}(\x)... ) )}{\partial \  \G_{\theta_{j-1}} ( ... \G_{\theta_0}(\x)... )  }\right|^{-1}.
	\label{eq:nf}
\end{equation}
We refer the readers to \cite{papamakarios2021normalizing,kobyzev2020normalizing} for more details.

For the complex system dynamics that a standard GP equipped with elementary kernels cannot model, a good way is to transform the GP prior. Specifically, by applying the same idea as the normalizing flow on random variables, we transform the standard GP,  $f(\cdot)$, to get a more flexible and expressive random process, $\tilde{f}(\cdot)$, namely the TGP prior \cite{maronas2021transforming}, defined as follows.
\begin{definition}[Transformed Gaussian process (TGP)]
A TGP, $\tilde{f}(\cdot)$, is a collection of random variables, such that any finite collection, ${\tf}_{1:T} \triangleq \tilde{f}(\x_{0:T-1}),
T \in \mathbb{N}$, has joint distribution defined by 
\begin{equation}
	\label{eq:pf_TGP}
	p(\tilde{\f}_{1:T}) = p(\tilde{f}(\x_{0:T-1}) ) =  p\left(\f_{1:T}\right)  \J_\f,  
\end{equation} 
where
\begin{equation}
	 \J_\f \triangleq  \prod_{j=1}^{J-1}\left|\operatorname{det} \frac{\partial \ \mathbb{G}_{\theta_{j}}\left( {\G_{\theta_{j-1}}\left( ... \G_{\theta_0}(\f_{1:T})... \right)} \right)}{\partial \  {\G_{\theta_{j-1}}\left( ... \G_{\theta_0}(\f_{1:T})... \right)}  }\right|^{-1},
\end{equation}
and $\{ \G_{\theta_j}(\cdot): \mathbb{R}^{Td_x} \mapsto \mathbb{R}^{Td_x} \}_{j=0}^{J-1}$ parameterized by $\btheta_F$ is a set of invertible and differentiable mapping functions such that the induced joint distribution, $p(\tilde{\f}_{1:T})$, satisfies \textit{Kolmogorov's consistency theorem} \cite{tao2011introduction}.
Samples from this joint distribution are obtained by the following generative equations:
\begin{equation}
	\tf_{1:T} = \tilde{f}(\x_{0:T-1}) =  \G_{\btheta_{F}}({\f}_{1:T}),  \ \  \f_{1:T} = f(\x_{0:T-1}), 
\label{eq:marg_trans}
\end{equation} where $\mathbb{G}_{\btheta_{F}}(\cdot)=\mathbb{G}_{\theta_{0}} \circ \mathbb{G}_{\theta_{1}} \circ \cdots \circ \mathbb{G}_{\theta_{J-1}},$
and $\f_{1:T}$ is finite collection of Gaussian random variables from the standard GP, $f(\cdot)$, evaluated at the input $\x_{0:T-1}$ (see Eq.~(\ref{eq:multivariate_GP})). 
%
%
%
%
\end{definition}
%
%
%
%
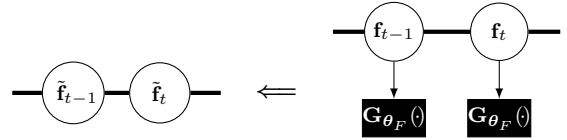
\begin{figure}[t!]
	\centering
	\footnotesize
	\begin{tikzpicture}[align = center, latent/.style={circle, draw, text width = 0.45cm}, observed/.style={circle, draw, fill=gray!20, text width = 0.45cm}, transparent/.style={circle, text width = 0.45cm}, transform/.style={rectangle, draw, fill=black, text width = 0.65cm}, node distance=1.15cm]
		\node[transparent](f2) {};
		\node[latent, right of=f2](ft-1) {$\!\!{\f}_{t-1}\!\!$};
		\node[latent, right of=ft-1, node distance=1.4cm](ft) {${\f}_{t}$};
		\node[transparent, right of=ft](finf) {};
		\node[transparent, below of=f2](fJ2) {};
		\node[transform, below of=ft-1](fJt-1) {$\!\!\textcolor{white}{\mathbf{G}_{\btheta_F}(\!\cdot\!)}\!\!$};
		\node[transform, below of=ft](fJt) {$\!\!\textcolor{white}{\mathbf{G}_{\btheta_F}(\!\cdot\!)}\!\!$};
		\node[transparent, below left of=f2](tf1) {};
		\node[latent, left of=tf1](tft) {${\tf}_{t}$};
		\node[transparent, right of=tft,  node distance=1.5cm](tf0) {$\bm{\Longleftarrow}$};
		\node[latent, left of=tft](tft-1) {${\tf}_{t-1}$};
		\node[transparent, left of=tft-1](tft-2) {};
		
		\draw[-latex] (ft-1) -- (fJt-1);
		\draw[-latex] (ft) -- (fJt);
		\draw[ultra thick]
		(tft-2) -- (tft-1)
		(tft-1) -- (tft)
		(tft) -- (tf1)
		(f2) -- (ft-1)
		(ft-1) -- (ft)
		(ft) -- (finf)	;
	\end{tikzpicture}
	\caption{(\textbf{Left}) Generic representation of TGP prior. (\textbf{Right}) GP prior is transformed by marginal flow, leading to a TGP prior.  The black square blocks represent the (coordinate-wise) marginal flow. 
 }
	\label{fig:TGP}
\end{figure}
The TGP definition implies that with the mapping $\G_{\btheta_F}(\cdot)$, any finite dimensional distribution $p(\tilde{f}(\x_{0:T-1}))$ can be non-Gaussian, thus the obtained TGP is a more flexible and expressive function prior. 
If the mapping $\G_{\btheta_{F}}(\cdot)$ in Eq.~(\ref{eq:marg_trans}) transforms the standard GP coordinate-wisely, i.e., $\G_{\btheta_{F}}(\cdot): \mathbb{R}^{d_x} \mapsto \mathbb{R}^{d_x}$, and
\begin{equation}
    \tf_{t} = \tilde{f}(\x_{t-1}) =  \G_{\btheta_{F}}({\f}_{t}),  \ \  \f_{t} = f(\x_{t-1}),  \ t = 1, 2,..., T,
    \label{eq:marignal_flow}
\end{equation}
 it is known as the \textit{marginal flow} \cite{rios2020contributions,maronas2021transforming}.  
The graphical representations of the TGP prior and marginal flow are depicted in Figure~\ref{fig:TGP}.
It has been formally proven in \cite{rios2020contributions,wilson2010copula} that a scalar-valued stochastic process transformed by marginal flow induces a valid stochastic process.  
This result has recently been applied to multi-class classification problems \cite{maronas2022efficient,maronas2021transforming}. We extend the theoretical result for the vector-valued GP and summarize it in the following corollary.
\begin{corollary}
\label{corollary1}
Given a marginal flow, $\mathbb{G}_{\btheta_{F}}(\cdot)$,  and a multidimensional GP \cite{chen2023remarks}, $f(\cdot)$, with input space $\mathcal{X}$,  the transformed process, denoted by $\tilde{f} \triangleq \G_{\btheta_{F}}(f)$, is a valid multivariate stochastic process in the same input space.
\end{corollary}
\begin{proof}
	 First, since $\G_{\btheta_{F}}(\cdot)$ is invertible and differentiable, the push-forward measure is well-defined.  Then,  because the marginal flows are coordinate-wise mappings, it is not difficult to verify that the joint distribution of any finite collection of variables, $\{ \tf_1, \tf_2, \ldots, \tf_T \},  T \in \mathbb{N}$, from $\tilde{f}$, satisfies the consistency conditions required by the \textit{Kolmogorov's consistency theorem} \cite{tao2011introduction}, which completes the proof.  
\end{proof}

It should be noted that by using marginal flows, the resulting finite dimensional distribution, $p(\tilde{\f}_{1:T})$, can have non-Gaussian marginals, while still sharing the same Copula as $p({\f}_{1:T})$, according to Sklar's theorem \cite{wilson2010copula}. Nevertheless, for our purposes, the increased expressiveness obtained by modifying the marginals is enough, as we prefer to benefit from a more efficient sparse variational inference algorithm when $\G_{\btheta_{F}}(\cdot)$ is marginal flow \cite{maronas2021transforming}. It is also noteworthy that the idea of using transformations to increase the flexibility of GP models is natural and straightforward. The TGP prior leveraging parametric normalizing flows was recently proposed in \cite{maronas2021transforming} and is in line with several existing works \cite{wauthier2010heavy,wilson2010copula,snelson2003warped,rios2019compositionally,rios2020contributions,maronas2022efficient}, such as  Copula processes \cite{wilson2010copula} and warped GPs \cite{snelson2003warped,rios2019compositionally,rios2020contributions}. However, these works mainly focus on supervised learning tasks instead of learning and inference in SSMs with latent states.  If the transformation is constructed by multi-layered GPs, then this leads to the so-called deep Gaussian process (DGP) models \cite{damianou2013deep}.  However, training the DGP is essentially challenging because of the high computational complexity involved during the training phase \cite{maronas2021transforming,maronas2022efficient,rios2020contributions}.  In this paper, we concentrate on exploiting parametric marginal flows to increase the model representation power in the original GPSSM without substantially increasing the overall computational complexity. 

\subsection{Construction and Usage of Normalizing Flow}
\label{subsec:NF_used}

This subsection presents practical guidelines for the construction and use of normalizing flows in the TGP model.
In essence, various types of normalizing flows can be employed as marginal flows, ranging from basic and interpretable elementary flows to more sophisticated data-driven ones developed in recent years \cite{kobyzev2020normalizing, papamakarios2021normalizing}. However, in cases where prior knowledge regarding the underlying system dynamics is available, elementary flows are typically preferred due to their superior interpretability and lower number of model parameters relative to their data-driven counterparts. For example, if the latent states of the underlying dynamical system exhibit significant dispersion and form a heavy-tailed distribution statistically, then Sinh-Arcsinh transformations may be employed to control the mean, variance, asymmetry, and kurtosis of the TGP priors while modeling the system dynamics \cite{jones2009sinh}. Another example is that by combining two affine-logarithmic transformations, bounded latent states can be effectively handled \cite{rios2020contributions}.
A simple and widely-used example for elementary flow compositions is given by stacking $J$ layers of Sinh-Arcsinh-Linear (SAL) flow \cite{maronas2021transforming,rios2019compositionally}:
\begin{equation}
    \label{eq:SAL_flows}
    \G_{\theta_j}(\cdot) = d_j \sinh \left(b_j \operatorname{arcsinh}(\cdot )-a_j \right)+c_j, 
\end{equation}
where $\theta_{j} \!=\! [a_j, b_j, c_j, d_j], j \!=\! 0, 1, ..., J\!-\!1$. We summarize these interpretable elementary flows along with their possible compositions used in the literature in Table \ref{tab:flows}, Appendix~\ref{subsec:appx_flow_table}. 


However, the elementary flows and their compositions may still be insufficient, especially when modeling the dynamical system in extremely harsh and complex scenarios. To this end, the normalizing flows developed recently can be integrated with GP prior to increase the model flexibility, such as the real-valued non-volume preserving (RealNVP) flow \cite{dinh2017density}  and continual normalizing flow \cite{chen2018neural}.   In this paper, we utilize the RealNVP flow as an example to model multidimensional TGP because of its universal approximation capability and competitive performance compared to other competitors in terms of the number of model parameters and the training speed  \cite{BondTaylor2021}.  
The transformations, $\G_{\theta_{j}}(\cdot), j =0, 1, ..., J-1,$ in RealNVP are constructed by coupling layers. Specifically, given a $d_x$-dimensional function output, $\f_t = [\f_{t,1}, \f_{t,2},..., \f_{t,d_x}]^\top \in \mathbb{R}^{d_x}$, and $d< d_x$, the first coupling layer, $\G_{\theta_{0}}(\f_t)$, outputs $\f_t^{1} = [ \f_{t, 1}^{1}, \f_{t, 2}^{1},..., \f_{t, d_x}^{1}]^\top$, following the equations
\begin{subequations}
	\begin{align}
		&  \f_{t, 1:d}^{1} = 	\f_{t, 1:d},\\
		&  \f_{t, d+1:d_x}^{1} = 	\f_{t, d+1:d_x}  \odot \exp( s(\f_{t, 1:d}) ) + r(\f_{t, 1:d})), 
	\end{align}
	\end{subequations}
where $s(\cdot)$ and $r(\cdot)$ are arbitrary complex mappings $\mathbb{R}^{d} \mapsto \mathbb{R}^{d_x - d}$,  which can be modeled by neural networks, and $\odot$ represents element-wise product. Then by stacking coupling layers in an alternating fashion, elements that are unchanged in the last coupling layer are updated in the next \cite{dinh2017density}. 
Note that if the elements of the function output, $\{\f_{t,1}, \f_{t,2}, ..., \f_{t,d_x}\}$, are independent, the outputs of the RealNVP flow, e.g., $\{\f_{t, 1}^{1}, \f_{t, 2}^{1},..., \f_{t, d_x}^{1}\}$, are dependent because of the transformation of the $J$ coupling layers.


\subsection{Transformed Gaussian Process State-Space Model}
\label{subsec:tgpssm} 
Placing the TGP prior $\tilde{f}(\cdot)$ defined in Section \ref{subsec:TGP_prior}  over the transition function of SSM gives rise to the transformed Gaussian process state-space model (TGPSSM). The corresponding graphical model is shown in Fig.~\ref{fig:TGPSSM_model_inducing}. Mathematically, the model can be expressed by the following equations:
\begin{subequations}
	\label{eq:tgpssm}
	\begin{align}
	    & \tilde{f}(\cdot) = \G_{\btheta_{F}}(f), \quad  f(\cdot)  \sim \mathcal{G} \mathcal{P}\left(\mu(\cdot), k(\cdot, \cdot) ; \btheta_{gp} \right)  \\
		&  \mathbf{x}_{0} \sim p\left(\mathbf{x}_{0}\right) \\
		&  \f_{t} = f(\x_{t-1})  \\
 		& \tilde{\f}_{t} = \G_{\btheta_F}(\f_{t}) =\tilde{f}\left(\mathbf{x}_{t-1}\right)  \\
		& \mathbf{x}_{t} \mid \tilde{\f}_{t}  \sim \mathcal{N}\left( \x_t \mid \tilde{\f}_{t}, \mathbf{Q}\right) \\
		&  \mathbf{y}_{t} \mid  \mathbf{x}_{t} \sim \cN \left(\mathbf{y}_{t} \mid \bm{C} \mathbf{x}_{t}, \boldsymbol{R}\right)
	\end{align}
\end{subequations}
where $\mathbb{G}_{\btheta_F}(\cdot)$ is a marginal flow described in Section \ref{subsec:TGP_prior}. Given an observation sequence of size $T$,  
the corresponding TGP function values are denoted as ${\tF} \triangleq {\tf}_{1:T}=\{\tilde{f}(\x_t)\}_{t=0}^{T-1}$, and 
the joint probability density function of the TGPSSM is
\begin{equation}
    p(\tF, \vx, \vy \vert \bm{\theta}) =p(\mathbf{x}_{0}) p(\tilde{\f}_{1:T} )  \prod_{t=1}^{T} p(\mathbf{x}_{t} \vert \tilde{\f}_{t} )   p(\mathbf{y}_{t} \vert \mathbf{x}_{t}),
    \label{eq:joint_density_TGPSSM}
\end{equation}
where $p(\tilde{\f}_{1:T})$ corresponds to the finite dimensional TGP distribution, and $\btheta \!\! = \!\! \{ \bm{\theta}_{gp}, \btheta_{F}, \bm{Q}, \bm{R}\}$ is the set of the model parameters.
Sampling from the TGPSSM prior is similar to sampling from GPSSM prior \cite{frigola2015bayesian}, except for an extra transformation step that transforms the sampled latent function value $\f_{t}$ into ${\tilde{\f}}_{t}$.  More details about the sampling steps and examples can be found in Appendix~\ref{subsec:samplingTGPSSM}. 
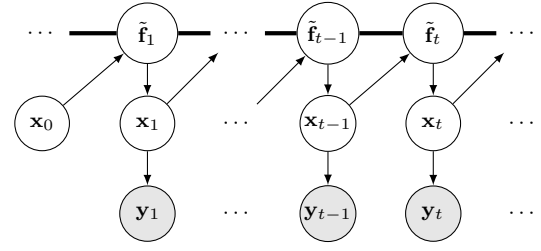
\begin{figure}[t!]
	\centering
	\footnotesize
	\begin{tikzpicture}[align = center, latent/.style={circle, draw, text width = 0.45cm}, observed/.style={circle, draw, fill=gray!20, text width = 0.45cm}, transparent/.style={circle, text width = 0.45cm}, node distance=1.2cm]
		\node[latent](x0) {${\x}_0$};
		\node[latent, right of=x0, node distance=1.4cm](x1) {${\x}_{1}$};
		\node[transparent, right of=x1](x2) {$\cdots$};
		\node[latent, right of=x2](xt-1) {$\!\!{\x}_{t-1}\!\!$};
		\node[latent, right of=xt-1, node distance=1.4cm](xt) {${\x}_{t}$};
		\node[transparent, right of=xt](xinf) {$\cdots$};
		\node[transparent, above of=x0](f0) {$\cdots$};
		\node[latent, above of=x1](f1) {${\tf}_{1}$};
		\node[transparent, right of=f1](f2) {$\cdots$};
		\node[latent, above of=xt-1](ft-1) {$\!\!{\tf}_{t-1}\!\!$};
		\node[latent, above of=xt](ft) {${\tf}_{t}$};
		\node[transparent, right of=ft](finf) {$\cdots$};
		\node[observed, below of=x1](y1) {${\y}_{1}$};
		\node[transparent, below of=x2](y2) {$\cdots$};
		\node[observed, below of=xt-1](yt-1) {$\!\!{\y}_{t-1}\!\!$};
		\node[observed, below of=xt](yt) {${\y}_{t}$};
		\node[transparent, right of=yt](yinf) {$\cdots$};
		\draw[-latex] (x0) -- (f1);
		\draw[-latex] (f1) -- (x1);
		\draw[-latex] (x1) -- (f2);
		\draw[-latex] (ft-1) -- (xt-1);
		\draw[-latex] (xt-1) -- (ft);
		\draw[-latex] (x2) -- (ft-1);
		\draw[-latex] (ft) -- (xt);
		\draw[-latex] (xt) -- (finf);
		\draw[-latex] (x1) -- (y1);
		\draw[-latex] (xt-1) -- (yt-1);
		\draw[-latex] (xt) -- (yt);
		\draw[ultra thick]
		(f0) -- (f1)
		(f1) -- (f2)
		(f2) -- (ft-1)
		(ft-1) -- (ft)
		(ft) -- (finf)
		;
	\end{tikzpicture}
	\caption{Graphical model of TGPSSM. The white circles indicate that the variables are latent, while the gray circles represent the observable variables. The thick horizontal bar represents a set of fully connected nodes, i.e., the TGP.}
	\label{fig:TGPSSM_model_inducing}
\end{figure}
\begin{remark}[Dependency construction] \label{remark:dependency_construction}
In the GPSSM described in Eq.~(\ref{eq:gpssm}), the $d_x$-dimensional ($d_x >1$) transition function is  typically modeled by $d_x$ mutually independent GPs, which is simplifying but may be unrealistic \cite{lin2022output}. 
This issue can be tackled in TGPSSMs as the $d_x$-dimensional TGP prior can establish dependencies between dimensions via normalizing flow.  For instance, coupling layers in RealNVP can output dependent function values. Notably, in the case of a linear normalizing flow, the TGPSSM is equivalent to  the output-dependent GPSSM proposed in \cite{lin2022output}, offering potential practical and inference advantages.
\end{remark}

The computational complexity of TGPSSM mainly arises from the computation of standard GP model. For instance, sampling from TGPSSM scales as $\mathcal{O}(d_x T^3)$ (see Appendix~\ref{subsec:samplingTGPSSM}), which is prohibitive for big data. To address the computational issue, we  leverage the sparse GP \cite{titsias2009variational} to alleviate the cubic computational complexity in the TGPSSM. By introducing sparse GP to augment the GP prior with $M$ inducing points $\vu \triangleq \u_{1:M} = \{\u_1, \u_2, ..., \u_M\}$, at locations $\vz \triangleq \z_{1:M} =  \{\z_1, \z_2, ..., \z_M\}$, i.e., $\u_{1:M} = f(\z_{1:M})$, the TGP prior after introducing the sparse inducing points is given by
\begin{equation}
	p(\tF, \tU)  = p(\tilde{\f}_{1:T}, \tilde{\u}_{1:M}) = \underset{=p(\tF \vert \tU)}{\underbrace{p(\vf \vert \vu)  \J_{\f}}} \cdot \underset{=p(\tU)}{\underbrace{p(\vu) \J_{\u}}},
	\label{eq:joint_TGP}
\end{equation}
where $p(\vu)$ is the sparse GP distribution, $p(\vf \vert \vu)$ is the corresponding noiseless GP posterior whose mean and covariance can be computed similarly to Eq.~(\ref{eq:GP_posterior}),
$\tU \triangleq \tilde{\u}_{1:M} = \tilde{f}(\z_{1:M})$, and 
\begin{subequations}
	\begin{align}
		& \J_{\f} = \prod_{j=1}^{J-1} \left| \det \frac{\partial \mathbb{G}_{\theta_{j}}\left( \G_{\theta_{j-1}}\left( ... \G_{\theta_0}(\vf)... \right) \right)}{\partial \G_{\theta_{j-1}}\left( ... \G_{\theta_0}(\vf)... \right)  }  \right|^{-1},\\
		& \J_{\u} = \prod_{j=1}^{J-1} \left|\operatorname{det} \frac{\partial \mathbb{G}_{\theta_{j}}\left( \G_{\theta_{j-1}}\left( ... \G_{\theta_0}(\vu)... \right) \right)}{\partial \G_{\theta_{j-1}}\left( ... \G_{\theta_0}(\vu)... \right) } \right|^{-1}.
	\end{align}
\end{subequations}
Detailed derivations of Eq.~(\ref{eq:joint_TGP}) can be found in Section~\ref{appedx:derivation_joint_tgp}, where we use the fact that the coordinate-wise marginal flow $\G_{\btheta_{F}}(\cdot)$, induces a valid TGP and, in turn, guarantees that the transformed sparse GP is consistent, see Corollary \ref{corollary1}. 
Therefore, the joint distribution of the TGPSSM augmented by the inducing points is
\begin{equation}
	p(\tF, \tU, \vx, \vy) \! \!=\!\! p(\mathbf{x}_{0})  p(\tf_{1:T}, \tu_{1:M}) \!\! \prod_{t=1}^{T}\!   p(\mathbf{x}_{t} \vert {\tf}_{t}) p(\mathbf{y}_{t} \vert \mathbf{x}_{t}).\!
	\label{eq:joint_dist}
\end{equation}
Similarly, we can sample from the TGPSSM prior by first generating a random instance $\vu$ from $p(\vu)$, and then sampling $\f_t$ from the conditional distribution $p(\f_t \vert \vu, \x_{t-1})$ at each time step $t$, given the sampled $\vu$ and previous state $\x_{t-1}$. The detailed sampling procedure for this prior is also given in Appendix~\ref{subsec:samplingTGPSSM}.   With the aid of the sparse GP, the computational complexity of the TGPSSM can be reduced to $\mathcal{O}(d_x T M^2)$ with $M \ll T$, comparing to the original $\mathcal{O}(d_x T^3)$.

Lastly, we note that if $M$ is sufficiently large, the set of inducing points, $\vu$, is often regarded as a \textit{sufficient statistic} for the GP function values, $\vf$, in a sense that given $\vu$, the function values $\vf$ and any novel $\f_*$ are independent \cite{titsias2009variational}, i.e., $p(\f_* \vert \vf, \vu) = p(\f_* \vert \vu)$, for any $\f_*$.   This property remains valid in the TGP owing to the bijective nature of normalizing flow. We summarize the result in the following theorem.
\begin{theorem}
	\label{theorem:ips_tgp}
	If the inducing points, $\vu$, is a sufficient statistic for the standard GP model, meaning that $M$ is sufficiently large, then the transformed inducing points, $\tilde{\vu} \!=\! \tilde{f}(\z_{1:M})$, is a \textit{sufficient statistic} for the TGP function values, $\tilde{\vf} \!=\! \tilde{f}(\x_{1:T})$.
\end{theorem}
\begin{proof}
	See Section~\ref{appedx:proof_thm2}.
\end{proof}
Theorem \ref{theorem:ips_tgp} suggests that TGPSSM inherits the merit of sparse GP models, which can greatly benefit the design of scalable learning and inference algorithms for TGPSSM (see Section \ref{sec:vari_infer}).  Similar to the sparse GPs, in TGPSSM, the number of inducing points, $M$, in practice can be predetermined based on available computational resources, and the inducing locations, $\vz$, can be selected by applying gradient-based optimization.  Carefully selecting the inducing locations, $\vz$,  is necessary when $M$ is insufficiently large, as it allows the transition function to be better specified by the transformed inducing points, $\tU$. In contrast, if $M$ is large enough and $\tU$ is sufficient for $\tF$, TGPSSM will be insensitive to the specific locations of inducing points, and optimizing their locations becomes optional. 

\section{Variational Learning and Inference}
\label{sec:vari_infer}
To learn TGPSSM and simultaneously infer the latent states without significant computational cost, this paper resorts to the sparse representation of TGPSSM and variational approximation methods \cite{theodoridis2020machine}. 
In the following, Section \ref{subsec:VI_problems} first points out the issues existing in the variational GPSSM literature.  To overcome these issues, Section \ref{subsec:mf_tgpssm} and Section \ref{subsec:mf_cotgpssm} detail the proposed variational inference algorithms for TGPSSM.

\subsection{Variational Inference and Approximations} \label{subsec:VI_problems}
Let us take the TGPSSM augmented by the sparse inducing points (see Eq.~(\ref{eq:joint_dist})) as an example. 
In Bayesian statistics, the model evidence $p(\vy \vert \btheta)$ is a fundamental quantity for model selection and comparison \cite{courts2021gaussian}. By maximizing the logarithm of $p(\vy \vert \btheta)$ w.r.t. the model parameters $\btheta$, the goodness of data fitting and the model complexity are automatically balanced, in accordance with Occam's razor principle \cite{cheng2022rethinking}.
However,  $p(\vy \vert \btheta)$ is obtained by integrating out all the latent variables $\{\vx, \tF, \tU\}$ in the joint distribution, see Eq.~(\ref{eq:joint_dist}), which is analytically intractable.  Thus, the posterior distribution of the latent variables, $p(\vx, \tF, \tU\vert \vy) = \frac{p(\vy, \vx, \tF, \tU)}{p(\vy \vert \btheta)},$ cannot be expressed in a closed-form analytical expression, either.  This intractability issue has been addressed in variational Bayesian methods by adopting a variational distribution \cite{theodoridis2020machine}, $q(\vx, \tF, \tU)$, to approximate the intractable $p(\vx, \tF, \tU\vert \vy)$. With the variational distribution $q(\vx, \tF, \tU)$, it can be shown that the logarithm of the evidence function is bounded by the so-called evidence lower bound (ELBO) \cite{cheng2022rethinking}, namely,
\begin{align}
\!\!\!\!	\log p(\vy \vert \btheta) \ge \operatorname{ELBO} \!\triangleq\! \mathbb{E}_{q(\vx, \tF, \tU )} \! \left[\log \frac{p(\tF, \tU, \vx, \vy)} {q(\vx, \tF, \tU) }\right]\!\!,
	\label{eq:ELBO_general}
\end{align}
where the ELBO will serve as a surrogate function to be maximized w.r.t. the model parameters $\btheta$ and variational distribution $q(\vx, \tF, \tU)$.  Maximizing the ELBO w.r.t. the variational distribution $q(\vx, \tF, \tU)$ corresponds to inference, while maximizing it w.r.t. the model parameters $\btheta$ corresponds to learning. 

The tightness of the ELBO is determined by the closeness between the variational distribution $q(\vx, \tF, \tU)$ and posterior $p(\vx, \tF, \tU\vert \vy)$, measured by the Kullback-Leibler (KL) divergence  \cite{theodoridis2020machine}, $\operatorname{KL}[ q(\vx, \tF, \tU) \| p(\vx, \tF, \tU\vert \vy) ]$. Therefore, it is essential to design a variational distribution flexible enough to match $p(\vx, \tF, \tU\vert \vy)$ closely, thus tightening the ELBO. Meanwhile, the variational distribution should be able to make the ELBO more tractable.
For a tight ELBO, the variational distribution for the TGPSSM should mirror the factorization of the true posterior distribution \cite{courts2021gaussian}, $p(\vx, \tF, \tU\vert \vy)$, which can be factorized as
\begin{equation} 
 p(\vx, \tF, \tU\vert \vy) = p(\tF, \tU \vert \vy) p(\x_0 \vert \vy) \prod_{t= 1}^T p(\x_{t} \vert \tilde{\f}_t, \vy), 
\end{equation}
according to the model defined in Eq.~(\ref{eq:joint_dist})  \cite{krishnan2017structured}.  Therefore,  the ideal variational distribution for the TGPSSM is factorized as  
\begin{equation}
	q(\vx, \tF, \tU) =  q(\tF, \tU) q(\x_0) \prod_{t= 1}^T q(\x_{t} \vert \tilde{\f}_t),
	\label{eq:generic_vi_dist}
\end{equation}
where $q(\tF, \tU), q(\x_0)$ and $q(\x_{t} \vert \tilde{\f}_t)$ are the corresponding variational distributions of the latent variables. The generic factorization of $q(\vx, \tF, \tU)$ in Eq.~(\ref{eq:generic_vi_dist}) has been known as the NMF assumption in the GPSSM literature \cite{doerr2018probabilistic}, because it explicitly builds the dependence between the latent states and the transition function values.  However, the NMF assumption is incapable of making the ELBO more tractable, thus bringing a more significant computational cost of learning and inference \cite{ialongo2019overcoming,curi2020structured,lindinger2022laplace}, compared to the MF approximation summarized as follows.
\begin{assumption}[Mean-field assumption]
\label{assumption:MF}
 We assume that the variational distribution, $q(\tF, \tU, \x_{0:T})$, is factorized such that the transition function values, $\{\tF, \tU\}$, and the latent states, $\x_{0:T}$, are independent, which mathematically can be expressed as
\begin{equation}
q(\tF, \tU) q(\x_0) \prod_{t=1}^T q(\x_{t} \vert \tf_t) = q(\tF, \tU) q(\x_{0:T}),
\label{eq:mf_assumption}
\end{equation}
where $q(\tF, \tU)$ and $q(\x_{0:T})$ are the variational distributions that need to be further designed (see next subsection). 
\end{assumption}
\noindent The MF assumption simplifies the variational distribution yet will enable us to handle a more tractable ELBO \cite{eleftheriadis2017identification}, see Section \ref{subsec:mf_tgpssm}.  The algorithms based on the MF assumption correspond to the MF variational learning and inference algorithms.
%

%
%

Due to the integration of normalizing flow, the existing MF variational inference algorithms in the GPSSM literature \cite{frigola2014variational,frigola2015bayesian,mchutchon2015nonlinear,eleftheriadis2017identification} are not directly applicable to the TGPSSM. For example, to make use of the algorithm in \cite{frigola2014variational}, analytical marginalization of the GP transition function values is required. However,  the nonlinear normalizing flow in the TGPSSM makes it impossible to analytically integrate out the corresponding TGP function values, see Remark \ref{remark:non-closed-form} in Section~\ref{ELBO_1}. Additionally, existing works, e.g., \cite{mchutchon2015nonlinear,eleftheriadis2017identification}, directly assume that $q(\x_{0:T})$ is a joint Gaussian distribution with a Markovian structure, which can fail to closely match the posterior distribution of the latent states that can be neither Gaussian nor unimodal. 
Another issue with the existing algorithms is their overemphasis on achieving high ELBO values, rather than learning informative latent state-space representations. As a result, the underlying system dynamics cannot be well captured, leading to a degenerated model inference performance. To address this issue, one must ensure that the learned state representation can reconstruct the observations and, at the same time, the transition function can characterize the state dynamics \cite{gedon2020deep}.  In the following subsections, we propose our variational algorithm to address the aforementioned issues.  By explicitly assuming a non-Gaussian variational distribution and exploiting a constrained optimization framework,  the proposed algorithm helps improve the inference and state-space representation capacities in TGPSSM.

\subsection{Mean-Field Variational Algorithm}
\label{subsec:mf_tgpssm}

By utilizing Theorem \ref{theorem:ips_tgp} and employing the same algebraic tricks as in the sparse variational GP \cite{titsias2009variational,hensman2013gaussian}, we set $q(\tF, \tU) \!=\! q(\tF \vert \tU) q(\tU) \!=\! p(\tF \vert \tU)q(\tU)$; additionally adopting Assumption \ref{assumption:MF}, the generic variational distribution in Eq.~(\ref{eq:generic_vi_dist}) then becomes
\begin{subequations}
	\label{eq:variation_di_tgp}
	\begin{align}
		q(\vx, \tF, \tU) 
		&=  q(\x_{0:T}) \cdot {q(\tU)} \cdot {p(\tF \vert \tU)}\\
		& = q(\x_{0:T}) \cdot  {q(\vu)} \J_{\u} \cdot  p(\vf \vert \vu)  \J_{\f}, \label{eq:variation_di_tgp_2}
	\end{align}
\end{subequations}
where the variational distribution of the inducing points $q(\vu)$ is assumed to be a free-form Gaussian \cite{hensman2013gaussian}, i.e., 
\begin{equation}
	    q(\vu)\!=\! \prod_{d=1}^{d_x} \cN(\{\u_{i, d}\}_{i=1}^M \vert \m_d, \mathbf{L}_{d} \mathbf{L}_{d}^\top ) \!=\! \cN(\vu \vert \mathbf{m}, \mathbf{S}),
	\label{eq:qu_variational}
\end{equation}
with free variational parameters $\mathbf{m} \!=\! [\m_1^\top, ..., \m_{d_x}^\top]^\top \in \mathbb{R}^{M d_x}$ and $\mathbf{S} \!=\! \operatorname{diag}(\mathbf{L}_1 \mathbf{L}_{1}^\top, ..., \mathbf{L}_{d_x} \mathbf{L}_{d_x}^\top) \in \mathbb{R}^{Md_x \times Md_x}$; the variational distribution of the latent states, $q(\x_{0:T})$, is assumed to be Markov-structured, and is modeled by a neural network, a.k.a. inference network, hence the number of variational parameters will not grow linearly with the length of the observation sequence \cite{krishnan2017structured}. More specifically, the structured variational distribution is characterized by the following set of equations, 
\begin{subequations}
\label{eq:MF_qx}
	\begin{align}
		& q(\x_0) = \cN \! \left(\x_0 \vert  \mathbf{m}_{0},  \mathbf{L}_{0} \mathbf{L}_{0}^\top \right)\\
		& q(\x_{0:T})  = q(\x_{0}) \prod_{t=1}^T q(\x_{t} \vert \x_{t-1} )\\
		& q(\x_{t} \vert \x_{t-1}) \!=\! \cN\! \left(  \x_t \vert \Phi_{\bphi}(\x_{t-1}, \y_{1:T}), \bm{\Sigma}_{\bphi}(\x_{t-1}, \y_{1:T}) \right)
	\end{align}
\end{subequations}
where $\Phi_{\bphi}(\cdot)$ and $\bm{\Sigma}_{\bphi}(\cdot)$ are the outputs of the inference network with inputs $\x_{t-1}$ and $\y_{1:T}$, and $\bphi$ denotes the network model parameters; vector ${\m}_{0} \in \mathbb{R}^{d_x}$ and lower-triangular matrix $\mathbf{L}_{0} \in \mathbb{R}^{d_x \times d_x}$ are free variational parameters of $q(\x_0)$.  
Note that the Markovian structure of $q(\x_{0:T})$ is assumed here as a result of Proposition \ref{proposition:optimal_Markov_qx_maintext}, of which the proof can be found in  Section~\ref{ELBO_1}.
\begin{proposition}
	\label{proposition:optimal_Markov_qx_maintext}
	The optimal  state distribution $q^*(\x_{0:T})$ for maximizing the ELBO in the MF case has a Markovian structure.
\end{proposition}
\begin{remark}
Unlike the works in \cite{mchutchon2015nonlinear, eleftheriadis2017identification} assuming that $q(\x_{0:T})$ is a joint Gaussian distribution with Markovian structure,  the variational distribution $q(\x_{0:T})$  defined in Eq.~(\ref{eq:MF_qx})  can be non-Gaussian due to the newly introduced inference network, making the variational distribution more flexible to match multimodal posterior distribution of the latent states.
 It is also worth noting that unlike previous work in \cite{mattos2019stochastic} that considers a full factorization over $q(\x_{0:T})$, i.e., $q(\x_{0:T}) = \prod_{t=0}^T q(\x_t)$, in this paper, we analyze the structure of the optimal distribution, $q^*(\x_{0:T})$, and adopt the optimal structure for the variational distribution $q(\x_{0:T})$.
\end{remark}

Based on the model defined in Eq.~(\ref{eq:joint_dist}) and the corresponding variational distributions, we can ultimately obtain the ELBO, as summarized in the following theorem.
\begin{theorem}[Evidence lower bound]
	\label{thm:ELBO_mf}
	Under Assumption \ref{assumption:MF}, and the variational distributions in Eqs.~(\ref{eq:variation_di_tgp}, \ref{eq:qu_variational}, \ref{eq:MF_qx}),  the ELBO of the TGPSSM defined in Eq.~(\ref{eq:ELBO_general}) is given by 
	\begin{subequations}
		\label{eq:elbo_tgpssm_mf}
		\begin{align}
			& \operatorname{ELBO} =  \nonumber \\
			&
			 - \operatorname{KL}\left[ q(\x_0) \| p(\x_0) \right]  \label{eq:term2} \\
			&  
		   	-\operatorname{KL}\left[ q(\vu) \| p(\vu) \right] \label{eq:term3}  \\
			&  
			 +  \sum_{t=1}^T \mathbb{E}_{q(\x_{t-1})}\left[ \frac{d_x}{2} \log (2 \pi)+\frac{1}{2} \log \left| \bm{\Sigma}_{\x_t} \right|+\frac{1}{2} d_x  \right] 
			 \label{eq:term4} \\
			&  +\sum_{t=1}^T  \E_{q(\x_{t-1:t}) q(\f_t) }  \left[  \log p(\x_t \vert \G_{\btheta_{F}}(\f_t))\right] \label{eq:term5} \\
			&   +	\sum_{t=1}^T  \underset{q(\x_{t-1})}{ \mathbb{E} } \left[ \log \mathcal{N}(\y_t \vert \bm{C}  \Phi_{\x_t}, \bm{R}) \!-\! \frac{1}{2}\operatorname{tr}\left[ \bm{R}^{-1} (\bm{C \Sigma_{\x_t} C^\top})\right] \right]
			\label{eq:term1} 
		\end{align}
	\end{subequations}
where $\Phi_{\x_t}\triangleq  \Phi_{\phi}(\x_{t-1}, \y_{1:T})$ and $\bm{\Sigma}_{\x_t} \triangleq \bm{\Sigma}_{\phi}(\x_{t-1}, \y_{1:T})$,  for notation brevity. 
\end{theorem}
\begin{proof}
	See  Section~\ref{ELBO_1}.  
	\vspace{-.03in}
\end{proof}

\begin{remark}[Interpretability of ELBO]
 Each component of the ELBO derived in Theorem \ref{thm:ELBO_mf}  is interpretable.  Their physical meanings are given as follows:
 \begin{itemize}
 	\item The KL term in Eq.~(\ref{eq:term2})  represents a regularization term for $q(\x_0)$, which discourages $q(\x_0)$ from deviating too much from the prior distribution $p(\x_0) = \cN(\x_0 \vert \bm{0}, \mathbf{I})$;
 	\item  Eq.~(\ref{eq:term3})  represents a regularization term for the GP transition surrogate $q(\vu)$, which discourages $q(\vu)$ from deviating too much from the GP prior $p(\vu)$; 
 	\item  Eq.~(\ref{eq:term4})  is the differential entropy term of the latent state trajectory. Maximizing the ELBO essentially encourages ``stretching'' every $q(\x_t)$ so that the approximated posterior distribution $q(\x_{0:T})$ will not be overly tight; 
 	\item  Eq.~(\ref{eq:term5})  represents the reconstruction of the latent state trajectory, which encourages the transition function  $\G_{\btheta_{F}}(\cdot)$ to fit the latent states $\x_{t} \sim q(\x_t\vert \x_{t-1}), \forall t$. In other words, this term measures the quality of learning/fitting the underlying system dynamics. Empowered by the normalizing flow, the TGP model in the TGPSSM is expected to be able to capture more complex system dynamics;
 	\item  Eq.~(\ref{eq:term1}) represents the data reconstruction error, which encourages any state trajectory $\x_{0:T}$ from $q(\x_{0:T})$ to accurately reconstruct the observations.
 \end{itemize}
\end{remark}
Note that there is no analytical form for Eq.~(\ref{eq:term5}), we thus use the following Monte-Carlo approximation:
\begin{align}
	& \text{Eq.(\ref{eq:term5})}  \approx  \sum_{t=1}^T  \E_{q(\x_{t-1:t}) }  \left[  \frac{1}{n} \sum_{i=1}^n \log p(\x_t \vert \G_{\btheta_{F}}(\f_t^{(i)}))\right] \nonumber \\
	& = \sum_{t=1}^T  \E_{q(\x_{t-1:t}) }  \left[  \frac{1}{n} \sum_{i=1}^n \log \cN(\x_t \vert \G_{\btheta_{F}}(\f_t^{(i)}),  \bm{Q})  \right], \label{eq:apprx_term5}
\end{align}
where $\{\f_t^{(i)}\}_{i=1}^n, n \in \mathbb{N}$, are generated from $q(\f_t)$, and 
\begin{align}
	q(\f_{t}) &= \int q(\f_{t}, \vu)  \mathrm{d} \vu= \mathcal{N}\left(\f_t \vert  {\m}_{\f_t}, \bm{\Sigma_{\f_t}} \right), \label{eq:q(f_0_t)}
\end{align}
and, with a bit abuse of notation, 
\begin{align}
	& {\m}_{\f_t} \!=\! K_{\x_{t-1}, \vz} K_{\vz, \vz}^{-1} {\m},\\
	& \bm{\Sigma_{\f_t}} \!=\! K_{\x_{t-1}, \x_{t-1}} \!-\!  K_{\x_{t-1}, \vz} K_{\vz, \vz}^{-1}\left[K_{\vz, \vz} \!-\! \mathbf{S}\right] K_{\vz, \vz}^{-1} K_{\vz, \x_{t-1}}. \nonumber
\end{align} 
Lastly, we use  the reparametrization trick \cite{kingma2019introduction} to compute the expectations w.r.t. the intractable $q(\x_{t-1:t})$. By sampling trajectories from $q(\x_{0:T})$ and evaluating the integrands in Eq.~(\ref{eq:elbo_tgpssm_mf}), we can obtain an estimate for the ELBO.
Together with all, we can apply gradient ascent-based methods \cite{kingma2015adam} to maximize the ELBO w.r.t. the model parameters $\btheta = [\bm{\theta}_{gp}, \btheta_{F}, \bm{Q}, \bm{R} ]$ and the variational parameters $\bm{\zeta} = [ \mathbf{m}_0, \mathbf{L}_0, \bphi, \mathbf{m}, \mathbf{S}, \vz]$. The gradient information can be back-propagated owning to the reparametrization trick \cite{kingma2019introduction}.   The pseudo code for implementing the MF variational learning and inference algorithm is summarized in Algorithm \ref{alg_tgpssm1}. 
\begin{remark} \label{remark:computaionalcomplexity}
The computational complexity of Algorithm \ref{alg_tgpssm1} lies in the computations of GP, normalizing flow, and inference network.  Compared to the $\mathcal{O}{(T d_x M^2)}$ computational costs in the GPSSMs using inference networks \cite{eleftheriadis2017identification,ialongo2019overcoming,curi2020structured} (assuming the GP is the main computational bottleneck), Algorithm \ref{alg_tgpssm1} for the TGPSSM only gently increases the computational complexity due to the use of normalizing flow. In fact, in the case of applying the elementary flows, the additional computational complexity it brings is negligible. Moreover, for large datasets with a large $T$, stochastic gradient optimization methods can be employed to make the algorithm scalable; See Appendix \ref{subsec:appex_scalability} for more detailed discussions.
\end{remark}

%
%
\begin{algorithm}[t!] 
	\caption{MF Variational Learning for TGPSSM} 
	\label{alg_tgpssm1} 
	\KwIn{Dataset $\{\y_{1:T}\}$.  Initial parameters $\btheta^{(0)}$, $\bm{\zeta}^{(0)}$.}
	\KwOut{Parameters $\btheta$ and $\bm{\zeta}$.} 
	\While{not converged}
	{	
		Evaluate Eq.~(\ref{eq:term2}) and Eq.~(\ref{eq:term3});\\
		Sample state trajectory $\x_{0:T} \! \sim  \! q(\x_{0:T})$ (Eq.~(\ref{eq:MF_qx}));\\
		  \For{$t=1:T$}{
			  Evaluate data reconstruction term, Eq.(\ref{eq:term1});\\
		    Evaluate entropy term,  Eq.~(\ref{eq:term4});\\
			Compute $q(\f_t)$ using Eq.~(\ref{eq:q(f_0_t)});\\
			Sample $\f_{t}^{(i)} \sim q(\f_{t}), i = 1, 2, ..., n$;\\
			Evaluate state reconstruction term, Eq.~(\ref{eq:apprx_term5});\\
		}
		Evaluate ELBO of Eq.~(\ref{eq:elbo_tgpssm_mf});\\
		Estimate the Monte-Carlo gradient w.r.t. $\btheta$ and $\bzeta$;\\
		Update $\btheta$ and $\bzeta$ using Adam \cite{kingma2015adam};
	}
\end{algorithm}	
%
%
\subsection{Taming TGPSSM with Constrained Optimization} \label{subsec:mf_cotgpssm}
Note that Algorithm \ref{alg_tgpssm1} optimizes both the parameters $\btheta$ and $\bzeta$ simultaneously.  However,  a local optimal obtained from maximizing the ELBO does not necessarily imply that the model has learned an appropriate latent state-space representations of the underlying system \cite{alemi2018fixing}.
Similar works in the area of variational autoencoder (VAE) address this issue by introducing heuristic weighting schedules, e.g., $\beta$-VAEs \cite{higgins2016beta}, where hand-crafted annealing of KL-terms is often used to achieve the desired performance. However, such solutions are sensitive to changes in model architecture and/or dataset \cite{knoblauch2022optimization}.
Another line of work reformulates the VAE objective function as the Lagrangian of a constrained optimization problem \cite{klushyn2019learning}. Inspired by this, we propose a constrained optimization framework to learn more informative TGPSSM representations of the underlying system dynamics.

For ease of discussion in the sequel, we denote the data reconstruction term, $\sum_{t=1}^T \E_{q(\x_{t})} \left[\log p(\y_t \vert \x_t)\right]$ (Eq.~(\ref{eq:term1}) in the ELBO), as $\mathcal{R}$.
To learn a good state-space representation that is able to reconstruct the observed data $\y_{1:T}$, we add an additional constraint on $\mathcal{R}$  to control the learning behavior. More specifically, we solve the following constrained optimization problem: 
\begin{equation}
	\min_{\btheta, \bzeta} \ -\operatorname{ELBO},  \quad  \operatorname{s.t.} \  \mathcal{R} \ge  \mathcal{R}_0,
	\label{eq:CO_ELBO}
\end{equation}
where $\mathcal{R}_0$ is the desired quality of observation reconstruction. The constraint on $\mathcal{R}$ serves as a guide for the inference network to learn the latent states that can accurately reconstruct the observations. This, in turn, provides a sound foundation for learning the underlying system dynamics. For further insights on how to empirically choose an appropriate value of $\mathcal{R}_0$, we refer the reader to Appendix~\ref{subsec:selection_R0_appx}. Note that compared to the unconstrained optimization problem, $\max_{\btheta, \bzeta} \operatorname{ELBO}$, solved by Algorithm \ref{alg_tgpssm1}, the constrained optimization problem formulated in Eq.~(\ref{eq:CO_ELBO}) can be interpreted as narrowing down the solution space of $\{\btheta, \bzeta\}$ by imposing an additional constraint on the data reconstruction term. As a result, once the constrained optimization problem is successfully solved, the resulting ELBO can still be guaranteed to be a valid lower bound of the model evidence.

To solve the constrained optimization problem, we can utilize the standard method of Lagrange multipliers \cite{theodoridis2020machine}, which involves introducing a Lagrange multiplier $\beta \geq 0$ and creating a new function called the Lagrangian $\mathcal{L}_{\beta}$. The Lagrangian is then optimized using a min-max optimization scheme, 
\begin{equation}
	\min_{\btheta, \bzeta}  \max_{\beta \ge 0} \ \mathcal{L}_{\beta} = -\operatorname{ELBO} + \beta (\mathcal{R}_0 - \mathcal{R}),
\end{equation}
where the parameters $\{\btheta, \bzeta\}$ are optimized through gradient descent, while the Lagrange multiplier $\beta$ is updated by 
\begin{equation}
\beta^{(i)} \leftarrow \beta^{(i-1)} \cdot  \exp\left( -\eta  \cdot (\hat{\mathcal R}^{(i)} - {\mathcal R}_0)\right)
\end{equation}
to enforce a non-negative $\beta$, where $\eta$ is the associated learning rate and $\hat{\mathcal{R}}^{(i)}$ is an estimate of the data reconstruction term in the $i$-th iteration.  In the context of the stochastic gradient training, $\hat{\mathcal{R}}^{(i)}$ can be estimated using a moving average, i.e., $\hat{\mathcal{R}}^{(i)} = (1-\alpha) \cdot \hat{\mathcal{R}}^{(i)}_{batch}  + \alpha \cdot \hat{\mathcal{R}}^{(i-1)}$, where $\alpha$ is a prefixed hyperparameter, and $\hat{\mathcal{R}}^{(i)}_{batch}$ is the reconstruction term estimated using current batch data. Throughout this paper, we empirically find that setting the values of $\alpha=0.5$ and $\eta = 0.001$ can consistently yield reasonable performance.  For clarity, we summarize the pseudo code for implementing the learning algorithm with constrained optimization framework in Algorithm \ref{alg_tgpssm2}.  It is obvious that Algorithm \ref{alg_tgpssm2} admits the same computational complexity as Algorithm \ref{alg_tgpssm1}.
\begin{remark}
Instead of manually tweaking the abstract hyperparameter $\beta$ that implicitly affects the model learning/fitting  performance as seen in the $\beta$-VAEs,  Algorithm \ref{alg_tgpssm2} leveraging the method of Lagrange multipliers and incorporating a more interpretable constraint, namely the desired data reconstruction quality, enables a more meticulous and principled update of the $\beta$.
Moreover, unlike the constrained optimization problem modeled in the VAEs  \cite{klushyn2019learning} that requires $0\le \beta \le 1$ to guarantee a valid lower bound of the model evidence,   our constrained optimization problem does not require setting an upper bound constraint for $\beta$. 
\end{remark}
\begin{algorithm}[t!] 
	\caption{Taming Variational TGPSSM with Constrained Optimization} 
	\label{alg_tgpssm2} 
	\KwIn{Dataset $\{\y_{1:T}\}$.  Initial parameters $\btheta^{(0)}$, $\bm{\zeta}^{(0)}$, $\beta^{(0)} = 1$, $i = 0$, $\alpha = 0.5$, $\eta = 0.001$. }
	\KwOut{Parameters $\btheta$ and $\bm{\zeta}$.} 
	\While{not converged}
	{	
		Estimate $\!\hat{\mathcal{R}}^{(i)}_{batch}\!$ using the $i$-th batch data (Eq.(\ref{eq:term1}));\\
		Estimate $\hat{\mathcal{R}}^{(i)} = (1-\alpha) \cdot \hat{\mathcal{R}}^{(i)}_{batch}  + \alpha \cdot \hat{\mathcal{R}}^{(i-1)}$;\\
		Update $\beta^{(i)} \leftarrow \beta^{(i-1)}  \cdot \exp  \left[ -\eta   \cdot \left(\mathcal{R}^{(i)} - \mathcal{R}_0\right) \right]$;\\
		Optimize $\mathcal{L}_{\beta}$ w.r.t. $\btheta$, $\bzeta$ as shown in Algorithm \ref{alg_tgpssm1};\\
		$i \leftarrow i + 1$;
	}
\end{algorithm}

\section{Experimental Results}
\label{sec:experimental-results}
This section presents a comprehensive numerical study of the proposed variational algorithms for TGPSSMs, which includes evaluating the performance of the TGPSSM on multiple datasets and comparing it with various benchmark algorithms. Additional details regarding the experimental setup can be found in Appendix, and the accompanying source code is publicly available online\footnote{\url{https://github.com/zhidilin/TGPSSM}}.

%
\begin{figure}[t!]
	\centering
	\subfloat[``Kink'' dynamical function]{\label{fig:KinkFunction}\includegraphics[width =0.22\textwidth, height= 3.3cm]{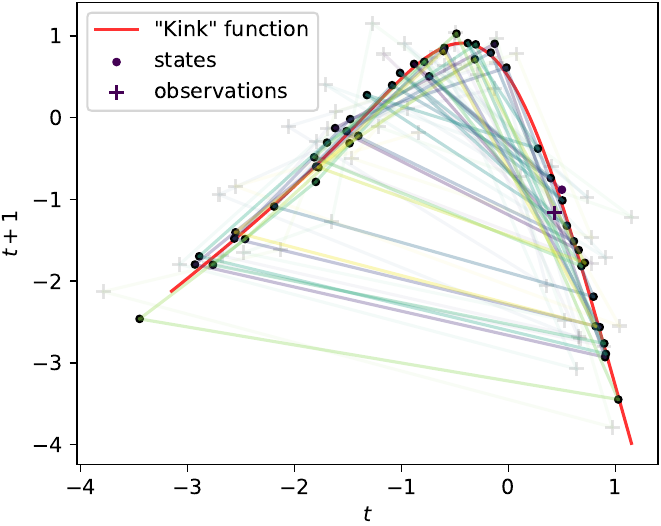}}  \hspace{.04cm}
	\subfloat[``Kink-step'' dynamical function]{\label{fig:ksfunc}\includegraphics[width =0.26\textwidth]{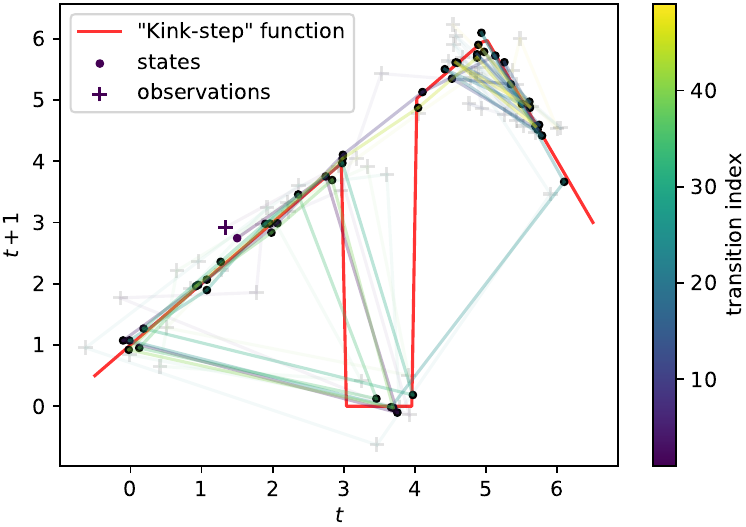}}
	\caption{Two 1-D dynamical systems and the generated 50 latent states \& observations.}
	\label{fig:1D_dataset}
\end{figure}
\begin{figure*}[t!]
	\centering
	\subfloat[JO-GPSSM (MSE: 2.2313)]{\label{fig:gp_ksfunc} \includegraphics[width=0.489\columnwidth]{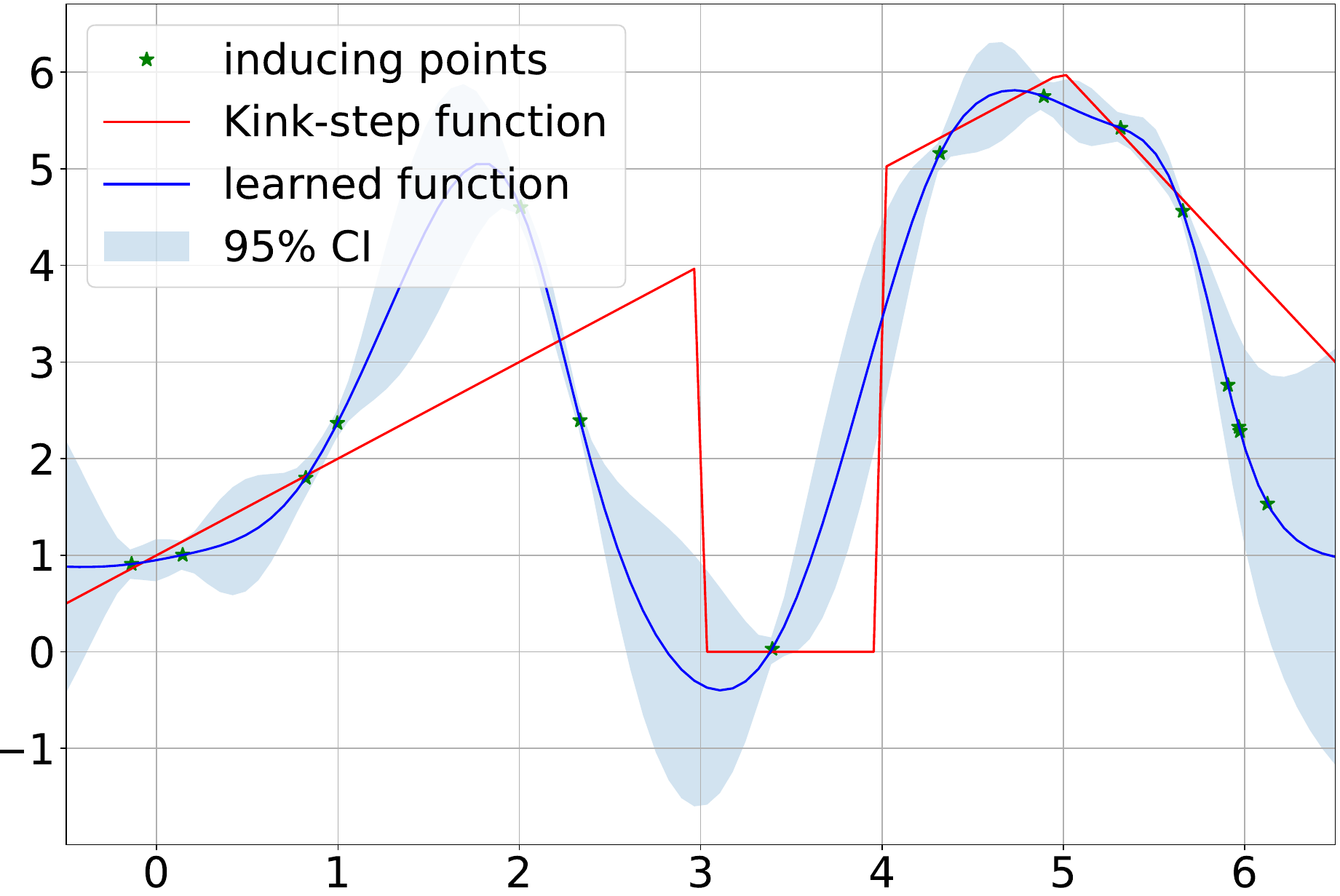}}
	\subfloat[CO-GPSSM (MSE: {0.3537}) ]{\label{fig:cogp_ksfunc} \includegraphics[width=0.485\columnwidth]{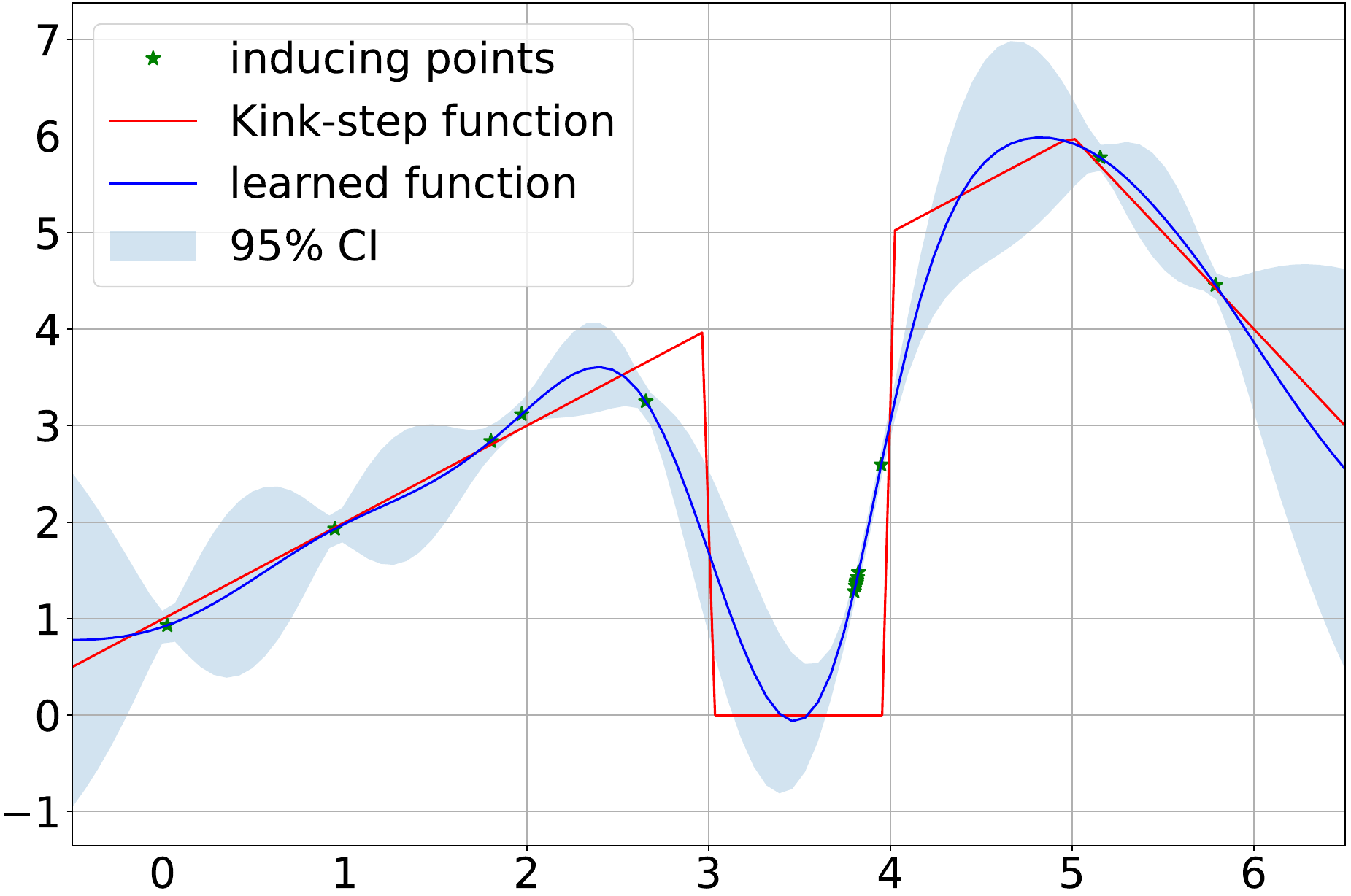}}
	\subfloat[JO-TGPSSM (MSE: 0.4049)]{\label{fig:tgp_ksfunc} \includegraphics[width=0.485\columnwidth]{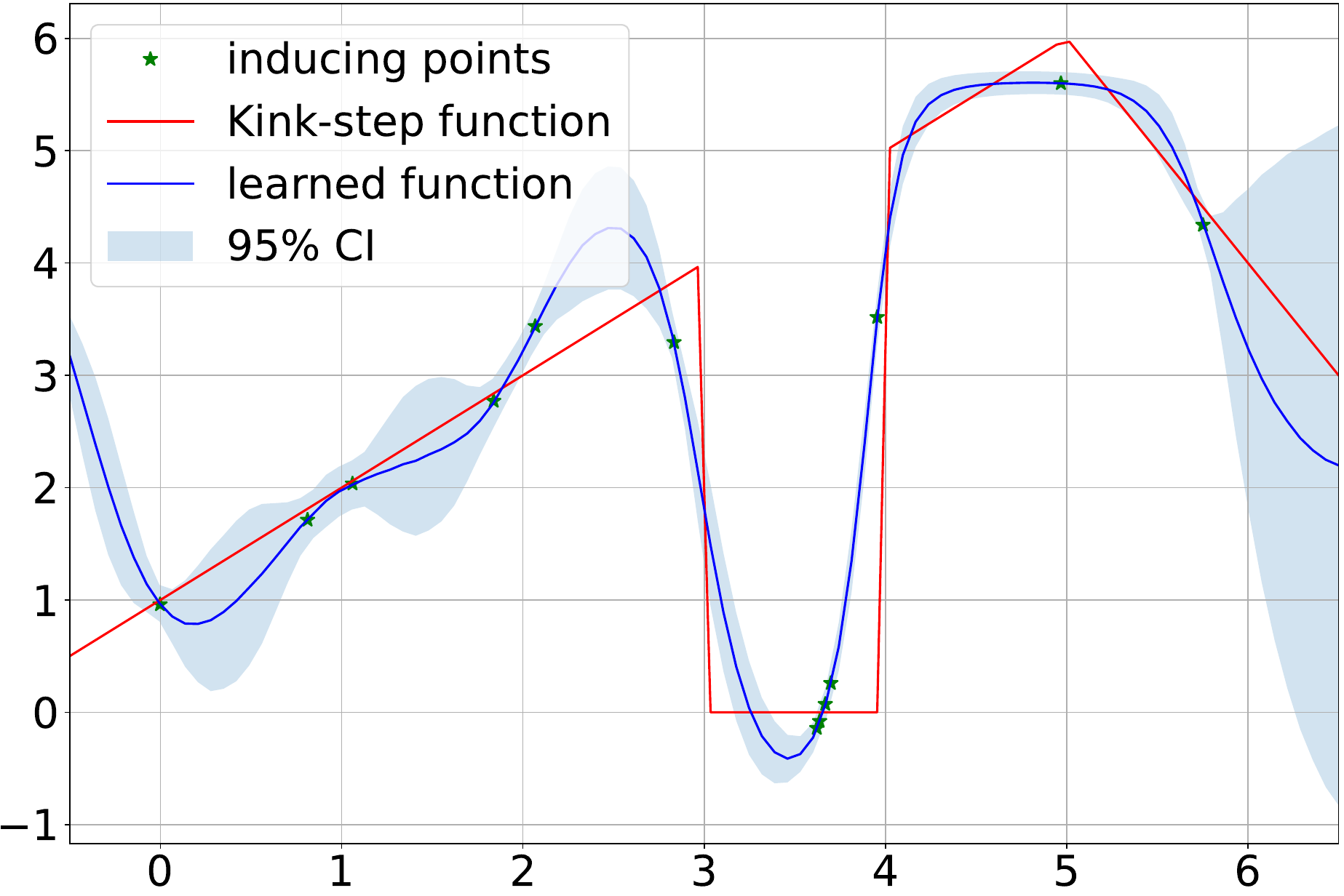}}
	\subfloat[CO-TGPSSM (MSE: {0.2319}) ]{\label{fig:cotgp_ksfunc} \includegraphics[width=0.485\columnwidth]{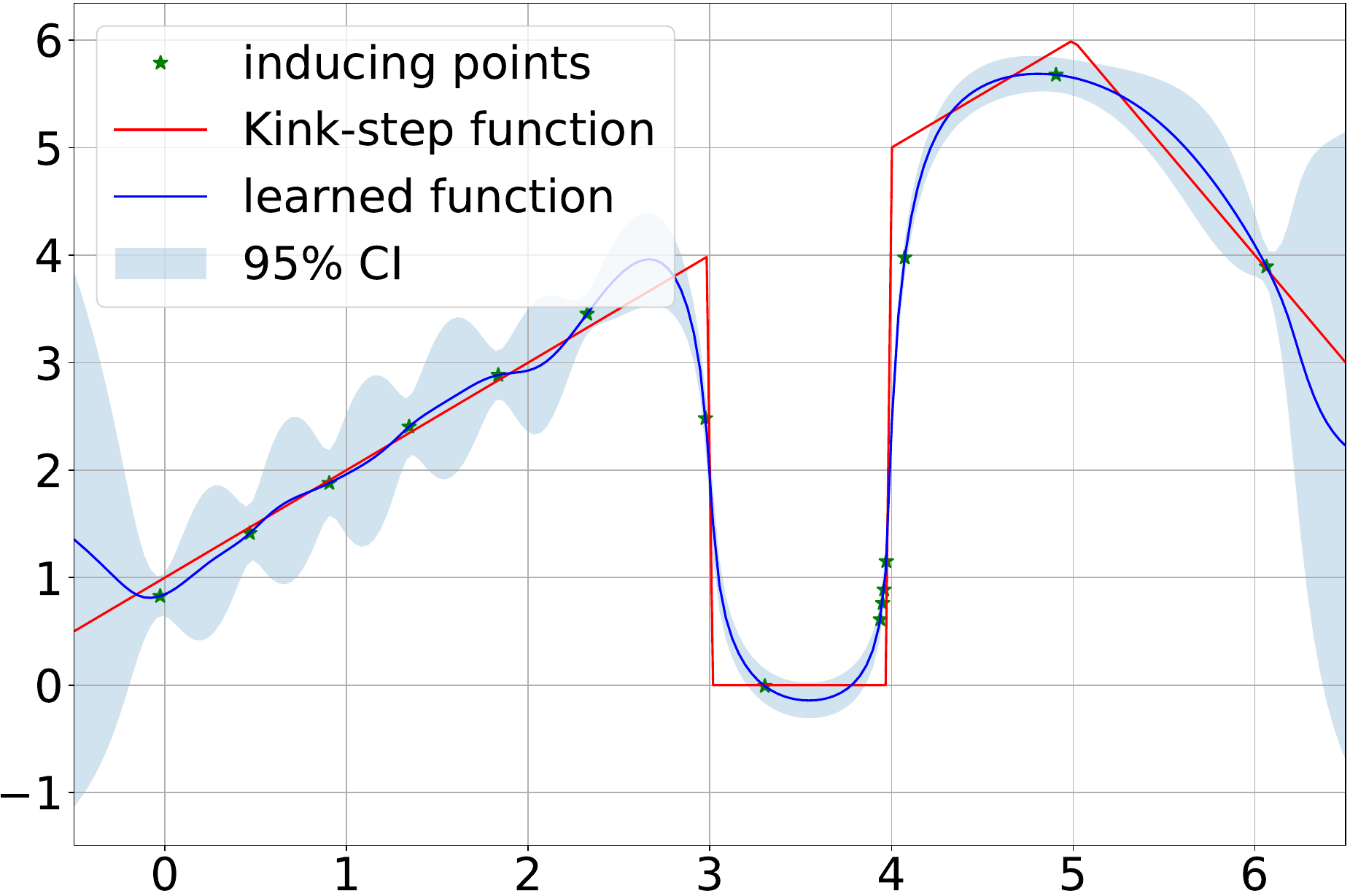}}
	\caption{Learning the ``kink-step'' dynamical system using GPSSMs and TGPSSMs.}
	\label{fig:1Ddataresults_SSMs}
\end{figure*}
%
\subsection{Learning the Dynamics} \label{subsec:1-Ddata}
This subsection aims to showcase the superior capability of TGPSSM for learning latent dynamics. To assess its performance, we employ two 1-D synthetic datasets, namely, the \textit{kink} function dataset and the \textit{kink-step} function dataset.  Details about the two datasets are given as follows.

\subsubsection{\textbf{Kink function dataset}}
This dataset is generated from a dynamical system that is described by Eq.~(\ref{eq:kink_function})
\begin{subequations}	
	\label{eq:kink_function}
	\begin{align}
		& \x_{t+1}  = \underbrace{0.8 + \left(\x_t+0.2\right)\left[1 - \frac{5}{1+ \exp(-2\x_t)}\right]}_{\triangleq \text{ ``kink function'' }f(\x_t)} + \mathbf{v}_{t},\\
		& \y_t = \x_t + \mathbf{e}_t, \  \mathbf{v}_{t} \sim \cN(0, 0.01), \  \mathbf{e}_t\sim \cN(0, 0.1),
	\end{align}
\end{subequations}
where the nonlinear, smooth and time-invariant transition function $f(\x_t)$ is called the ``kink'' function, as depicted in Fig.~\ref{fig:KinkFunction} (and Fig.~\ref{fig:KinkFunction_appx} in Appendix~\ref{appx:more_experimental_results}).
It is noted that this dynamical system has been widely used in the GPSSM literature to verify the goodness of the learned GP transition posterior \cite{lindinger2022laplace}.

\subsubsection{\textbf{Kink-step function dataset}} Similar to the kink function dataset, we create the kink-step function dataset using a modified dynamical system, where the nonlinear time-invariant transition function,  called the ``kink-step'' function, is a nonsmooth piecewise function with both a step function and a kink function, as described by the following equations and depicted in Fig.~\ref{fig:ksfunc} (and Fig.~\ref{fig:ksfunc_appx} in Appendix~\ref{appx:more_experimental_results}),
\begin{subequations}	
	\label{eq:kinkstep_function}
	\begin{align}
		& \x_{t+1}  \!=\! 
		\left\{ \!\!
		\begin{array}{ll}
			\x_t + 1 + \mathbf{v}_{t},        & \operatorname{if} \   \x_t<3  \ \text{or}  \  4 \le \x_t < 5  \\
			0 + \mathbf{v}_{t},                   & \operatorname{if} \   3 \le \x_t < 4 \\
			16 - 2\x_t  + \mathbf{v}_{t},   & \operatorname{if} \   \x_t \ge 5  \\
		\end{array}  \right. \\
		& \y_t = \x_t + \mathbf{e}_t, \quad  \mathbf{v}_{t} \sim \cN(0, 0.01), \quad  \mathbf{e}_t\sim \cN(0, 0.1). 
	\end{align}
\end{subequations}
This modified dataset is used to test the performance of the proposed TGPSSM when there are complex/sharp transitions in the underlying dynamics.

For the aforementioned two dynamical systems, we independently generate $30$ sequences of length $T=20$ noisy observations $\y_t$ to train the following SSMs.
\begin{itemize}
	\item \textbf{JO-TGPSSM}: The TGPSSM trained using the joint optimization (JO) framework (cf. Algorithm \ref{alg_tgpssm1})
	\item\textbf{CO-TGPSS}: The TGPSSM trained using the constrained optimization (CO) framework (cf. Algorithm \ref{alg_tgpssm2})
\end{itemize}
We compare our TGPSSMs with the following competitors.
\begin{itemize}
	\item \textbf{BS-GPSSM}: The GPSSM baseline with a Gaussian variational distribution for the latent states \cite{eleftheriadis2017identification, mchutchon2015nonlinear}.
	\item \textbf{JO-GPSSM}: The GPSSM with a non-Gaussian variational distribution for the latent states; it corresponds to the JO-TGPSSM without using normalizing flow in the transition function modeling.
	\item \textbf{CO-GPSSM}: The GPSSM with a non-Gaussian variational distribution for the latent states and trained using the constrained optimization framework; it corresponds to the CO-TGPSSM without using normalizing flow in the transition function modeling.
	\item \textbf{PRSSM}: The GPSSM trained using the NMF variational algorithm proposed in \cite{doerr2018probabilistic}.
	\item \textbf{ODGPSSM}: The output-dependent GPSSM trained using the NMF learning algorithm proposed in \cite{lin2022output}.
\end{itemize}
For all models, the size of the inducing points is set to $15$, and the GP model is equipped with the standard SE kernel. For TGPSSMs, we use a normalizing flow by concatenating $3$ SAL flows with $1$ Tanh flow (see Appendix \ref{subsec:appx_flow_table}), which only involves $16$ additional flow parameters. 
For the JO-TGPSSM and JO-GPSSM, we found that heuristically setting $\beta = T$ as the annealing factor of the data reconstruction term helps to achieve reasonably good performance. All the models are trained using the full gradient information and the same number of training epochs (1500 epochs). More details about the experimental configurations can be found in the accompanying source code. Next, we conduct a few ablation studies.
\begin{table}[t!]
	\centering
	\caption{Dynamic learning results of different probabilistic SSMs (MSE)}
	\setlength{\tabcolsep}{3mm}{
		\centering
		\begin{tabular}{ r  cc}
			\toprule
			Model & ``Kink'' function & ``Kink-step''  function  \\
			\midrule
			\textbf{BS-GPSSM}   &0.3059   &3.0663    \\
			\textbf{JO-GPSSM} &0.0364   & 2.2313    \\
			\textbf{CO-GPSSM}  &0.0410   &\textbf{0.3537}    \\
			\midrule
			\textbf{PRSSM}   &1.5605   & 3.4517   \\
			\textbf{ODGPSSM}  &1.8431   &3.2610    \\
			\midrule
			\textbf{JO-TGPSSM}  &\textbf{0.0361}  &  0.4049  \\
			\textbf{CO-TGPSSM} &\textbf{0.0351}  & \textbf{0.2319}  \\
			\bottomrule
		\end{tabular}
	}\label{tab:synthetic_dataset}
	\vspace{-.08in}
\end{table}
\textbf{TGPSSMs vs. GPSSMs}. Table \ref{tab:synthetic_dataset} (and Fig.~\ref{fig:1Ddataresults_SSMs_appx} in Appendix~\ref{appx:more_experimental_results}) reports the latent transition function learning performance for different SSMs in terms of the data fitting mean squared error (MSE). It can be observed that the performance of TGPSSMs and GPSSMs (with the exception of PRSSM and ODGPSSM) appears to be satisfactory when tested on the kink function dataset. This is because the underlying dynamic function is smooth and differentiable, see Fig.~\ref{fig:KinkFunction}, making it easy to accurately represent using the standard GP model.
%
Nevertheless, it should be noted that most GPSSMs struggle to accurately capture the complex and sharp dynamics of the kink-step function. In contrast, TGPSSMs demonstrate superior performance in this regard, owing to the exceptional transformation capability of the normalizing flow. This is evidenced by the illustrative results of JO-TGPSSM and CO-TGPSSM, which consistently outperform their corresponding GPSSMs, as depicted in Figure~\ref{fig:1Ddataresults_SSMs}.
In fact, the TGPSSM incorporates a normalizing flow that enables automatic transformation of the standard GP model to accurately represent the underlying dynamic function.  When the underlying system dynamics are simple, e.g., the differentiable and smooth kink function, shown in Figs.~\ref{fig:tgp_kinkfunc_} and \ref{fig:tgp_kinkfunc_subplot}, the normalizing flow in the TGPSSM behaves like an affine transformation, degrading the TGP into the standard GP model. While for the kink-step function, which is too complex for the standard GP to capture, the normalizing flow in the TGPSSM automatically transforms the GP model to fit the intricate dynamics, as revealed in Figs.~\ref{fig:tgp_ksfunc_} and \ref{fig:tgp_ksfunc_subplot}. These results empirically validate that the flexible TGPSSM is a unified model framework for the standard probabilistic SSMs, including the GPSSM.

It is worth noting that GPSSMs adopting the NMF algorithms, such as PRSSM and ODGPSSM, demonstrate consistently inferior performance.  This may be attributed to the fact that NMF algorithms typically require a larger number of training epochs. For instance, the NMF algorithm proposed in \cite{ialongo2019overcoming} required more than $30000$ training epochs to converge when applied to the kink function dataset. Additionally, existing NMF algorithms necessitate careful initialization of model parameters, making the training process less reliable \cite{lindinger2022laplace}.
In contrast, the TGPSSMs based on the MF algorithm require much simpler parameter initialization and are able to achieve the desired performance by using only $1500$ training epochs, thus demonstrating the superiority in terms of training and learning performance. 

%
\begin{figure}[t!]
	\centering
	\subfloat[Learning kink function]{\label{fig:tgp_kinkfunc_} \includegraphics[width=0.65\columnwidth]{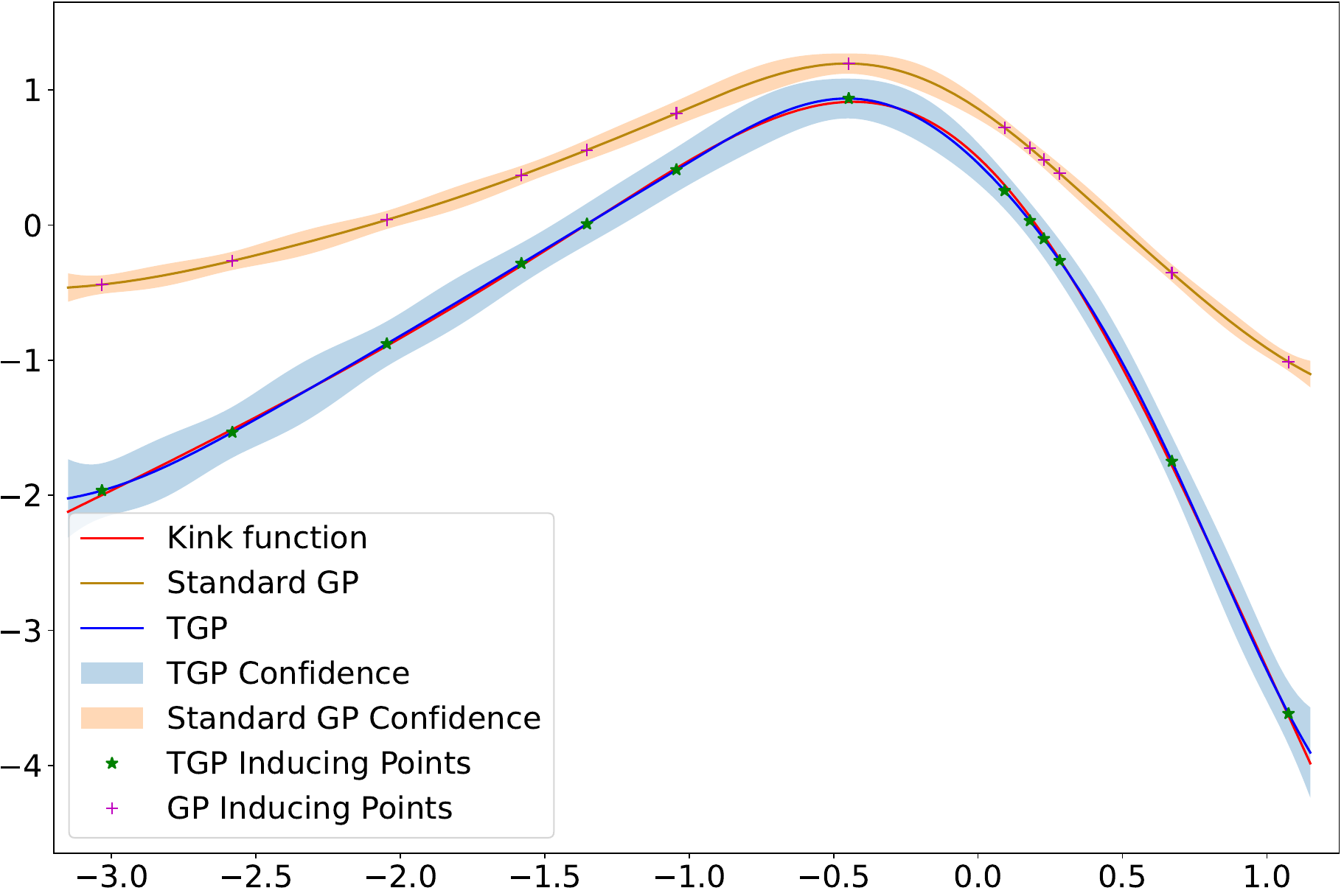}}
	\subfloat[Normalizing flow]{\label{fig:tgp_kinkfunc_subplot} \includegraphics[width=0.34\columnwidth, height=3.8cm]{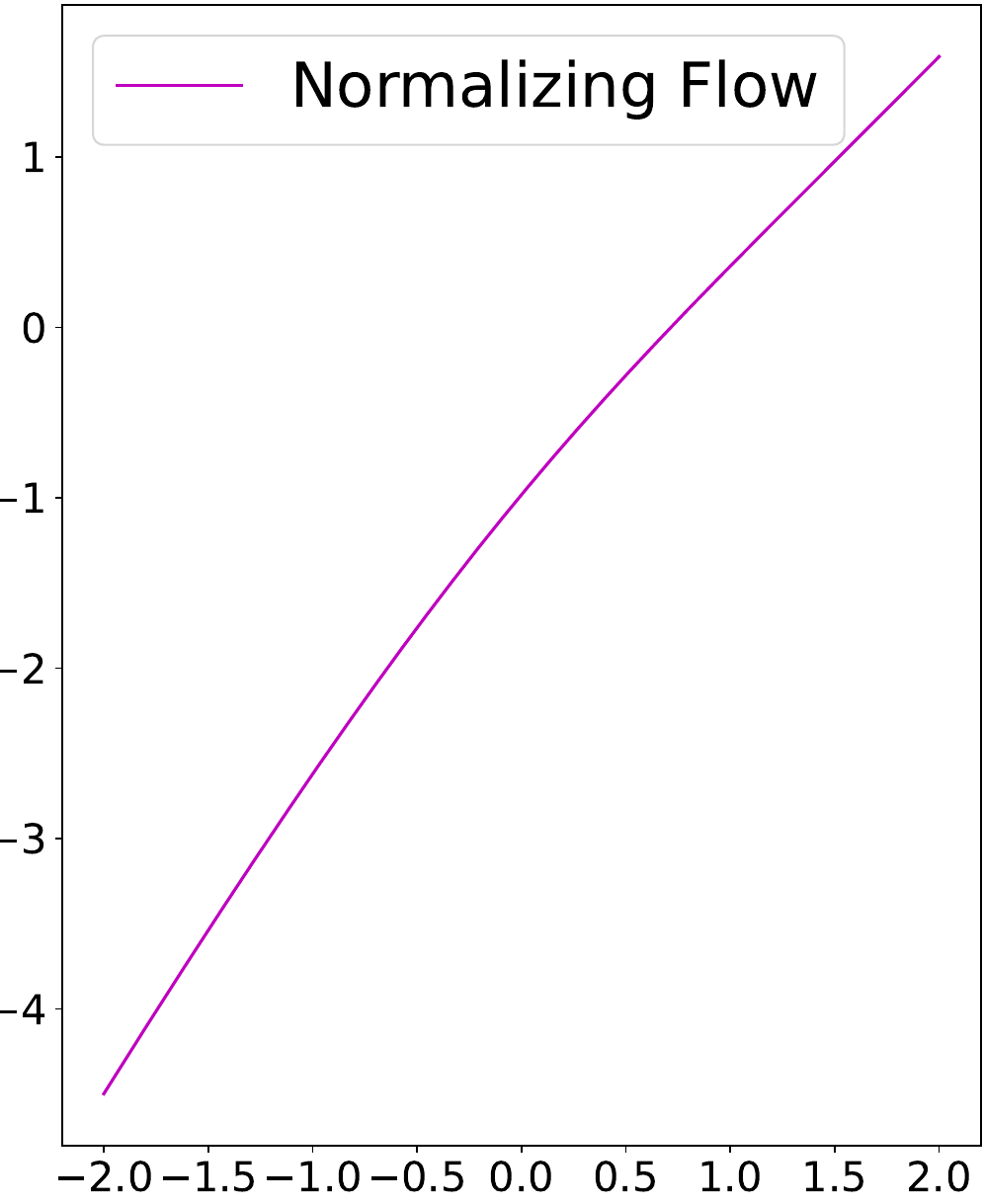}}
	
	\subfloat[Learning kink-step function]{\label{fig:tgp_ksfunc_} \includegraphics[width=0.65\columnwidth]{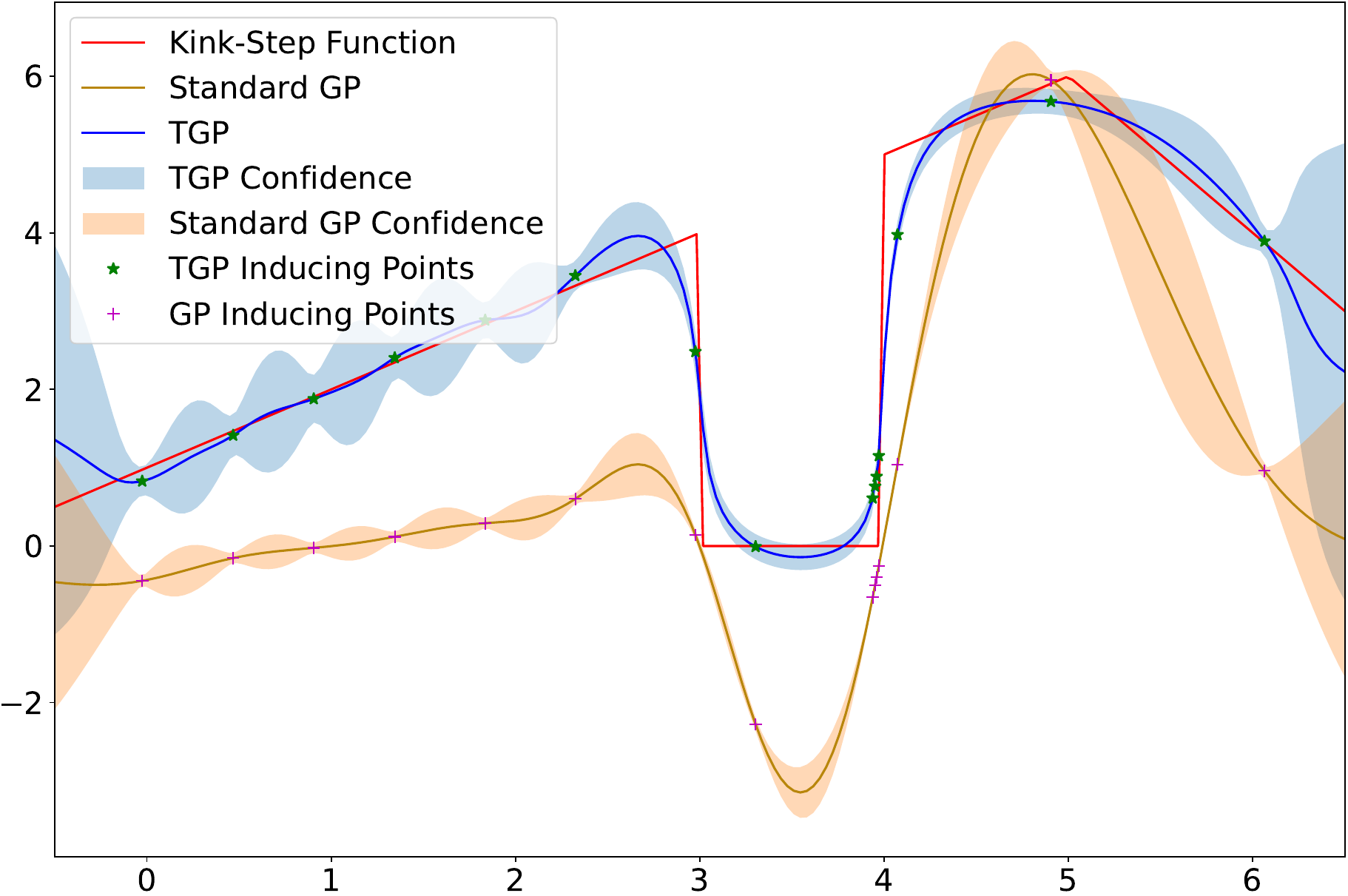}}
	\subfloat[Normalizing flow]{\label{fig:tgp_ksfunc_subplot} \includegraphics[width=0.33\columnwidth, height=3.8cm]{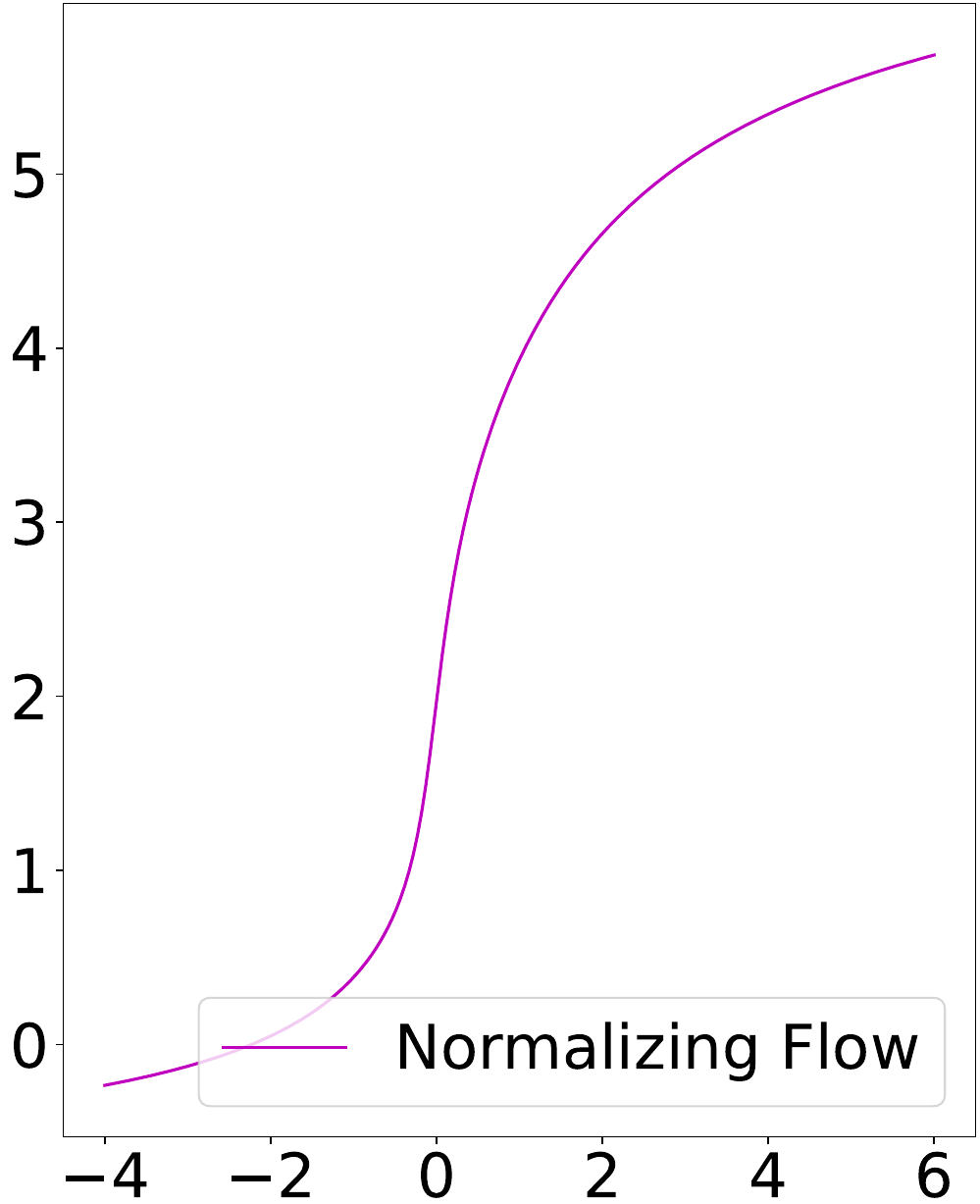}}
	\caption{Details about the TGP posteriors for the two transition functions.}
	\label{fig:TGP_posterior}
\end{figure}

\textbf{Constrained Optimization vs. Joint Optimization}. 
Upon comparing the performance of the joint optimization and constrained optimization frameworks in GPSSM and TGPSSM models, as reported in Table \ref{tab:synthetic_dataset} and Fig.~\ref{fig:1Ddataresults_SSMs}, it is evident that the algorithms utilizing the constrained optimization framework outperform their joint optimization counterparts consistently. Specifically, as illustrated in Fig.~\ref{fig:gp_ksfunc}, the JO-GPSSM algorithm is susceptible to being trapped in a local optimum during the learning phase, which ultimately leads to a learned latent dynamic that fails to capture the underlying system dynamics. On the contrary, the CO-GPSSM algorithm, as depicted in Fig.~\ref{fig:cogp_ksfunc}, effectively captures the underlying system dynamics to a reasonable extent, despite the limited model capacity of the standard GP with a plain SE kernel. This is a testament to the ability of CO-GPSSM to construct an informative state-space representation that serves as a sound foundation for learning the underlying system dynamics. Similar conclusions can be drawn from the TGPSSM cases, as evidenced by Fig.~\ref{fig:tgp_ksfunc} and Fig.~\ref{fig:cotgp_ksfunc}. These findings corroborate the efficacy of the constrained optimization framework across different models and provide further evidence supporting its effectiveness in enhancing the learning and inference performance of GPSSMs and TGPSSMs.

\textbf{Gaussian vs. Non-Gaussian}.
The joint Gaussian variational distribution for the latent states lacks the necessary flexibility to accurately approximate the non-Gaussian posterior distribution, resulting in a deteriorated learning performance, as can be observed by comparing the performance of the BS-GPSSM and JO-GPSSM presented in Table \ref{tab:synthetic_dataset}. In contrast, a more flexible, non-Gaussian variational distribution for the latent states allows GPSSMs to construct more informative state-space representations for latent dynamic learning. Although enriching the state-space representations might result in the training process getting trapped in local optima (cf.  the JO-GPSSM result, Fig.~\ref{fig:gp_ksfunc}), such an issue can be remedied using the constrained optimization framework (cf.  the CO-GPSSM result, Fig.~\ref{fig:cogp_ksfunc}).

\begin{table*}[ht!]
	\centering
	\caption{Prediction performance (RMSE) of the different models on the system identification datasets (standardized). Mean and standard deviation of the prediction results are shown across five seeds. The top three results are highlighted in bold.}
	\setlength{\tabcolsep}{2mm}{
		\centering
		\begin{tabular}{c|  c | ccccc }
			\toprule
			Model Description & Model & Actuator &  Ball Beam &  Drive &  Dryer &  Gas Furnace\\
			\midrule
			\multirow{3}{2.5cm}{\centering GPSSMs WITH MF ALGORITHM}
			& \textbf{BS-GPSSM}   
			&  $0.2115\pm 0.0294$
			&  $1.0025\pm  0.0997$ 
			&  $1.0580\pm 0.2941$ 
			&  $\textbf{0.5485}\pm\textbf{0.0928}$ 
			&  $0.4329\pm0.0583$  \\
			& \textbf{JO-GPSSM} 
			&  $0.1083\pm0.0225$ 
			&  $0.2851\pm0.0179$ 
			&  $0.9290\pm0.1591$ 
			&  ${0.7121}\pm{0.0855}$ 
			&  $0.4301\pm0.0562$  \\
			& \textbf{CO-GPSSM}  
			&  $\textbf{0.1067} \pm\textbf{0.0363}$
			&  $\textbf{0.1897}\pm\textbf{0.0614}$
			&  $0.9240\pm0.1096$
			&  ${0.6910}\pm {0.0416}$
			&  $\textbf{0.3718}\pm\textbf{0.0395}$ \\
			\midrule
			\multirow{2}{2.5cm}{\centering GPSSMs WITH NMF ALGORITHM} 
			& \textbf{PRSSM}  
			&  $0.4973\pm0.0503$
			&  $0.6191\pm0.0478$
			&  $0.9802 \pm0.1387$
			&  $\textbf{0.5783}  \pm \textbf{0.2305}$
			&  $0.6380\pm0.0384$  \\
			& \textbf{ODGPSSM}  
			&  $0.4814\pm0.0661$
			&  $0.5922\pm0.0597$
			&  $0.9853\pm0.1660$
			&  $\textbf{0.5535}\pm\textbf{0.2247}$
			&  $0.5072 \pm0.0823$ \\
			\midrule
			\multirow{1}{2.5cm}{\centering DSSMs} 
			& \textbf{DSSM} 
			&   $1.0498\pm0.2082$ 
			&   $1.2594\pm0.2463$ 
			&   $1.6879\pm0.2509$ 
			&   $1.7298\pm 0.2261$ 
			&   $1.6248 \pm0.2044$ \\
			\midrule
			\multirow{4}{2.5cm}{\centering \textbf{TGPSSMs} (\textbf{proposed})} 
			& \textbf{JO-TGPSSM}   
			& $\bm{0.1007} \pm \bm{0.0591}$
			& $0.2227\pm 0.0417$
			& $\textbf{0.7708}\pm\textbf{0.1302}$
			& $1.5288\pm0.1400$
			& $0.4156\pm0.0532$ \\
			& \textbf{CO-TGPSSM}  
			&  $\bm{0.1008}\pm\bm{0.0267}$
			&  $0.2371\pm0.0485$
			&  $\textbf{0.8425}\pm\textbf{0.1146}$
			&  $0.7197\pm0.0685$
			&  $\textbf{0.3757}\pm\textbf{0.0159}$ \\
			& \begin{tabular}[c]{@{}c@{}}\textbf{JO-TGPSSM}-NVP  \end{tabular}
			&  $0.1513\pm0.0394$
			&  $\textbf{0.1681}\pm\textbf{0.0512}$
			&  $0.9103\pm0.1570$
			&  $1.1145\pm0.1322$
			&  $0.4077\pm0.0401$ \\
			& \begin{tabular}[c]{@{}c@{}}\textbf{CO-TGPSSM}-NVP \end{tabular}
			&  $0.1366\pm0.0202$
			&  $\textbf{0.1707}\pm\textbf{0.0134}$
			&  $\textbf{0.8453}\pm\textbf{0.1249}$
			&  $0.7538\pm0.1803$
			&  $\textbf{0.3867}\pm\textbf{0.0383}$\\
			\bottomrule
	\end{tabular}}
	\label{tab:systemidentifcation}
\end{table*}
\subsection{Time Series Prediction}\label{subsec:SysIdenData}
This subsection demonstrates the series prediction performance of the proposed methods on five public real-world system identification datasets\footnote{\url{https://homes.esat.kuleuven.be/~smc/daisy/daisydata.html}}, which consist of one-dimensional time series of varying lengths between $296$ to $1024$ data points (see Table \ref{tab:SIDdataset} in Appendix~\ref{appx:more_experimental_results}). For comprehensive model comparisons, in addition to the SSMs considered in Section \ref{subsec:1-Ddata}, we also implement the DSSM\footnote{\url{https://github.com/guxd/deepHMM}} proposed in \cite{krishnan2017structured} and the TGPSSMs using the RealNVP flow with $J = 3$ coupling layers (denoted as the JO-TGPSSM-NVP and CO-TGPSSM-NVP). We assume that the one-dimensional observations $\y_t$, are governed by four-dimensional latent states $\x_t, \forall t$. The primary objective of the SSMs considered in this study is to accurately predict the observations based on the learned high-dimensional latent states and the associated nonlinear dynamics. All the SSMs are trained using standardized datasets, namely, the data are normalized to zero mean and unit variance based on the available training data, and the test data are scaled accordingly.
For (T)GPSSMs, the number of inducing points is commonly set to $20$, and the GP models are equipped with the standard SE kernel.
More experimental settings can be found in the accompanying source code.
Table \ref{tab:systemidentifcation} reports the prediction performance, where the root mean square error (RMSE) results correspond to predicting $20$ steps into the future from the end of the training sequence, and the top three are highlighted in bold.

The performance of the four model categories on the five datasets showcased in Table \ref{tab:systemidentifcation}, indicates that, in general, TGPSSMs yield the best prediction results (see Fig.~\ref{fig:sysID_dataset_appx} in Appendix~\ref{appx:more_experimental_results}), followed by GPSSMs with MF and NMF algorithms, while DSSMs display inferior performance. 
These results convincingly illustrate that the proposed TGPSSMs benefit from the flexible function prior and the associated variational learning algorithms, as is the case for TGPSSMs. Furthermore, the underwhelming performance of DSSMs, which employ deep neural networks to model transition functions, is unsurprising, given that these networks typically require a large amount of data to train their numerous parameters and thus perform poorly in low-data regimes. Conversely, the GPSSMs and TGPSSMs are more suitable for small datasets, owing to the  nature of their Bayesian nonparametric GP models.

The results in Table \ref{tab:systemidentifcation} also show that the TGPSSMs based on RealNVP achieve comparable performance to the TGPSSMs based on elementary flow compositions. From the perspective of model flexibility, TGPSSMs built with the more advanced RealNVP flow are more flexible and more generic to dynamical system modeling problems, particularly when no model prior knowledge can be provided. Interestingly, however, the TGPSSMs based on elementary flow compositions perform fairly well on the five real-world system identification datasets,  suggesting that a GP combined with simple flow transformations can construct a flexible TGP prior that satisfies the requirements of complex dynamic system learning and inference.

Finally, it is worth noting that the BS-GPSSM achieves the best performance on the \textit{Dryer} dataset. This suggests that the latent states in the \textit{Dryer} dataset exhibit a more ``Gaussian'' behavior, aligning well with the assumed joint Gaussian variational distribution of the latent states. Increasing the model flexibility in this case complicates model training and degrades model performance as a consequence. Specifically, as demonstrated in Table \ref{tab:systemidentifcation}, the performance of the JO-TGPSSMs significantly deviates from that of the BS-GPSSM. However, it is noteworthy that in this case, the SSMs with a constrained optimization framework (CO-TGPSSMs and CO-GPSSM) still perform well, approaching the performance of the BS-GPSSM. This further validates the effectiveness of the constrained optimization framework in training the GPSSM model parameters.

\subsection{Estimating the Latent States} \label{subsec:Lorenz_attractor}          
Lastly, we demonstrate the model efficacy of the proposed TGPSSM in the state inference task. Concretely, we employ the three-dimensional Lorenz system \cite{revach2022kalmannet}, which mathematically can be described as follows:
\begin{subequations}
	\label{eq:Lorenz_system}
	\begin{align}
		& \x_{t+1}  = \bm{F}(\x_t) \cdot \x_t + \mathbf{v}_t,  & \mathbf{v}_t \sim \cN(\bm{0},  \ 0.0015  \cdot \bm{I}_3), \\
		& \y_t = \bm{I}_{3} \cdot \x_t  + \mathbf{e}_t,        & \mathbf{e}_t \sim \cN(\bm{0},  \ 0.1 \cdot \bm{I}_3), \ \ \ \ 
	\end{align}
\end{subequations}
where the definition and detailed computations of the (Jacobian) transition matrix $\bm{F}(\x_t)$ can be found in Eq.~(20) of \cite{revach2022kalmannet}.
We compare the CO-GPSSM and CO-TGPSSM with the classic EKF because it was shown in \cite{revach2022kalmannet} that when model-based filtering algorithms know exactly the discretized Lorenz system information, the EKF can achieve the best state inference results. We thus use the EKF as a practical lower bound for evaluating the MSE of state estimation.  In addition, we also compare the scenario where the EKF has a model mismatch, referred to as EKF-M. Specifically, we assume that during filtering, the transition function used in the EKF is corrupted by small perturbations and has the following mismatched form: 
\begin{equation}
    f(\x_t) = \bm{F}(\x_t) \cdot \x_t + 0.2.
\end{equation} 
For learning and inference in both the CO-GPSSM and CO-TGPSSM, we generate a sequence of length $T = 2000$. 

The state inference results are presented in Table \ref{tab:Lorenz_results} (and Fig.~\ref{fig:Lorenz_dataresults_SSMs} in Appendix~\ref{appx:more_experimental_results}). Notably, the state inference performance of (T)GPSSMs is comparable to that of the EKF in terms of state-fitting MSE, despite being trained solely on noisy measurements without any physical model knowledge. Furthermore, the significant deviations observed in the state estimation due to the slight model mismatch in the EKF-M highlights the vital role that physical model knowledge plays in state estimation.  Therefore,  to further improve the state inference performance in the data-driven (T)GPSSMs, we could explore the integration of physical model knowledge into the learned latent space representation of the (T)GPSSMs, which could be a fruitful avenue for future research.
\begin{table}[t!]
	\centering
	\caption{State estimation performance (MSE) of different models 
	} 
	\setlength{\tabcolsep}{2.45mm}{
		\begin{tabular}{c c c c c }
			\toprule
			{Observations} 
			& {EKF} 
                & {EKF-M}
			& {CO-GPSSM}
			& {CO-TGPSSM}  \\
			\midrule
			0.1012 &  0.0237  & 3.2066  &  0.0983  &  0.0895\\
			\bottomrule
		\end{tabular}
	}
	\label{tab:Lorenz_results}
\end{table}

\section{Conclusion}\label{sec:conclusion} 	  
In this paper, a flexible and unified probabilistic SSM called the TGPSSM has been proposed. By leveraging the normalizing flow technique, TGPSSM enriches the GP priors in the standard GPSSM, making the TGPSSM more flexible and expressive for modeling complex dynamical systems.  We also present an efficient variational learning algorithm that is superior in modeling flexibility and interpretability, enabling scalable learning and inference. Furthermore, a constrained optimization framework is integrated into the proposed algorithm to further enhance the state-space representation capabilities of TGPSSMs and optimize the hyperparameters. Experimental results on various test datasets demonstrate that the proposed TGPSSM, empowered by the proposed variational learning algorithm, can improve complex dynamical system learning and inference performance compared to state-of-the-art methods.

\section{Proofs} \label{sec:proofs}
\subsection{Augmented TGP Prior: Derivation of Eq.~(\ref{eq:joint_TGP})} \label{appedx:derivation_joint_tgp}
\begin{figure*}
    \begin{equation}
         p(\tF, \tU) = p(\vf, \vu) \prod_{j=1}^{J-1} \left|\operatorname{det} \left( \begin{bmatrix}
            \overbrace{\frac{\partial \mathbb{G}_{\theta_{j}}\left( \G_{\theta_{j-1}}\left( \ldots \G_{\theta_0}(\vf)\ldots \right) \right)}{\partial \G_{\theta_{j-1}}\left( \ldots \G_{\theta_0}(\vf)\ldots \right)  } }^{A} & \overbrace{\frac{\partial \mathbb{G}_{\theta_{j}}\left( \G_{\theta_{j-1}}\left( \ldots \G_{\theta_0}(\vf)\ldots \right) \right)}{\partial \G_{\theta_{j-1}}\left( \ldots \G_{\theta_0}(\vu)\ldots \right) }}^B\\
            \underbrace{\frac{\partial \mathbb{G}_{\theta_{j}}\left( \G_{\theta_{j-1}}\left( \ldots \G_{\theta_0}(\vu)\ldots \right) \right)}{\partial \G_{\theta_{j-1}}\left( \ldots \G_{\theta_0}(\vf)\ldots \right) } }_C& \underbrace{\frac{\partial \mathbb{G}_{\theta_{j}}\left( \G_{\theta_{j-1}}\left( \ldots \G_{\theta_0}(\vu)\ldots \right) \right)}{\partial \G_{\theta_{j-1}}\left( \ldots \G_{\theta_0}(\vu)\ldots \right) }}_D
        \end{bmatrix}\right) \right|^{-1},
        \label{eq:joint_tgp_appedx} 
    \end{equation}
    \hrulefill
\end{figure*}
Similar derivations can be found in \cite{maronas2021transforming}. For ease of reference, we also present the proof here.  By definition, it is easy to write down $p(\tilde{\vf}, \tilde{\vu})$, see Eq.~(\ref{eq:joint_tgp_appedx}) in next page, and $p(\tilde{\vu})$,
\begin{equation}
    p(\tU) = p(\vu) \underbrace{\prod_{j=1}^{J-1} \left|\operatorname{det} \frac{\partial \mathbb{G}_{\theta_{j}}\left( \G_{\theta_{j-1}}\left( \ldots \G_{\theta_0}(\vu)\ldots \right) \right)}{\partial \G_{\theta_{j-1}}\left( \ldots \G_{\theta_0}(\vu)\ldots \right) } \right|^{-1}}_{\J_{\u}}.		
    \label{eq:u_tgp_appedx}
\end{equation}
Since 
\begin{subequations}
\label{eq:bayes_deter}
	\begin{align}
		& p(\tilde{\mathbf F} \vert \tilde{\mathbf U}) = \frac{p(\tilde{\mathbf F}, \tilde{\mathbf U})}{p(\tilde{\mathbf U})}, \\
		& \operatorname{det}\left(\begin{array}{ll}
			A & B \\
			C & D
		\end{array}\right)=\operatorname{det}\left(A-B D^{-1} C\right) \operatorname{det}(D),
	\end{align}
\end{subequations}
by combining Eqs.~(\ref{eq:joint_tgp_appedx}), (\ref{eq:u_tgp_appedx}) and (\ref{eq:bayes_deter}), we have
\begin{equation}
	\begin{aligned}
		& p(\tF \vert \tU) = p(\vf \vert \vu) \prod_{j=1}^{J-1} \left|\operatorname{det} \left( A - BD^{-1} C\right)\right|^{-1}\\
		& \quad = p(\vf \vert \vu) \prod_{j=1}^{J-1} \left|\operatorname{det} A\right|^{-1} \quad  (B=0, C = 0 \text{ in Eq.~(\ref{eq:joint_tgp_appedx})})\\
		& \quad = p(\vf \vert \vu) \underbrace{\prod_{j=1}^{J-1} \left| \operatorname{det} \frac{\partial \mathbb{G}_{\theta_{j}}\left( \G_{\theta_{j-1}}\left( \ldots \G_{\theta_0}(\vf)\ldots \right) \right)}{\partial \G_{\theta_{j-1}}\left( \ldots \G_{\theta_0}(\vf)\ldots \right)  }  \right|^{-1}}_{\J_{\f}}.
		\label{eq:appxdcondition_u}
	\end{aligned}
\end{equation}
Therefore, the augmented TGP prior $p(\tF, \tU)$ is 
\begin{equation}
	\begin{aligned}
		p(\tF, \tU) =p(\tF \vert \tU) p(\tU) = \underset{p(\tilde{\f}_{1:T} \vert \tilde{\u}_{1:M})}{\underbrace{p(\vf \vert \vu) \cdot \J_{\f}}} \cdot \underset{p(\tilde{\u}_{1:M})}{\underbrace{p(\vu)\cdot \J_{\u}}}.
	\end{aligned}
\end{equation}

\subsection{Proof of Theorem \ref{theorem:ips_tgp}}
\label{appedx:proof_thm2}
\begin{itemize}[leftmargin=*]
    \item We first show that if $\vu$ is sufficient for $\vf$, then $\vu$ is sufficient for $\tilde{\vf} = \G_{\btheta_{F}}(\vf)$.  
We denote that $\vf \sim \cN(\bm{\mu}, \bm{K})$. 
Since $\vu$ is sufficient for $\vf$, according to Fisher–Neyman factorization theorem \cite{casella2021statistical}, 
 there exist non-negative functions $\gamma$ and $\Gamma$, such that 
\begin{equation}
    p(\vf; \bm{\mu}, \bm{K}) = \gamma( \vu; \bm{\mu}, \bm{K}) \Gamma(\vf)
\end{equation} 
where $\gamma$ depends on $(\bm{\mu}, \bm{K})$ and the sufficient statistic $\vu$, while $\Gamma$ does not depend on $(\bm{\mu}, \bm{K})$.
Note that $\vf = \G_{\btheta_{F}}^{-1}(\G_{\btheta_{F}}(\vf)) = \G_{\btheta_{F}}^{-1}(\tilde{\vf})$.  Let 
\begin{equation}
    p( \vf;  \bm{\mu}, \bm{K})  =  p\left( \G_{\btheta_{F}}^{-1}(\tilde{\vf});   \bm{\mu}, \bm{K}\right) \triangleq  \pi( \tilde{\vf} ;   \bm{\mu}, \bm{K}).
    \label{eq:th1_1}
\end{equation}
Moreover, since
\begin{equation}
    \begin{aligned}
        p\left(\G_{\btheta_{F}}^{-1}(\tilde{\vf});   \bm{\mu}, \bm{K}\right) & =  \gamma(\vu;  \bm{\mu}, \bm{K}) \Gamma(\G_{\btheta_{F}}^{-1}(\tilde{\vf})) \\
        & \triangleq \gamma(\vu;  \bm{\mu}, \bm{K})   \Gamma^*(\tF), 
    \end{aligned}
    \label{eq:th1_2}
\end{equation}
we have
\begin{equation}
    \pi( \tilde{\vf} ;   \bm{\mu}, \bm{K})= \gamma(\vu;  \bm{\mu}, \bm{K})  \Gamma^*(\tF),
\end{equation}
which implies that $\vu$ is sufficient for $\tilde{\vf}$ according to Fisher–Neyman factorization theorem \cite{casella2021statistical}. 

\item Then following the similar reasoning, we can prove that if $\vu$ is sufficient statistic for $\tilde{\vf}$, then $\tilde{\vu}$ is sufficient statistic for $\tilde{\vf}$.

\item Together step 1) and 2), we can conclude that  if $\vu$ is sufficient for $\vf$, then $\tilde{\vu}$ is sufficient statistic for $\tilde{\vf}$. 	\qed
\end{itemize}

\subsection{Proof of Proposition \ref{proposition:optimal_Markov_qx_maintext} and Theorem \ref{thm:ELBO_mf}} 
\label{ELBO_1}
\subsubsection{The ELBO} 
    \begin{equation}
	\label{eq:mf_elbo}
	\begin{aligned}
		&\operatorname{ELBO}  =  \mathbb{E}_{q(\vx, \tF, \tU )} \left[\log \frac{p(\tF, \tU, \vx, \vy)} {q(\vx, \tF, \tU) }\right]\\
		& = \mathbb{E} \left[ \log \frac{ p\left(\mathbf{x}_{0}\right) \cdot p(\vu) \cancel{ \J_{\u} \cdot p(\vf \vert \vu)  \J_{\f} }     \prod_{t=1}^{T}  p\left(\mathbf{y}_{t} \vert \mathbf{x}_{t}\right) p\left(\mathbf{x}_{t} \vert {\tf}_{t}\right) }{ q(\x_0) \cdot q(\vu) \cancel{ \J_\u  \cdot p(\vf \vert \vu) \J_\f } \prod_{t= 1}^T q(\x_{t} \vert \x_{t-1} ) }  \right] \\
		& = \mathbb{E}_{q(\vx, \tF, \tU )} \left[ \log \frac{p(\vu) p(\x_0) \prod_{t= 1}^T p(\y_t \vert \x_t) p(\x_t \vert {\tf_{t} })}{q(\vu) q(\x_0)  \prod_{t= 1}^T q(\x_{t} \vert \x_{t-1} )}  \right] \\
		& =\underbrace{\int_{\x_{0:T}} q(\x_{0:T}) \log\prod_{t= 1}^T p(\y_t \vert \x_t) }_{\text{term 1}} + \underbrace{\int_{\x_0} q(\x_0) \log \frac{p(\x_0)}{q(\x_0)}}_{\text{term 2}} \\
            & \quad + \underbrace{\int_{\{\tU, \tF\}} q(\tU) p(\tF \vert \tU)   \log \frac{p(\vu)}{q(\vu)}}_{\text{term 3}} \\
		& \quad + \underbrace{\int_{\x_{0:T}} q(\x_0) q(\x_{1:T} \vert \x_0) \log \frac{1}{q(\x_{1:T} \vert \x_0 ) } }_{\text{term 4}} \\
            & \quad + \underbrace{\int_{\{ \x_{0:T}, \tU, \tF \}} q(\x_{0:T}) q(\tU) p(\tF \vert \tU) \log \prod_{t=1}^T p(\x_t \vert \tf_t )}_{\text{term 5}} 
	\end{aligned}
\end{equation}
\hrulefill

\subsubsection{Proof of Proposition \ref{proposition:optimal_Markov_qx_maintext}}
\label{subsec:proof_propsition1_appx}
\begin{proof}
 The distribution $q(\x_{0:T})$ that maximizes the ELBO corresponds to a stationary point of the ELBO. More specifically, the ELBO achieves a stationary point w.r.t. $q(\x_{0:T})$ if and only if the distribution satisfies the Euler-Lagrange equation \cite{frigola2015bayesian}, i.e., 
 \begin{equation}
    \begin{aligned}
        & 0 = \frac{\partial }{\partial q(\x_{0:T})} \left\{ q(\x_{0:T}) \log \prod_{t= 1}^T p(\y_t \vert \x_t)  +  q(\x_0) \log p(\x_0) \right. \\
        & \left.  - q(\x_{0:T})\log q(\x_{0:T}) + q(\x_{0:T}) \mathbb{E}_{q(\tU) p(\tF \vert \tU)} \left[  \log \prod_{t=1}^T p(\x_t \vert \tf_t ) \right] \right\},  \\
    \end{aligned}
 \end{equation}
 which implies that the optimal distribution $q^*(\x_{0:T})$ satisfies
 \begin{equation}
    \begin{aligned}
        0 =&  \log \prod_{t= 1}^Tp(\y_t \vert \x_t)   + \log p(\x_0)  - \log q^*(\x_{0:T})  - 1  \\
        &  +  \sum_{t= 1}^T  \underbrace{  \mathbb{E}_{q(\vu)} \left[  \mathbb{E}_{p(\f_t \vert \vu, \x_{t-1})} \left[ \log p(\x_t \vert \G_{\btheta_{F}}(\f_t)) \right] \right] }_{\triangleq  \Psi(\x_{t-1:t})} 
    \end{aligned}
 	\label{eq:stationary_points_qx}
 \end{equation}
where the function $\Psi(\x_{t-1:t})$ only depends on $\x_{t-1}$ and $\x_t$ after marginalizing out $\f_t$ and $\vu$ (Generally, there is no analytical solution for integrating out  $\f_t$ because of the nonlinearity of $\G_{\btheta_{F}}(\cdot)$).  Taking exponentiation  on both sides of Eq.~(\ref{eq:stationary_points_qx}), we get the optimal distribution $q^*(\x_{0:T})$
\begin{equation}
	q^*(\x_{0:T}) \propto p(\x_0) \prod_{t= 1}^T p(\y_t \vert \x_t ) \exp\left[\Psi(\x_{t-1:t})\right],
	\label{eq:simpler_SSM}
\end{equation}
where $\exp\left[\Psi(\x_{t-1:t})\right]$ can be interpreted as an implicit transition function of the Markov SSM, and the optimal distribution of $q^*(\x_{0:T})$ is the corresponding smoothing distribution, which is Markov-structured.
\end{proof}
\begin{remark} \label{remark:non-closed-form}
In the GPSSM work proposed in \cite{frigola2014variational}, the analytical marginalization of GP function values $\f_t$ and $\vu$ enables the determination of the optimal variational distribution $q^*(\x_{0:T})$ as a smoothing distribution of a simpler SSM, which can be numerically obtained using particle filters. However, in TGPSSM, the nonlinearity of $\G_{\btheta_{F}}(\cdot)$ prevents the closed-form integration over $\tf_{t}$, as demonstrated in Eq.(\ref{eq:stationary_points_qx}), thereby making the Markovian SSM in Eq.(\ref{eq:simpler_SSM}) intractable. Consequently, particle filters, as utilized in \cite{frigola2014variational}, cannot be directly applied to TGPSSM for variational learning and inference. While it is possible to simultaneously sample $\vu$, $\tf_t$, and $\x_{0:T}$ from Eq. (\ref{eq:simpler_SSM}) using particle filters, this way does not yield the computational efficiency benefits of the variational inference method.
\end{remark}
\subsubsection{Proof of Theorem \ref{thm:ELBO_mf}}
\begin{proof} Next, we detail the five terms in the ELBO (cf. Eq.~(\ref{eq:mf_elbo})).
	\begin{itemize}[leftmargin=*]
		\item \textbf{Term 1}
		\begin{equation}
                \begin{aligned}
                    \text{term 1} & = \E_{q(\x_{0:T})} \left[ \sum_{t=1}^T \log p(\y_t \vert \x_t)\right]\\
                    & = \sum_{t=1}^T \E_{q(\x_{t-1})q(\x_{t} \mid \x_{t-1})} \left[\log p(\y_t \vert \x_t)\right],
                \end{aligned}
			\nonumber
		\end{equation}
		where $p(\y_t \vert  \x_t) = \cN(\y_t \vert \bm{C} \x_t, \bm{R})$ is Gaussian, and $q(\x_t \vert \x_{t-1}) =  \cN(\x_t \vert \Phi_{\phi}(\x_{t-1}, \y_{1:T}), \bm{\Sigma}_{\phi}(\x_{t-1}, \y_{1:T}))$ (see Eq.~(\ref{eq:MF_qx})),  thus term 1 exists closed-form solution (cf. the result in Appendix~\ref{appedx:elgl}),
		\begin{equation}
			\begin{aligned}
				\text{term 1} & = \sum_{t=1}^T \mathbb{E}_{q(\x_{t-1})} \left[  \E_{q(\x_{t} \vert \x_{t-1})} \left[\log p(\y_t \vert \x_t)\right] \right]\\
				&=  \sum_{t=1}^T \mathbb{E}_{q(\x_{t-1})} \left[ \log \mathcal{N}(\y_t\mid \bm{C}  \Phi_{\x_t}, \bm{R}) \right. \\
                    & \left. \qquad \qquad - \frac{1}{2}\operatorname{tr}\left[ \bm{R}^{-1} (\bm{C \Sigma_{\x_t} C^\top})\right] \right],
			\end{aligned}
		\end{equation}
		where for notation brevity, $\Phi_{\x_t}\triangleq  \Phi_{\phi}(\x_{t-1}, \y_{1:T})$ and $\bm{\Sigma}_{\x_t} \triangleq \bm{\Sigma}_{\phi}(\x_{t-1}, \y_{1:T})$.
		Term 1 represents the overall data fitting performance averaged over all latent states generated from their joint distribution $q(\x_{1:T})$, which encourages accurate reconstruction of the observations. 
		\item \textbf{Term 2}: We assume $q(\x_0) = \cN(\x_0 \vert {\m}_{0}, \mathbf{L}_0 \mathbf{L}_0^\top )$ and $p(\x_0) = \cN(\x_0 \vert \mathbf{0}, \mathbf{I})$ are both Gaussian, thus the term 2 has closed-form solution:
		\begin{align}
			& \text{term 2} = - \E_{q(\x_0)} \log \frac{q(\x_0)}{p(\x_0)} \nonumber\\
			& \quad = - \operatorname{KL}(q(\x_0) \| p(\x_0)) \nonumber\\
			& \quad = -\frac{1}{2}\left[ \left({\m}_{0}^{\top} {\m}_{0} \right)+\operatorname{tr}\left(\mathbf{L}_0 \mathbf{L}_0^\top  \right)-\log \left|\mathbf{L}_0 \mathbf{L}_0^\top \right|-d_x\right] \nonumber\\ 
			& \quad = -\frac{1}{2}\left[ \left(\mathbf{m}_{0}^{\top} \mathbf{m}_{0} \right)+\operatorname{tr}\left(\mathbf{L}_0 \mathbf{L}_0^\top  \right)- 2\log \left|\mathbf{L}_0 \right|-d_x\right]
		\end{align}
		Term 2 represents a regularization term for $q(\x_0)$, which encourages $q(\x_0)$ not staying too ``far away'' from $p(\x_0)$.  
		\item \textbf{Term 3}: 
		\begin{equation}
			\begin{aligned}
				& \text{term 3} = \int_{\{\tU, \tF\}} q(\tU) p(\tF \vert \tU)   \log \frac{p(\vu)}{q(\vu)} \  \mathrm{d} \tF \mathrm{d} \tU \\
				& \quad = \int_{ \tU } q(\tU) \log \frac{p(\vu)}{q(\vu)} \  \mathrm{d} \tU \qquad \text{ (integrate out $\tF$ )}\\
				& \quad = \int_{ \vu } q(\vu) \log \frac{p(\vu)}{q(\vu)}  \  \mathrm{d} \vu \qquad \text{ (LOTUS rule \cite{papamakarios2021normalizing}}\\
				& \quad = -\operatorname{KL}\left[ q(\vu) \| p(\vu) \right]\\
				& \quad = - \frac{1}{2}\left[{\m}^{\top} \bm{K}_{\vz,\vz}^{-1}{\m}+\operatorname{tr}\left(\bm{K}_{\vz,\vz}^{-1} \mathbf{S} \right)-\log \frac{\left|  \mathbf{S}  \right|}{\left|\bm{K}_{\vz,\vz}\right|}-M d_x\right]  
			\end{aligned}
		\end{equation}
		where $p(\vu) = \cN(\vu \mid \bm{0}, \bm{K}_{\vz,\vz})$ and $q(\vu) = \cN(\vu \vert \mathbf{m}, \mathbf{S})$. Term 3 represents a regularization term for the GP transition, which encourages $q(\vu)$ not staying too ``far away'' from the prior $p(\vu)$. 
		\item \textbf{Term 4}
		\begin{align}
			\!\!\!\! & \text{term 4}  = -\mathbb{E}_{q(\x_0)} \left[   \mathbb{E}_{q(\x_{1:T} \vert \x_0 )} \left( \log q(\x_{1:T} \vert \x_0 )  \right) \right] \nonumber \\
			& \ = -\mathbb{E}_{q(\x_0)}
			\left[ \int_{\x_{1:T}} \prod_{t= 1}^T q(\x_t \vert \x_{t-1}) \left(\sum_{t= 1}^T \log  q(\x_t \vert \x_{t-1})  \right)  \mathrm{d}  \x_{1:T} \right] \nonumber \\			
			& \ =  \sum_{t=1}^T	\mathbb{E}_{q(\x_{t-1})} \left[ \underbrace{ - \mathbb{E}_{q(\x_{t} \vert \x_{t-1} )} \left( \log q(\x_t \vert \x_{t-1})  \right) }_{\text{entropy}} \right] \nonumber \\
			& \ =  \sum_{t=1}^T \mathbb{E}_{q(\x_{t-1})}\left[ \frac{d_x}{2} \log (2 \pi)+\frac{1}{2} \log \left| \bm{\Sigma}_{\x_t} \right|+\frac{1}{2} d_x  \right] 
			\label{eq:h_x}
		\end{align}
		Term 4 is the differential entropy term of the latent state trajectory.  From Eq.~(\ref{eq:h_x}), we can see that maximizing the ELBO essentially encourages ``stretching'' every $\x_t$ so that the approximated smoothing distribution over the state trajectories $q(\x_{1:T} \vert \x_0)$ will not be overly tight. 
		%
		\item \textbf{Term 5}
            \begin{subequations}
                \begin{align}
    			& \text{term 5} = \mathbb{E}_{ q( \x_{0:T}, \tU, \tF) } \left[ \sum_{t=1}^T  \log  p\left(\x_t \vert \tf_t \right)  \right] \nonumber\\
    			& \ = \mathbb{E}_{ q( \x_{0:T}, \vu, \vf) } \left[ \sum_{t=1}^T  \log  p\left(\x_t \vert \G_{\btheta_{F}}(\f_t) \right)  \right]  ~ \text{ (LOTUS rule \cite{papamakarios2021normalizing}} \nonumber\\
    			& \ = \sum_{t=1}^T \mathbb{E}_{ q( \x_{t-1:t}, \vu, \f_{t}) } \left[   \log  p\left(\x_t \vert \G_{\btheta_{F}}(\f_t) \right)  \right] \nonumber\\
    			& \ = \sum_{t=1}^T \E_{q(\x_{t-1:t}), q(\f_t) } \left[\log p(\x_t \vert \G_{\btheta_{F}}(\f_t))\right]  \label{eq:integ_vu}\\
    			& \ \approx  \sum_{t=1}^T  \E_{q(\x_{t-1:t}) }  \left[  \frac{1}{n} \sum_{i=1}^n \log p(\x_t \vert \G_{\btheta_{F}}(\f_t^{(i)}))\right],  \\
                    & \ \qquad \f_t^{(i)} \sim q(\f_t), i = 1, 2, ..., n, \nonumber
    		\end{align}
            \end{subequations}
		where integrating out $\vu$ in Eq.~(\ref{eq:integ_vu}) is
            \begin{equation}
            \begin{aligned}
			q(\f_{t}) & = \int_\vu q(\f_{t}, \vu) \mathrm{d}  \vu =   \int_\vu q(\vu) p(\f_t \vert \x_{t-1}, \vu) \mathrm{d}  \vu  \\
                          & =	\mathcal{N}\left(\f_t \mid  {\m}_{\f_t}, \bm{\Sigma_{\f_t}} \right) 
		\end{aligned}
            \label{eq:q(f_0_t_supp)}
            \end{equation}
		and 
		\begin{subequations}
			\begin{align}
				& {\m}_{\f_t} = K_{\x_{t-1}, \vz} K_{\vz, \vz}^{-1} {\m}\\
				& \bm{\Sigma_{\f_t}} = K_{\x_{t-1}, \x_{t-1}} \!-\!  K_{\x_{t-1}, \vz} K_{\vz, \vz}^{-1}\left[K_{\vz, \vz} - \mathbf{S} \right] K_{\vz, \vz}^{-1} K_{\vz, \x_{t-1}}
			\end{align}
		\end{subequations}
		Term 5 represents the reconstruction of the latent state trajectories, which encourages the $\f_t$ sampled from GP transition $q(\f_t)$ to accurately reconstruct the latent state $\x_t$ (from $q(\x_t)$). In other words, this term measures the quality of learning/fitting the underlying dynamical function. \vspace{.1in}
		\item Therefore, the ELBO in (\ref{eq:mf_elbo}) eventually becomes
            \begin{align}
                & \operatorname{ELBO} \approx \nonumber \\
                    & \sum_{t=1}^T \mathbb{E}_{q(\x_{t-1})} \! \left[ \log \mathcal{N}(\y_t\mid \bm{C}  \Phi_{\x_t}, \bm{R}) \! - \! \frac{1}{2}\operatorname{tr}\left[ \bm{R}^{-1} (\bm{C \Sigma_{\x_t} C^\top})\right] \right] \nonumber  \\
                & \ -\frac{1}{2}\left[ \left(\mathbf{m}_{0}^{\top} \mathbf{m}_{0} \right)+\operatorname{tr}\left(\mathbf{L}_0 \mathbf{L}_0^\top  \right)- 2\log \left|\mathbf{L}_0 \right|-d_x\right] \nonumber  \\
                & \ - \frac{1}{2}\left[{\m}^{\top} \bm{K}_{\vz,\vz}^{-1}{\m}+\operatorname{tr}\left(\bm{K}_{\vz,\vz}^{-1} \mathbf{S} \right)-\log \frac{\left|  \mathbf{S}  \right|}{\left|\bm{K}_{\vz,\vz}\right|}-M d_x\right] \nonumber  \\
                & \ +  \sum_{t=1}^T \mathbb{E}_{q(\x_{t-1})}\left[ \frac{d_x}{2} \log (2 \pi)+\frac{1}{2} \log \left| \bm{\Sigma}_{\x_t} \right|+\frac{1}{2} d_x  \right]  \nonumber \\
                & \ +\sum_{t=1}^T  \E_{q(\x_{t-1:t}) }  \left[  \frac{1}{n} \sum_{i=1}^n \log p(\x_t \vert \G_{\btheta_{F}}(\f_t^{(i)}))\right].
            \end{align}
	\end{itemize}
\end{proof}

\bibliographystyle{IEEEtran}
\bibliography{ref-tgpssm.bib}

\begin{thebibliography}{10}
\providecommand{\url}[1]{#1}
\csname url@samestyle\endcsname
\providecommand{\newblock}{\relax}
\providecommand{\bibinfo}[2]{#2}
\providecommand{\BIBentrySTDinterwordspacing}{\spaceskip=0pt\relax}
\providecommand{\BIBentryALTinterwordstretchfactor}{4}
\providecommand{\BIBentryALTinterwordspacing}{\spaceskip=\fontdimen2\font plus
\BIBentryALTinterwordstretchfactor\fontdimen3\font minus
  \fontdimen4\font\relax}
\providecommand{\BIBforeignlanguage}[2]{{%
\expandafter\ifx\csname l@#1\endcsname\relax
\typeout{** WARNING: IEEEtran.bst: No hyphenation pattern has been}%
\typeout{** loaded for the language `#1'. Using the pattern for}%
\typeout{** the default language instead.}%
\else
\language=\csname l@#1\endcsname
\fi
#2}}
\providecommand{\BIBdecl}{\relax}
\BIBdecl

\bibitem{sarkka2013bayesian}
S.~S{\"a}rkk{\"a}, \emph{Bayesian filtering and smoothing}.\hskip 1em plus
  0.5em minus 0.4em\relax Cambridge University Press, 2013, no.~3.

\bibitem{yan2020gaussian}
Z.~Yan, P.~Cheng, Z.~Chen, Y.~Li, and B.~Vucetic, ``Gaussian process
  reinforcement learning for fast opportunistic spectrum access,'' \emph{IEEE
  Trans. Signal Process.}, vol.~68, pp. 2613--2628, Apr. 2020.

\bibitem{alaa2019attentive}
A.~M. Alaa and M.~van~der Schaar, ``Attentive state-space modeling of disease
  progression,'' in \emph{Proc. Adv. Neural Inf. Process. Syst. (NeurIPS)},
  Vancouver, BC, Canada, Dec. 2019, pp. 11\,338--11\,348.

\bibitem{revach2022kalmannet}
G.~Revach, N.~Shlezinger, X.~Ni, A.~L. Escoriza, R.~J. Van~Sloun, and Y.~C.
  Eldar, ``{KalmanNet: Neural network aided Kalman filtering for partially
  known dynamics},'' \emph{IEEE Trans. Signal Process.}, vol.~70, pp.
  1532--1547, Mar. 2022.

\bibitem{karl2017deep}
M.~Karl, M.~Soelch, J.~Bayer, and P.~Van~der Smagt, ``{Deep variational Bayes
  filters: Unsupervised learning of state space models from raw data},'' in
  \emph{Proc. Int. Conf. Learn. Represent. (ICLR)}, Toulon, France, Apr. 2017.

\bibitem{krishnan2017structured}
R.~Krishnan, U.~Shalit, and D.~Sontag, ``Structured inference networks for
  nonlinear state space models,'' in \emph{Proc. AAAI Conf. Artif. Intell.
  (AAAI)}, San Francisco, CA, United states, Feb. 2017, pp. 2101--2109.

\bibitem{frigola2015bayesian}
R.~Frigola, ``Bayesian time series learning with {G}aussian processes,'' Ph.D.
  dissertation, University of Cambridge, 2015.

\bibitem{theodoridis2020machine}
S.~Theodoridis, \emph{Machine Learning: {A Bayesian} and Optimization
  Perspective}, 2nd~ed.\hskip 1em plus 0.5em minus 0.4em\relax Academic Press,
  2020.

\bibitem{williams2006gaussian}
C.~E. Rasmussen and C.~K.~I. Williams, \emph{Gaussian Processes for Machine
  Learning}.\hskip 1em plus 0.5em minus 0.4em\relax MIT Press, 2006.

\bibitem{gedon2020deep}
D.~Gedon, N.~Wahlstr{\"o}m, T.~B. Sch{\"o}n, and L.~Ljung, ``Deep state space
  models for nonlinear system identification,'' \emph{IFAC-PapersOnLine},
  vol.~54, no.~7, pp. 481--486, 2021.

\bibitem{kullberg2021online}
A.~Kullberg, I.~Skog, and G.~Hendeby, ``Online joint state inference and
  learning of partially unknown state-space models,'' \emph{IEEE Trans. Signal
  Process.}, vol.~69, pp. 4149--4161, 2021.

\bibitem{yin2020fedloc}
F.~Yin, Z.~Lin, Q.~Kong, Y.~Xu, D.~Li, S.~Theodoridis, and S.~R. Cui,
  ``{Fedloc: Federated learning framework for data-driven cooperative
  localization and location data processing},'' \emph{IEEE Open J. Signal
  Process.}, vol.~1, pp. 187--215, 2020.

\bibitem{doerr2018probabilistic}
A.~Doerr, C.~Daniel, M.~Schiegg, N.-T. Duy, S.~Schaal, M.~Toussaint, and
  T.~Sebastian, ``Probabilistic recurrent state-space models,'' in \emph{Proc.
  Int. Conf. Mach. Learn. (ICML)}, Stockholm, Sweden, Jul. 2018, pp.
  1280--1289.

\bibitem{zhao2019cramer}
Y.~Zhao, C.~Fritsche, G.~Hendeby, F.~Yin, T.~Chen, and F.~Gunnarsson,
  ``{Cram{\'e}r--Rao bounds for filtering based on Gaussian process state-space
  models},'' \emph{IEEE Trans. Signal Process.}, vol.~67, no.~23, pp.
  5936--5951, 2019.

\bibitem{ko2011learning}
J.~Ko and D.~Fox, ``{Learning GP-BayesFilters via Gaussian process latent
  variable models},'' \emph{Auton. Robots}, vol.~30, no.~1, pp. 3--23, Oct.
  2011.

\bibitem{turner2010state}
R.~Turner, M.~Deisenroth, and C.~Rasmussen, ``{State-space inference and
  learning with Gaussian processes},'' in \emph{Proc. Int. Conf. Artif. Intell.
  Stat. (AISTATS)}, Sardinia, Italy, May 2010, pp. 868--875.

\bibitem{deisenroth2013gaussian}
M.~P. Deisenroth, D.~Fox, and C.~E. Rasmussen, ``Gaussian processes for
  data-efficient learning in robotics and control,'' \emph{IEEE Trans. Pattern
  Anal. Mach. Intell.}, vol.~37, no.~2, pp. 408--423, 2013.

\bibitem{deisenroth2011robust}
M.~P. Deisenroth, R.~D. Turner, M.~F. Huber, U.~D. Hanebeck, and C.~E.
  Rasmussen, ``{Robust filtering and smoothing with Gaussian processes},''
  \emph{IEEE Trans. Autom. Control}, vol.~57, no.~7, pp. 1865--1871, 2011.

\bibitem{wang2007gaussian}
J.~M. Wang, D.~J. Fleet, and A.~Hertzmann, ``{Gaussian process dynamical models
  for human motion},'' \emph{IEEE Trans. Pattern Anal. Mach. Intell.}, vol.~30,
  no.~2, pp. 283--298, 2007.

\bibitem{frigola2013bayesian}
R.~Frigola, F.~Lindsten, T.~B. Sch{\"o}n, and C.~E. Rasmussen, ``{Bayesian
  inference and learning in Gaussian process state-space models with particle
  MCMC},'' in \emph{Proc. Adv. Neural Inf. Process. Syst. (NeurIPS)}, Lake
  Tahoe, NV, United states, Dec. 2013, pp. 3156--3164.

\bibitem{frigola2014variational}
R.~Frigola, Y.~Chen, and C.~E. Rasmussen, ``{Variational Gaussian process
  state-space models},'' in \emph{Proc. Adv. Neural Inf. Process. Syst.
  (NeurIPS)}, Montreal, QC, Canada, Dec. 2014, pp. 3680--3688.

\bibitem{mchutchon2015nonlinear}
A.~J. McHutchon, ``{Nonlinear modelling and control using Gaussian
  processes},'' Ph.D. dissertation, University of Cambridge, 2014.

\bibitem{eleftheriadis2017identification}
S.~Eleftheriadis, T.~Nicholson, M.~P. Deisenroth, and J.~Hensman,
  ``Identification of {Gaussian} process state space models,'' in \emph{Proc.
  Adv. Neural Inf. Process. Syst. (NeurIPS)}, Long Beach, CA, United states,
  Dec. 2017, pp. 5309--5319.

\bibitem{ialongo2019overcoming}
A.~D. Ialongo, M.~van~der Wilk, J.~Hensman, and C.~E. Rasmussen, ``{Overcoming
  mean-field approximations in recurrent Gaussian process models},'' in
  \emph{Proc. Int. Conf. Mach. Learn. (ICML)}, Long Beach, CA, United states,
  Jun. 2019, pp. 2931--2940.

\bibitem{curi2020structured}
S.~Curi, S.~Melchior, F.~Berkenkamp, and A.~Krause, ``Structured variational
  inference in partially observable unstable {G}aussian process state space
  models,'' in \emph{Proc. Learning for Dynamics and Control (L4DC)}, Virtual,
  Online, Jun. 2020, pp. 147--157.

\bibitem{lindinger2022laplace}
J.~Lindinger, B.~Rakitsch, and C.~Lippert, ``Laplace approximated {Gaussian}
  process state-space models,'' in \emph{Proc. Conf. Uncertain. Artif. Intell.
  (UAI)}, Eindhoven, Netherlands, Aug. 2022.

\bibitem{lin2022output}
Z.~Lin, L.~Cheng, F.~Yin, L.~Xu, and S.~Cui, ``Output-dependent {G}aussian
  process state-space model,'' in \emph{Proc. IEEE Int. Conf. Acoust. Speech
  Signal Process. (ICASSP)}, Rhodes, Greek, 2023, pp. 1--5.

\bibitem{liu2020gpssm}
Y.~Liu and P.~M. Djuri{\'c}, ``{Gaussian} process state-space models with
  time-varying parameters and inducing points,'' in \emph{Proc. European Signal
  Proces. Conf. (EUSIPCO)}, Amsterdam, Netherlands, Jan. 2021, pp. 1462--1466.

\bibitem{liu2022inference}
Y.~Liu, M.~Ajirak, and P.~M. Djuri{\'c}, ``Inference with deep {Gaussian}
  process state space models,'' in \emph{Proc. European Signal Proces. Conf.
  (EUSIPCO)}, Belgrade, Serbia, Oct. 2022, pp. 792--796.

\bibitem{zhao2022streaming}
Y.~Zhao, J.~Nassar, I.~Jordan, M.~Bugallo, and I.~M. Park, ``Streaming
  variational monte carlo,'' \emph{IEEE Trans. Pattern Anal. Mach. Intell.},
  vol.~45, no.~1, pp. 1150--1161, 2022.

\bibitem{dowlingreal}
M.~Dowling, Y.~Zhao, and I.~M. Park, ``Real-time variational method for
  learning neural trajectory and its dynamics,'' in \emph{Proc. Int. Conf.
  Learn. Represent. (ICLR)}, 2023.

\bibitem{damianou2013deep}
A.~Damianou and N.~D. Lawrence, ``{Deep Gaussian processes},'' in \emph{Proc.
  Int. Conf. Artif. Intell. Stat. (AISTATS)}, 2013, pp. 207--215.

\bibitem{wilson2010copula}
A.~G. Wilson and Z.~Ghahramani, ``Copula processes,'' \emph{Proc. Adv. Neural
  Inf. Process. Syst. (NeurIPS)}, vol.~23, 2010.

\bibitem{rios2020contributions}
G.~A. R{\'\i}os~D{\'\i}az, ``{Contributions to Bayesian machine learning via
  transport maps},'' Ph.D. dissertation, University of Chile, 2020.

\bibitem{maronas2021transforming}
J.~Maro{\~n}as, O.~Hamelijnck, J.~Knoblauch, and T.~Damoulas, ``{Transforming
  Gaussian processes with normalizing flows},'' in \emph{Proc. Int. Conf.
  Artif. Intell. Stat. (AISTATS)}, Virtual, Online, Apr. 2021, pp. 1081--1089.

\bibitem{maronas2022efficient}
J.~Maro{\~n}as and D.~Hern{\'a}ndez-Lobato, ``Efficient transformed {G}aussian
  processes for non-stationary dependent multi-class classification,''
  \emph{arXiv preprint arXiv:2205.15008}, 2022.

\bibitem{wilson2013gaussian}
A.~Wilson and R.~Adams, ``Gaussian process kernels for pattern discovery and
  extrapolation,'' in \emph{Proc. Int. Conf. Mach. Learn. (ICML)}, Atlanta, GA,
  United states, Jun. 2013, pp. 1067--1075.

\bibitem{yin2020linear}
F.~Yin, L.~Pan, T.~Chen, S.~Theodoridis, Z.-Q.~T. Luo, and A.~M. Zoubir,
  ``{Linear multiple low-rank kernel based stationary Gaussian processes
  regression for time series},'' \emph{IEEE Trans. Signal Process.}, vol.~68,
  pp. 5260--5275, 2020.

\bibitem{suwandi2022gaussian}
R.~C. Suwandi, Z.~Lin, Y.~Sun, Z.~Wang, L.~Cheng, and F.~Yin, ``Gaussian
  process regression with grid spectral mixture kernel: Distributed learning
  for multidimensional data,'' in \emph{Proc. Int. Conf. Inf. Fusion (FUSION)},
  Linkoping, Sweden, July 2022, pp. 1--8.

\bibitem{wilson2016deep}
A.~G. Wilson, Z.~Hu, R.~Salakhutdinov, and E.~P. Xing, ``Deep kernel
  learning,'' in \emph{Proc. Int. Conf. Artif. Intell. Stat. (AISTATS)}, Cadiz,
  Spain, May 2016, pp. 370--378.

\bibitem{dai2020interpretable}
Y.~Dai, T.~Zhang, Z.~Lin, F.~Yin, S.~Theodoridis, and S.~Cui, ``An
  interpretable and sample efficient deep kernel for {G}aussian process,'' in
  \emph{Proc. Conf. Uncertain. Artif. Intell. (UAI)}, Virtual, Online, Aug.
  2020, pp. 759--768.

\bibitem{papamakarios2021normalizing}
G.~Papamakarios, E.~Nalisnick, D.~J. Rezende, S.~Mohamed, and
  B.~Lakshminarayanan, ``Normalizing flows for probabilistic modeling and
  inference,'' \emph{J. Mach. Learn. Res.}, vol.~22, no.~57, pp. 1--64, Mar.
  2021.

\bibitem{kobyzev2020normalizing}
I.~Kobyzev, S.~J. Prince, and M.~A. Brubaker, ``{Normalizing flows: An
  introduction and review of current methods},'' \emph{IEEE Trans. Pattern
  Anal. Mach. Intell.}, vol.~43, no.~11, pp. 3964--3979, 2020.

\bibitem{titsias2009variational}
M.~Titsias, ``Variational learning of inducing variables in sparse {Gaussian}
  processes,'' in \emph{Proc. Int. Conf. Artif. Intell. Stat. (AISTATS)},
  Clearwater, FL, United states, Apr. 2009, pp. 567--574.

\bibitem{higgins2016beta}
I.~Higgins, L.~Matthey, A.~Pal, C.~Burgess, X.~Glorot, M.~Botvinick,
  S.~Mohamed, and A.~Lerchner, ``{$\beta$-VAE}: Learning basic visual concepts
  with a constrained variational framework,'' in \emph{Proc. Int. Conf. Learn.
  Represent. (ICLR)}, Toulon, France, Apr. 2017.

\bibitem{tao2011introduction}
T.~Tao, \emph{An introduction to measure theory}.\hskip 1em plus 0.5em minus
  0.4em\relax American Mathematical Society Providence, RI, 2011, vol. 126.

\bibitem{chen2023remarks}
Z.~Chen, J.~Fan, and K.~Wang, ``Multivariate {G}aussian processes: definitions,
  examples and applications,'' \emph{METRON}, pp. 1--11, 2023.

\bibitem{wauthier2010heavy}
F.~L. Wauthier and M.~Jordan, ``Heavy-tailed process priors for selective
  shrinkage,'' in \emph{Proc. Adv. Neural Inf. Process. Syst. (NeurIPS)},
  Vancouver, BC, Canada, Dec. 2010.

\bibitem{snelson2003warped}
E.~Snelson, Z.~Ghahramani, and C.~Rasmussen, ``{Warped Gaussian processes},''
  \emph{Proc. Adv. Neural Inf. Process. Syst. (NeurIPS)}, vol.~16, 2003.

\bibitem{rios2019compositionally}
G.~Rios and F.~Tobar, ``{Compositionally-warped Gaussian processes},''
  \emph{Neural Netw.}, vol. 118, pp. 235--246, 2019.

\bibitem{jones2009sinh}
M.~C. Jones and A.~Pewsey, ``Sinh-arcsinh distributions,'' \emph{Biometrika},
  vol.~96, no.~4, pp. 761--780, 2009.

\bibitem{dinh2017density}
L.~Dinh and S.~Bengio, ``{Density estimation using Real NVP},'' in \emph{Proc.
  Int. Conf. Learn. Represent. (ICLR)}, Toulon, France, Apr. 2017.

\bibitem{chen2018neural}
R.~T. Chen, Y.~Rubanova, J.~Bettencourt, and D.~K. Duvenaud, ``Neural ordinary
  differential equations,'' in \emph{Proc. Adv. Neural Inf. Process. Syst.
  (NeurIPS)}, Montreal, QC, Canada, Dec. 2018, pp. 6572--6583.

\bibitem{BondTaylor2021}
S.~Bond-Taylor, A.~Leach, Y.~Long, and C.~G. Willcocks, ``Deep generative
  modelling: A comparative review of {VAEs, GANs}, normalizing flows,
  energy-based and autoregressive models,'' \emph{IEEE Trans. Pattern Anal.
  Mach. Intell.}, vol.~44, no.~11, pp. 7327--7347, Sep. 2021.

\bibitem{courts2021gaussian}
J.~Courts, A.~G. Wills, and T.~B. Sch{\"o}n, ``Gaussian variational state
  estimation for nonlinear state-space models,'' \emph{IEEE Trans. Signal
  Process.}, vol.~69, pp. 5979--5993, Oct. 2021.

\bibitem{cheng2022rethinking}
L.~Cheng, F.~Yin, S.~Theodoridis, S.~Chatzis, and T.-H. Chang, ``Rethinking
  {Bayesian} learning for data analysis: The art of prior and inference in
  sparsity-aware modeling,'' \emph{IEEE Signal Process. Mag.}, vol.~39, no.~6,
  pp. 18--52, Nov. 2022.

\bibitem{hensman2013gaussian}
J.~Hensman, N.~Fusi, and N.~D. Lawrence, ``Gaussian processes for big data,''
  in \emph{Proc. Conf. Uncertain. Artif. Intell. (UAI)}, Bellevue, WA, United
  states, Jul. 2013, pp. 282--290.

\bibitem{mattos2019stochastic}
C.~L.~C. Mattos and G.~A. Barreto, ``{A stochastic variational framework for
  recurrent Gaussian processes models},'' \emph{Neural Netw.}, vol. 112, pp.
  54--72, Apr. 2019.

\bibitem{kingma2019introduction}
D.~P. Kingma and M.~Welling, ``An introduction to variational autoencoders,''
  \emph{Found. Trends Mach. Learn.}, vol.~12, no.~4, pp. 307--392, 2019.

\bibitem{kingma2015adam}
D.~P. Kingma and J.~Ba, ``Adam: A method for stochastic optimization,'' in
  \emph{Proc. Int. Conf. Learn. Represent. (ICLR)}, San Diego, CA, United
  states, May 2015.

\bibitem{alemi2018fixing}
A.~Alemi, B.~Poole, I.~Fischer, J.~Dillon, R.~A. Saurous, and K.~Murphy,
  ``Fixing a broken {ELBO},'' in \emph{Proc. Int. Conf. Mach. Learn. (ICML)},
  Stockholm, Sweden, Jul. 2018, pp. 159--168.

\bibitem{knoblauch2022optimization}
J.~Knoblauch, J.~Jewson, and T.~Damoulas, ``An optimization-centric view on
  bayes’ rule: Reviewing and generalizing variational inference,'' \emph{J.
  Mach. Learn. Res.}, vol.~23, no. 132, pp. 1--109, 2022.

\bibitem{klushyn2019learning}
A.~Klushyn, N.~Chen, R.~Kurle, B.~Cseke, and P.~van~der Smagt, ``Learning
  hierarchical priors in {VAEs},'' in \emph{Proc. Adv. Neural Inf. Process.
  Syst. (NeurIPS)}, Vancouver, BC, Canada, Dec. 2019, pp. 2870--2879.

\bibitem{casella2021statistical}
G.~Casella and R.~L. Berger, \emph{Statistical inference}, 2nd~ed.\hskip 1em
  plus 0.5em minus 0.4em\relax Cengage Learning, 2001.

\bibitem{stimper2023normflows}
V.~Stimper, D.~Liu, A.~Campbell, V.~Berenz, L.~Ryll, B.~Sch{\"o}lkopf, and
  J.~M. Hern{\'a}ndez-Lobato, ``normflows: A pytorch package for normalizing
  flows,'' \emph{arXiv preprint arXiv:2302.12014}, 2023.

\end{thebibliography}
\vfill

\newpage
\appendices
\onecolumn
\section{Sampling from TGPSSM} \label{subsec:samplingTGPSSM}
For notational brevity, we only show the sampling steps in the case of one-dimensional hidden state. It can be straightforwardly extended to high-dimensional hidden state cases.
\begin{itemize}
	\item \textbf{TGPSSM}: If $f(\cdot) \sim \mathcal{GP}(0, k(\cdot, \cdot))$, and the marginal flow is $\G(\cdot): \mathcal{F} \mapsto \mathcal{F}$, then we can sample the entire TGPSSM state trajectory by the following steps:
	\begin{subequations}
		\begin{align}
			& \x_0 \sim p(\x_0), \\
			&\f_{1} \mid \x_0 \sim \cN(\f_{1} \mid \bm{0}, \bm{K}_{\x_0, \x_0})\\
			& \tilde{\f}_{1} = \G(\f_{1})\\
			& \x_1 \mid \tilde{\f}_{1}  \sim \cN(\x_1 \mid \tilde{\f}_{1}, \bm{Q})	\\
			& \f_{2} \mid \f_{1}, \x_{0:1}  \sim \cN(\f_{2} \mid \bm{K}_{\x_1, \x_0}\bm{K}_{\x_0, \x_0}^{-1}\f_{1},   \  \bm{K}_{\x_1, \x_1} -  \bm{K}_{\x_1, \x_0}\bm{K}_{\x_0, \x_0}^{-1}  \bm{K}_{\x_0, \x_1}) \\
			& \tilde{\f}_{2} = \G(\f_{2})\\
			& \x_2 \mid \tilde{\f}_{2} \sim \cN(\x_2 \mid \tilde{\f}_{2}, \bm{Q})	\\
			& \f_{3} \mid \f_{1:2}, \x_{0:2} \sim \cN(\f_{3} \mid \bm{K}_{\x_2, \x_{0:1}}\bm{K}_{\x_{0:1}, \x_{0:1}}^{-1}\f_{1:2}, \ \bm{K}_{\x_2, \x_2} -  \bm{K}_{\x_2, \x_{0:1}}\bm{K}_{\x_{0:1}, \x_{0:1}}^{-1}  \bm{K}_{\x_{0:1}, \x_2})\\
			& \tilde{\f}_{3} = \G(\f_{3})\\
			& \x_3 \mid \tilde{\f}_{3}  \sim \cN(\x_3 \mid \tilde{\f}_{3}, \bm{Q})	\\
			& \qquad \vdots \nonumber\\
			& \f_{t} \mid \f_{1:t-1}, \x_{0:t-1} \sim \cN(\f_{t} \mid \bm{K}_{\x_{t-1}, \x_{0:t-2}}\bm{K}_{\x_{0:t-2}, \x_{0:t-2}}^{-1}\f_{1:{t-1}}, \ \bm{K}_{\x_{t-1}, \x_{t-1}} \!-\!  \bm{K}_{\x_{t-1}, \x_{0:t-2}}\bm{K}_{\x_{0:t-2}, \x_{0:t-2}}^{-1}  \bm{K}_{\x_{0:t-2}, \x_{t-1}})\\
			& \tilde{\f}_{t} = \G(\f_{t})\\
			& \x_t \mid \tilde{\f}_{t}  \sim \cN(\x_t \mid \tilde{\f}_{t}, \bm{Q})	\\
			& \qquad \vdots \nonumber
		\end{align}
	\end{subequations}
	Note that the GP sampling steps conditioning on previous sampled states guarantee the sampled state trajectory is consistent. \vspace{.05in}
	\item \textbf{Augmented TGPSSM}:  In the TGPSSM augmented by sparse inducing points, the set of inducing points $\vu$ serves as the surrogate (sufficient statistic)  of $\vf$, therefore, the GP transition function value $\f_t$ in each step can be obtained by conditioning on $\vu$, the sampling steps are summarized as follows:
	\begin{subequations}
		\begin{align}
			& \vu \sim p(\vu \mid \bm{0}, \bm{K}_{\vz,\vz})\\
			& \x_0 \sim p(\x_0), \\
			&\f_{1} \mid \x_0, \vu \sim \cN \left(\f_{1} \mid \bK_{\x_0, \vz} \bK_{\vz,\vz}^{-1} \vu,   \ \bK_{\x_0,\x_0} \!-\! \bK_{\x_0, \vz} \bK_{\vz,\vz}^{-1}  \bK_{\x_0,\vz} ^\top \right)\\
			& \tilde{\f}_{1} = \G(\f_{1})\\
			& \x_1 \mid \tilde{\f}_{1}  \sim \cN(\x_1 \mid \tilde{\f}_{1}, \bm{Q})	\\
			& \f_{2} \mid \x_{1}, \vu \sim \cN \left( \f_{2}  \mid \bK_{\x_1, \vz} \bK_{\vz,\vz}^{-1} \vu,   \ \bK_{\x_1,\x_1} \!-\! \bK_{\x_1, \vz} \bK_{\vz,\vz}^{-1}  \bK_{\x_1,\vz} ^\top \right)\\
			& \tilde{\f}_{2} = \G(\f_{2})\\
			& \x_2 \mid \tilde{\f}_{2} \sim \cN(\x_2 \mid \tilde{\f}_{2}, \bm{Q})	\\
			& \qquad \vdots \nonumber\\
			& \f_{t} \mid \x_{t-1}, \vu \sim \cN \left( \f_{t}  \mid \bK_{\x_{t-1}, \vz} \bK_{\vz,\vz}^{-1} \vu,   \ \bK_{\x_{t-1},\x_{t-1}} \!-\! \bK_{\x_{t-1}, \vz} \bK_{\vz,\vz}^{-1}  \bK_{\x_{t-1},\vz} ^\top \right) \\
			& \tilde{\f}_{t} = \G(\f_{t})\\
			& \x_t \mid \tilde{\f}_{t}  \sim \cN(\x_t \mid \tilde{\f}_{t}, \bm{Q})	\\
			& \qquad \vdots \nonumber
		\end{align}
	\end{subequations}
\end{itemize}

Fig.~\ref{fig:SSM_prior} presents examples of state trajectories sampled from TGPSSM and GPSSM, both learned from the kink-step function dataset. The SE kernel function is used in both models, and for TGPSSM, the normalizing flows are a combination of three blocks of SAL flow and one block of Tanh flow.



\begin{figure}[t!]
	\centering
	\subfloat[GPSSM learned from kink-step function dataset]{\label{fig:gpssm_prior1} \includegraphics[width=0.4\columnwidth]{figs/COGPSSM_ksfunc.pdf}} \hspace{.2in}
	\subfloat[State trajectories sampled from the GPSSM]{\label{fig:gpssm_prior2} \includegraphics[width=0.4\columnwidth]{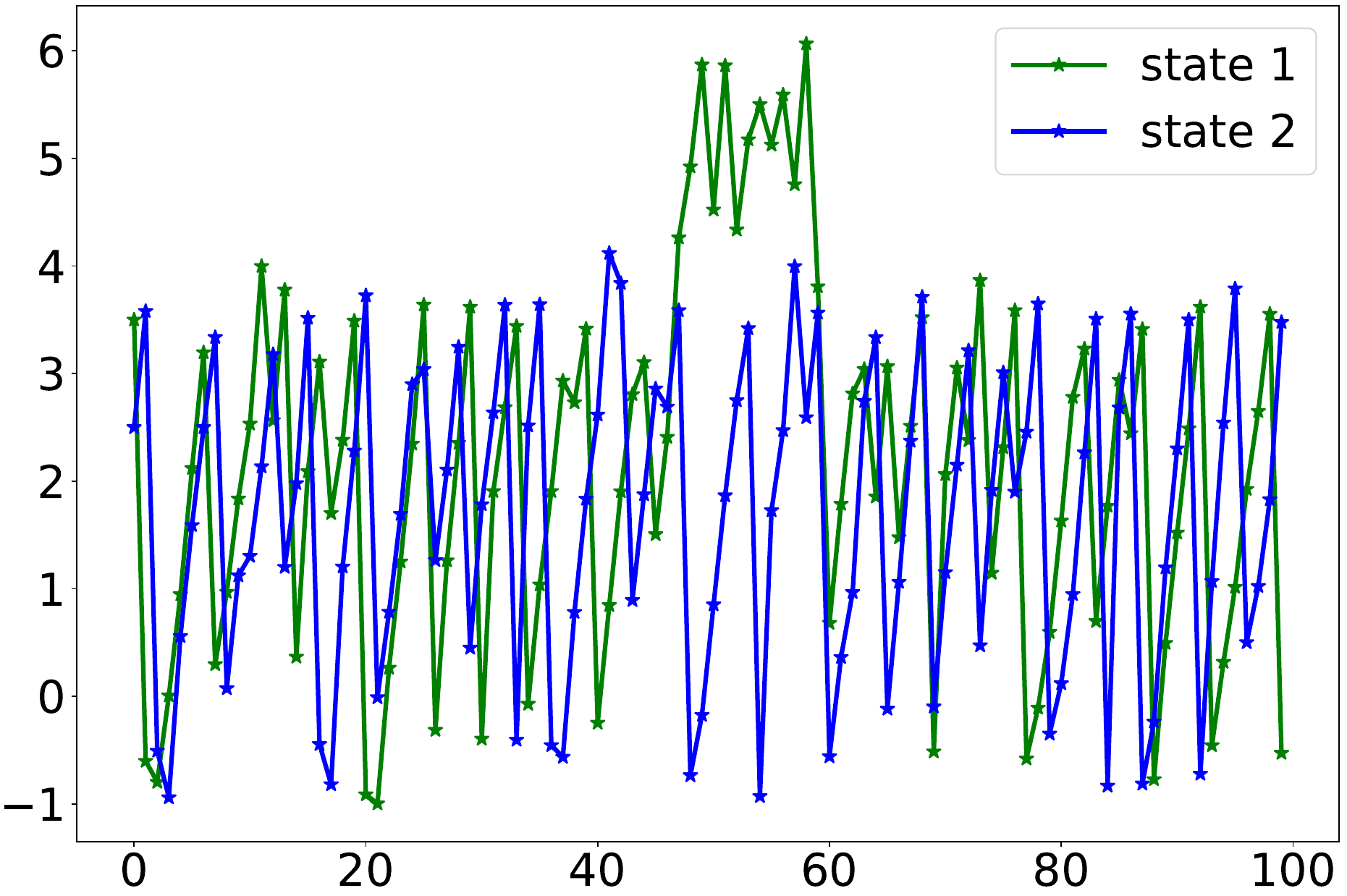}}
 
	\subfloat[TGPSSM learned from kink-step function dataset]{\label{fig:tgpssm_prior1} \includegraphics[width=0.4\columnwidth]{figs/COTGPSSM_ksfunc.pdf}} \hspace{.2in}
	\subfloat[State trajectories sampled from the TGPSSM]{\label{fig:tgpssm_prior2} \includegraphics[width=0.4\columnwidth]{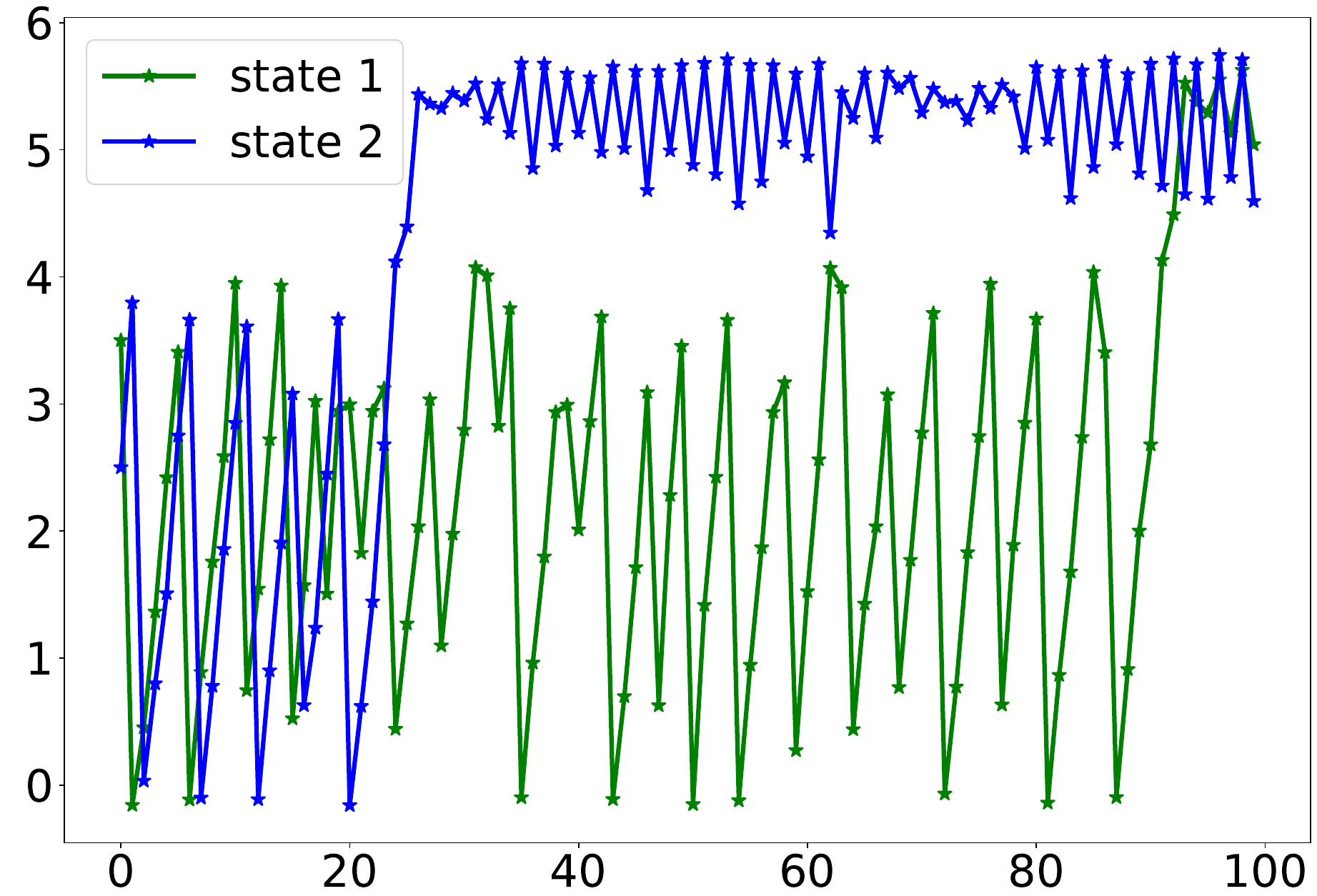}}
	\caption{The learned GPSSM and TGPSSM, and the corresponding state trajectories sampled from each model. The SE kernel function is used in both GPSSM and TGPSSM. The normalizing flows used in TGPSSM are simply a combination of three blocks of SAL flow and one block of Tanh flow.  
	}
	\label{fig:SSM_prior}
\end{figure}

\section{More Practical Implementation Details}
\label{appendix_III}

\subsection{Description of Flows} \label{subsec:appx_flow_table}
This subsection provides some elementary flows along with their compositions commonly utilized in the literature. See Table~\ref{tab:flows}.
%
\begin{table*}[!ht]
	\centering
	\begin{threeparttable}[b]
		\caption{Elementary flows and their compositions}
		\begin{tabular}{p{1.5cm}<{\centering}|ccccccl}
			\toprule
			\multirow{3}{*}{\begin{tabular}[c]{@{}c@{}}\textbf{Elementary} \\ \textbf{Flow}\end{tabular}} & \multicolumn{1}{c|}{\textbf{Arcsinh}} & \multicolumn{1}{c|}{\textbf{Log}} & \multicolumn{1}{c|}{\textbf{Exp}}  & \multicolumn{1}{c|}{\textbf{Linear}}     & \multicolumn{1}{c|}{\textbf{Sinh-Arcsinh}} & \multicolumn{1}{c|}{\textbf{Box-Cox}}   & \multicolumn{1}{c}{\textbf{Tanh}}                  \\ \cline{2-8} 
			& \multicolumn{1}{c|}{$a \!+\! b \operatorname{arcsinh}\left[d (\f-c) \right]\!\!,$}       & \multicolumn{1}{c|}{$\log(\f)$}    & \multicolumn{1}{c|}{$\exp(\f)$}    & \multicolumn{1}{c|}{$a + b \f,$}      &  \multicolumn{1}{c|}{$\sinh \left[ b  \operatorname{arcsinh}\left(\mathbf{f}\right) \! - \! a\right]$,}       & \multicolumn{1}{c|}{$\frac{1}{\lambda}\left(\operatorname{sgn}\left(\mathbf{f}\right)\left|\mathbf{f}\right|^{\lambda}-1\right)$,}             &  \multicolumn{1}{c}{$a \tanh \left[b\left(\mathbf{f}+c\right)\right] \! + \! d$,}     \\ 
			& \multicolumn{1}{c|}{$a, b,c,d \in \mathbb{R}$}        & \multicolumn{1}{c|}{}     & \multicolumn{1}{c|}{}     & \multicolumn{1}{c|}{$a, b \in \mathbb{R}$}             & \multicolumn{1}{c|}{$a,b \in \mathbb{R}$}         & \multicolumn{1}{c|}{$\lambda > 0$}                &  \multicolumn{1}{c}{$a, b, c, d \in \mathbb{R}$}                       \\  \midrule 
			\multirow{2}{*}{\begin{tabular}[c]{@{}c@{}}\textbf{Flow} \\ \textbf{Composition}\end{tabular}}  & \multicolumn{3}{c|}{\textbf{Sum of Log-Exp} \cite{wilson2010copula}}                            & \multicolumn{2}{c|}{\textbf{Sinh-Arcsinh-Linear (SAL) } \cite{rios2020contributions}}    & \multicolumn{2}{c}{\textbf{Sum of Tanh} \cite{snelson2003warped}}  \\ \cline{2-8} 
			& \multicolumn{3}{c|}{$\sum_{j=0}^{\mathrm{J-1}} a_{j} \log \left(1+\exp \left[b_{j}(\f+c_{j})\right]\right)$,}         & \multicolumn{2}{c|}{$d \sinh \left(b \operatorname{arcsinh}(\mathbf{f} )-a \right)+c$,}   & \multicolumn{2}{c}{$\f+\sum_{j=0}^{\mathrm{J-1}} a_{j} \tanh \left[ b_{j}(\f+c_{j}) \right]$,}       \\ 
			& \multicolumn{3}{c|}{$a_j, b_j \ge 0, \forall j$}                    &  \multicolumn{2}{c|}{$a, b, c, d \in \mathbb{R}$}        & \multicolumn{2}{c}{$a_j, b_j \ge 0, \forall j$}                      \\
			\bottomrule
		\end{tabular}
		\label{tab:flows}
		\begin{tablenotes}
			\item[$\bullet$] If the input argument $\f$ is multidimensional, one can use a common flow to transform $\f$ element-wisely or multiple flows to transform each dimension of $\f$ independently. We transform each dimension of $\f$ independently in the main-text when using elementary flows.
		\end{tablenotes}
	\end{threeparttable}
\end{table*}

\subsection{Selection of Desired Data Reconstruction Quality} \label{subsec:selection_R0_appx}

This subsection provides a detailed account of the empirical process used to select $\mathcal{R}_0$ for constrained optimization algorithms. In the case that we have sufficient prior knowledge of the data, we can manually set a desired data reconstruction quality value, $\mathcal{R}_0$. Alternatively, we can pre-train the inference network to obtain the reconstruction quality of the generated latent state trajectories with respect to the observations, and use this as the empirical reconstruction quality, $\mathcal{R}_0$. The specific steps for empirically calculating $\mathcal{R}_0$ are summarized in Algorithm~\ref{alg_tgpssm1_R0_selection}.
%
%
\begin{algorithm}[ht!] 
	\caption{Empirical $\mathcal{R}_0$ Calculation for Algorithm \ref{alg_tgpssm2}} 
	\label{alg_tgpssm1_R0_selection} 
	\KwIn{Dataset $\{\y_{1:T}\}$.  Initial parameters $\btheta^{(0)}$, $\bm{\zeta}^{(0)}$.}
	\While{not terminated}
	{	
		Evaluate Eq.~(\ref{eq:term2}) and Eq.~(\ref{eq:term3});\\
		Sample state trajectory $\x_{0:T} \! \sim  \! q(\x_{0:T})$ (Eq.~(\ref{eq:MF_qx}));\\
		  \For{$t=1:T$}{
                Evaluate entropy term,  Eq.~(\ref{eq:term4});\\
			  Evaluate data reconstruction term, Eq.(\ref{eq:term1});\\
		}
		Evaluate ELBO (Eq.~(\ref{eq:elbo_tgpssm_mf})) without calculating Eq.~(\ref{eq:term5});\\
		Estimate the Monte-Carlo gradient w.r.t. $\btheta$ and $\bzeta$;\\
		Update $\btheta$ and $\bzeta$ using Adam \cite{kingma2015adam};
	}
        \KwOut{$\mathcal{R}_0$=Eq.(\ref{eq:term1}), and initialized parameters $\btheta$ and $\bm{\zeta}$.} 
\end{algorithm}	
%

\subsection{Neural Network Architectures}
This subsection provides specific information on the neural network architectures we implemented for the (T)GPSSMs. For more comprehensive details, readers can access the publicly available source code.

Our inference network utilizes a bi-directional LSTM to encode the observed sequence, producing a 128-dimensional hidden space. The LSTM output and the latent states are passed through a one-layer neural network, performing nonlinear mapping to output the mean and covariance of the latent states. Throughout all experiments, we maintain consistency in the structure of the inference network, and default parameter values are employed. We utilize the publicly available package \cite{stimper2023normflows} to implement RealNVP, and employ a two-layer MLP with a hidden layer dimensionality of 64 as the nonlinear mappings for the RealNVP.

\subsection{Scalability of the proposed Algorithms} \label{subsec:appex_scalability}
To improve scalability for large datasets, such as long observation sequences, the two algorithms presented in the main text can utilize the stochastic gradient optimization method. In practice, this involves partitioning the complete sequence into multiple mini-batches of sub-trajectories, and then conducting stochastic optimization on each mini-batch.  In order to support this method, an additional recognition network is required to model the variational distribution $q(\x_0)$ of the initial latent state $\x_0$. To accomplish this, we utilize a LSTM to map the observations into $q(\x_0)$, which is modeled as a Gaussian. The LSTM recognition network is composed of two hidden layers with a hidden dimension of 32.





\section{More Experimental Illustration Results}
\label{appx:more_experimental_results}
This section offers additional information about datasets. Specifically, Fig.\ref{fig:1D_dataset_appx} presents the synthetic 1-D datasets; Table\ref{tab:SIDdataset} provides details about the five real-world datasets, each of them having one-dimensional deterministic control inputs. Moreover, more illustrative results are given in Figs.~\ref{fig:1Ddataresults_SSMs_appx}, \ref{fig:sysID_dataset_appx}, and \ref{fig:Lorenz_dataresults_SSMs}.

\begin{figure*}[t!]
	\centering		
	\subfloat[``Kink'' function and the generated $30$ latent states \& observations.]{\label{fig:KinkFunction_appx}\includegraphics[width =0.49\textwidth]{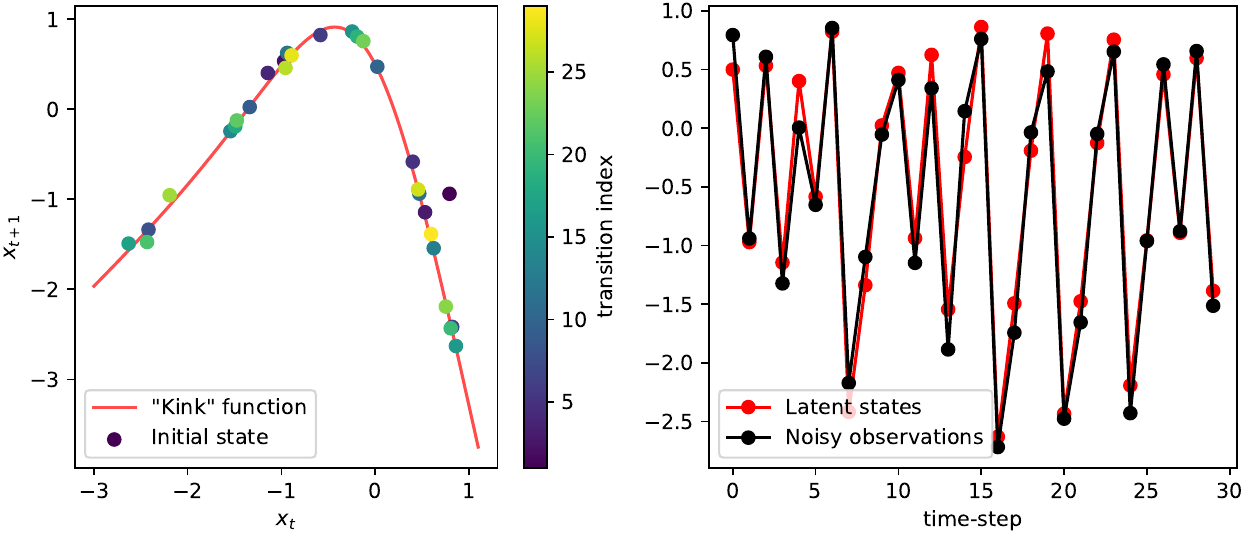}}  \hspace{.1in}
	\subfloat[``Kink-step'' function and the generated $30$ latent states \& observations.]{\label{fig:ksfunc_appx}\includegraphics[width =0.49\textwidth]{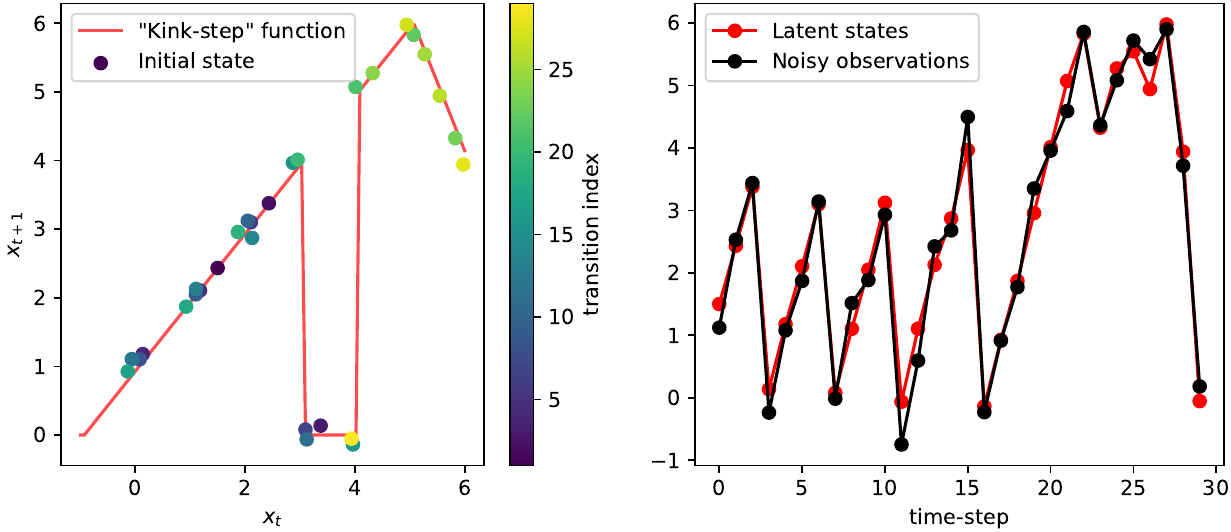}}
	\caption{The 1-D datasets (kink function and kink-step function), including the latent state trajectories and the corresponding observations.}
	\label{fig:1D_dataset_appx}
\end{figure*}
\begin{figure*}[t!]
	\centering		
	\subfloat[BS-GPSSM (MSE: 0.3059)]{\label{fig:bs_gp_kinkfunc_appx} \includegraphics[width=0.2\columnwidth]{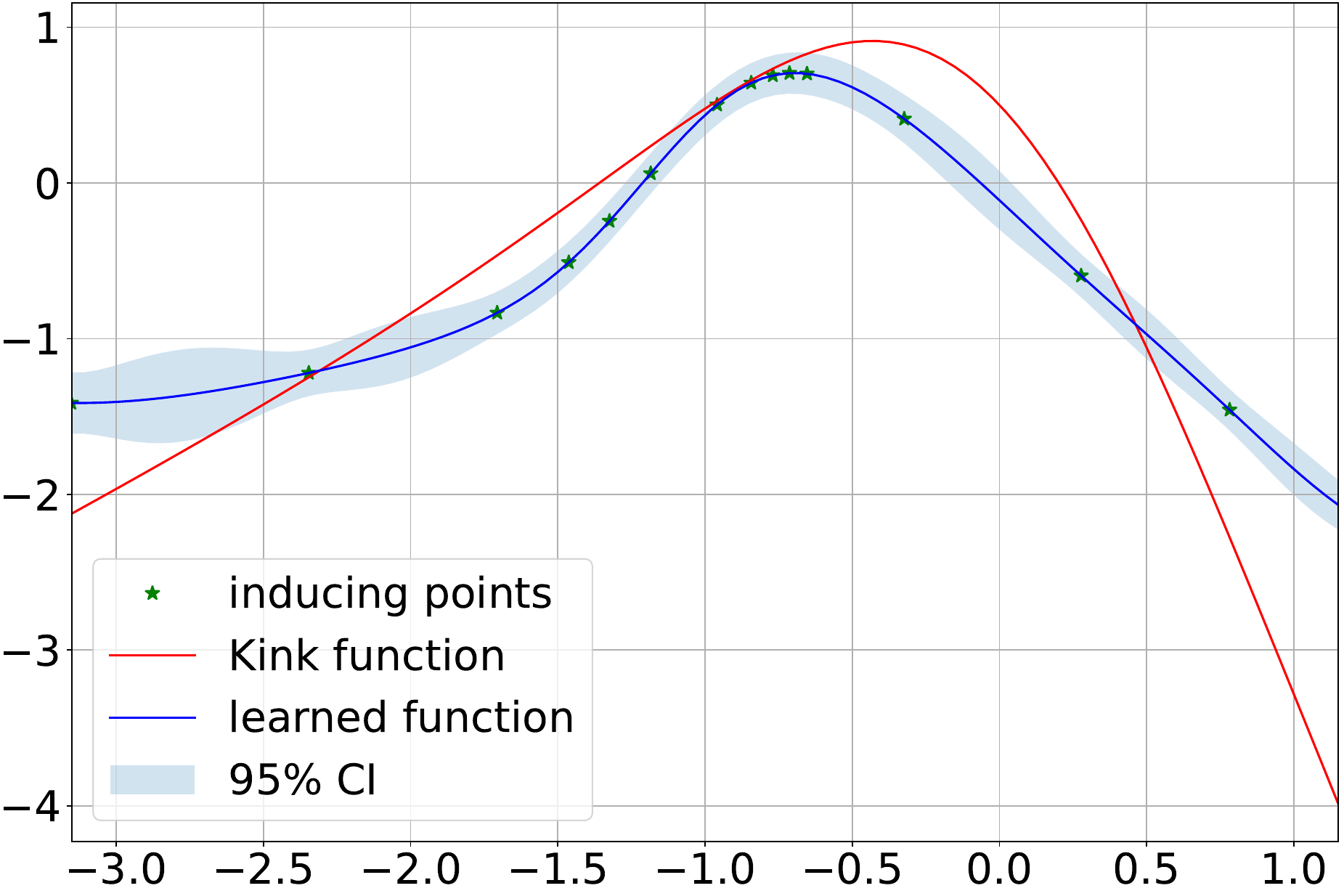}} 
	\subfloat[BS-GPSSM (MSE: {3.0663}) ]{\label{fig:bs_gp_ksfunc_appx} \includegraphics[width=0.2\columnwidth]{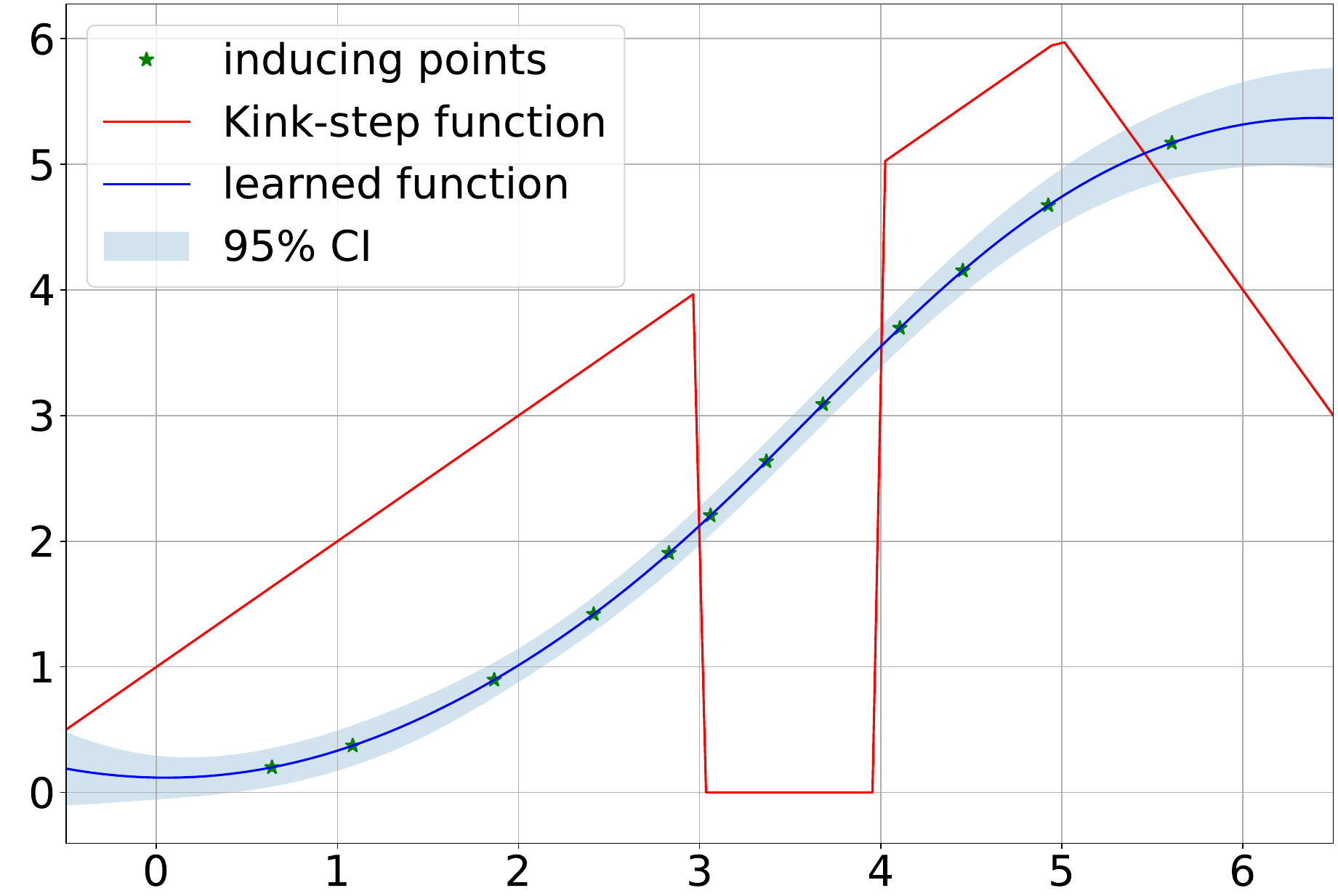}}
	\subfloat[JO-GPSSM (MSE: 0.0364)]{\label{fig:gp_kinkfunc_appx} \includegraphics[width=0.2\columnwidth]{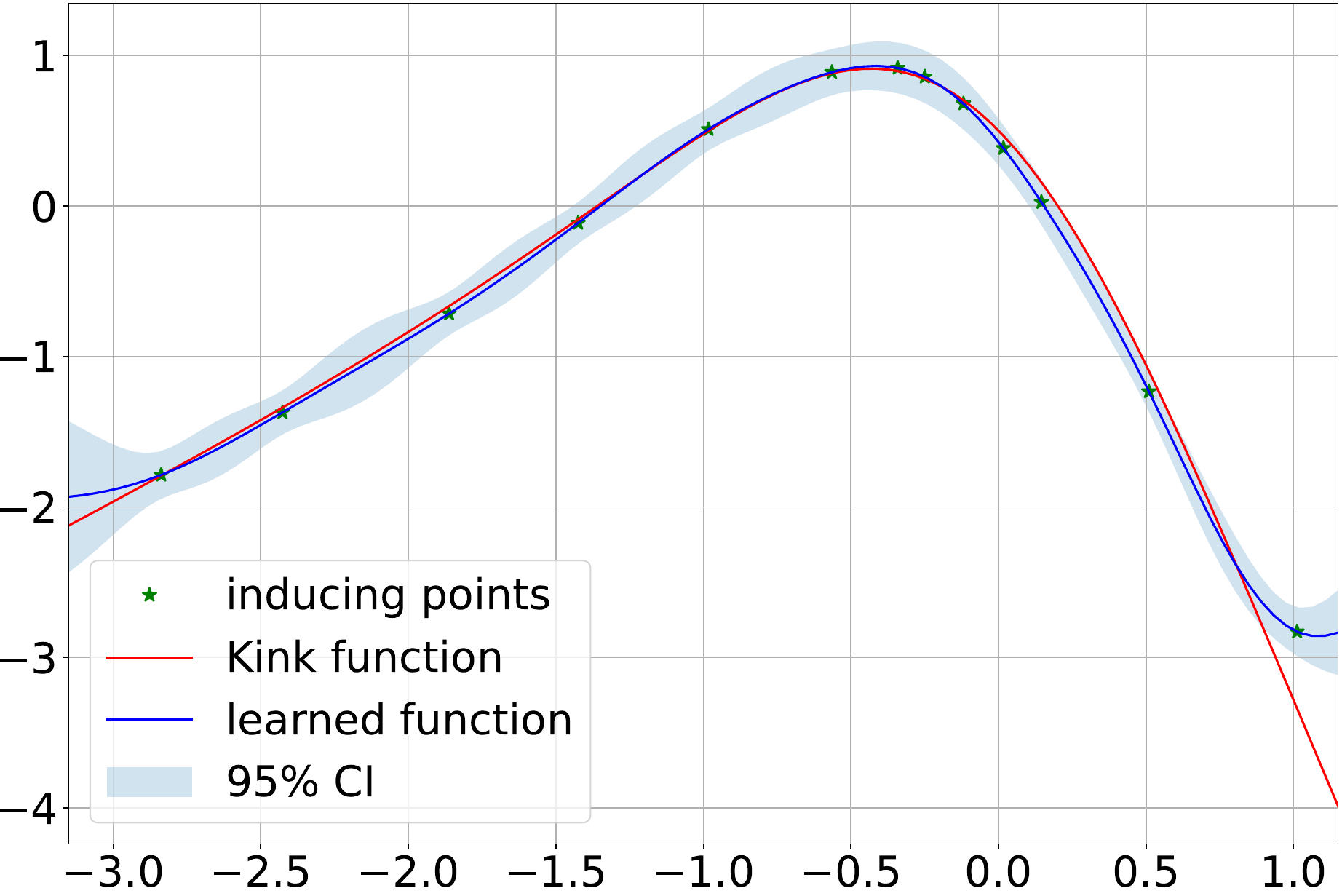}} 
	\subfloat[CO-GPSSM (MSE: {0.0410}) ]{\label{fig:cogp_kinkfunc_appx} \includegraphics[width=0.2\columnwidth]{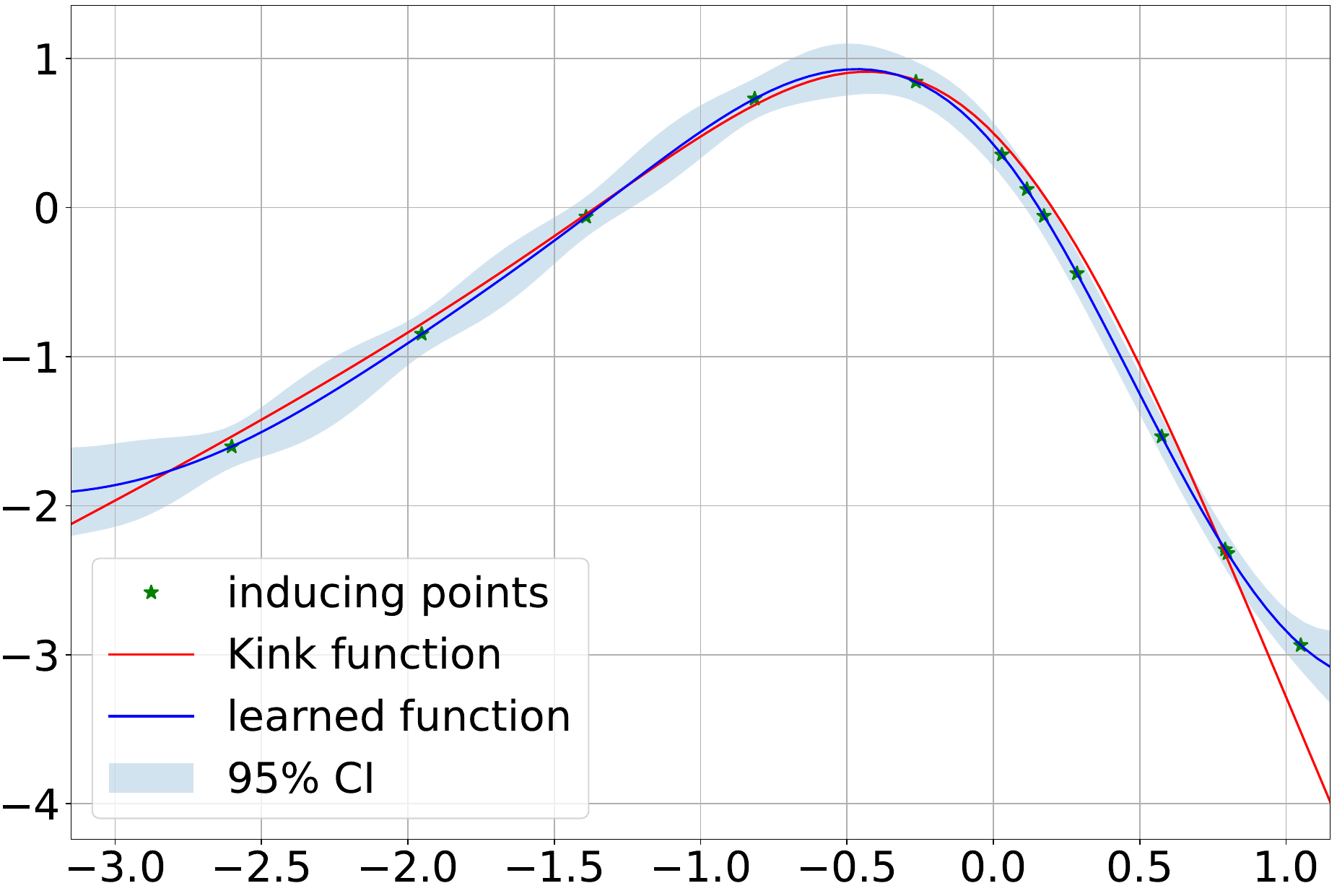}}
	\subfloat[JO-TGPSSM (MSE: 0.0361)]{\label{fig:tgp_kinkfunc_appx} \includegraphics[width=0.2\columnwidth]{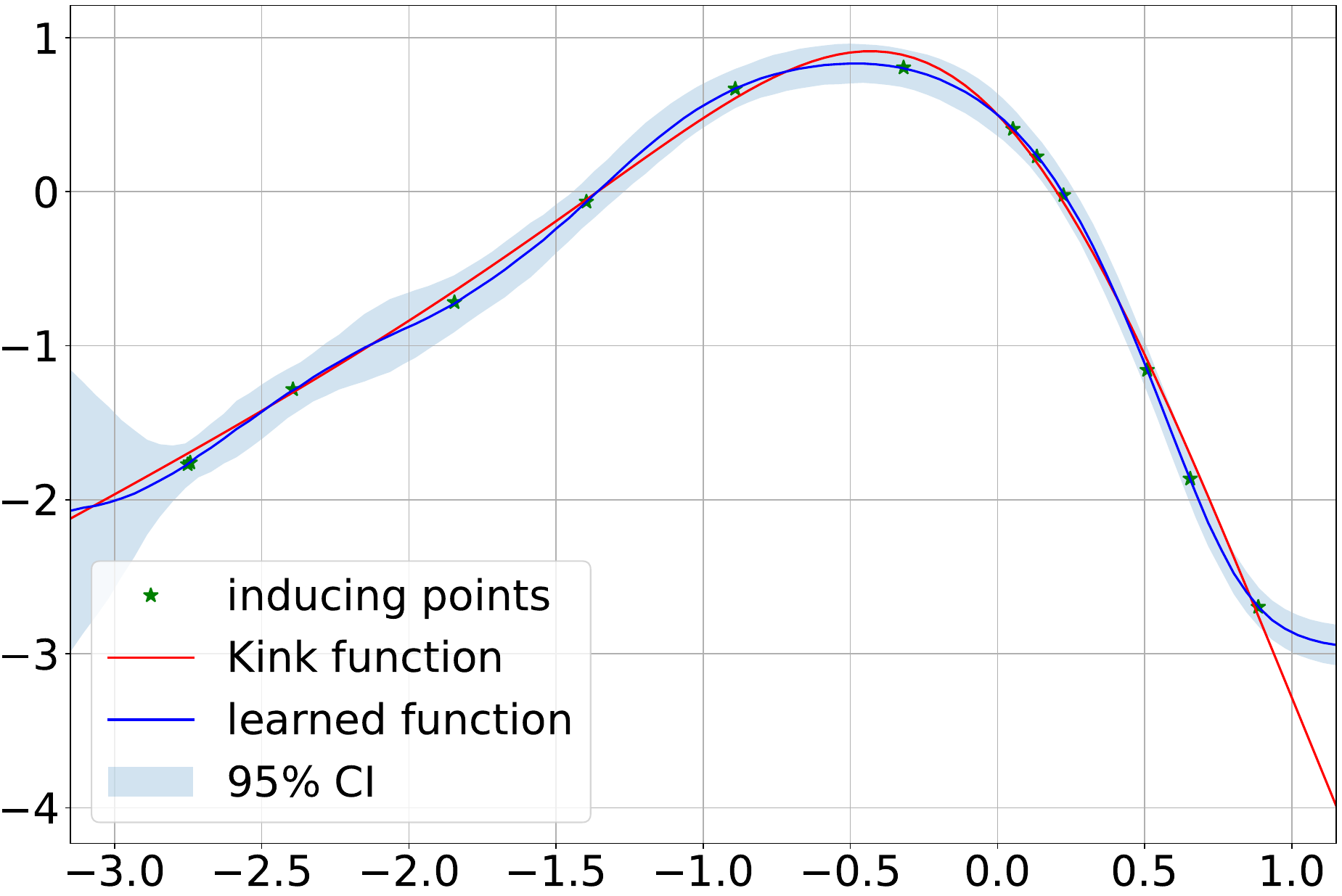}}
 
	\subfloat[CO-TGPSSM (MSE: {0.0351}) ]{\label{fig:cotgp_kinkfunc_appx} \includegraphics[width=0.2\columnwidth]{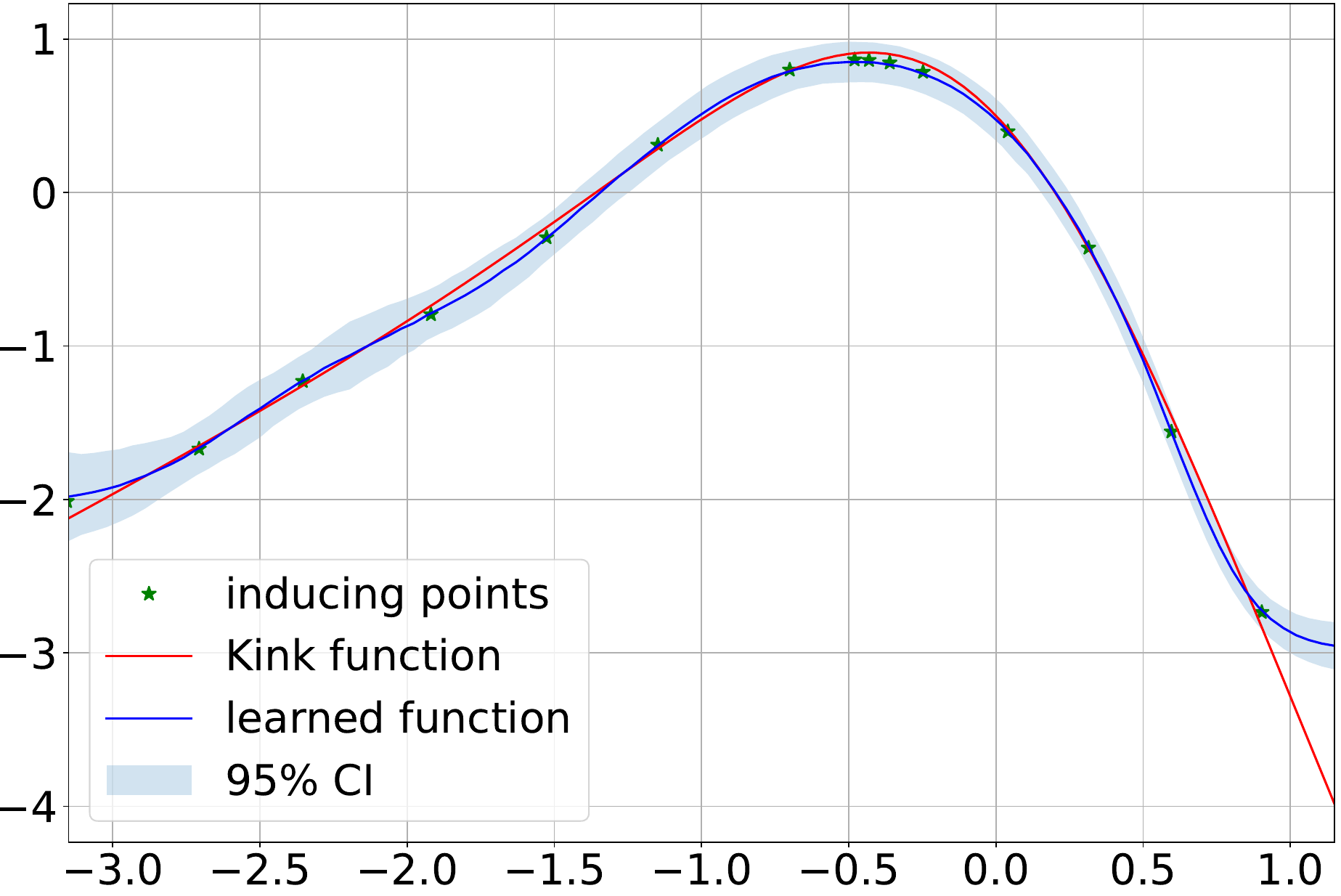}}
	\subfloat[JO-GPSSM (MSE: 2.2313)]{\label{fig:gp_ksfunc_appx} \includegraphics[width=0.2\columnwidth]{figs/JOGPSSM_ksfunc}} 
	\subfloat[CO-GPSSM (MSE: {0.3537}) ]{\label{fig:cogp_ksfunc_appx} \includegraphics[width=0.2\columnwidth]{figs/COGPSSM_ksfunc}}
	\subfloat[JO-TGPSSM (MSE: 0.4049)]{\label{fig:tgp_ksfunc_appx} \includegraphics[width=0.2\columnwidth]{figs/JOTGPSSM_ksfunc}}
	\subfloat[CO-TGPSSM (MSE: {0.2319}) ]{\label{fig:cotgp_ksfunc_appx} \includegraphics[width=0.2\columnwidth]{figs/COTGPSSM_ksfunc}}
	\caption{Learning the ``kink'' and ``kink-step'' dynamical systems using GPSSMs and TGPSSMs.}
	\label{fig:1Ddataresults_SSMs_appx}
\end{figure*}

\begin{table}[!th]
	\centering
	\caption{Details of the system identification datasets} 
	\begin{tabular}{r c  c l}
		\toprule
		{Dataset} & {Control Input} & {Observations} & {Length}  \\
		\midrule
		\textbf{Actuator} &  $ \bm{c}_t  \in \mathbb{R}$ &  $ \y_t  \in \mathbb{R}$ &  $T = 1024$\\
		\textbf{Ball Beam} &  $ \bm{c}_t  \in \mathbb{R}$ &  $ \y_t  \in \mathbb{R}$ &  $T = 1000$ \\
		\textbf{Drive}  &  $ \bm{c}_t  \in \mathbb{R}$ &  $ \y_t  \in \mathbb{R}$ &  $T = 500$ \\
		\textbf{Dryer}  &  $ \bm{c}_t  \in \mathbb{R}$ &  $ \y_t  \in \mathbb{R}$ &  $T = 1000$  \\
		\textbf{Gas Furnace} &  $ \bm{c}_t  \in \mathbb{R}$ &  $ \y_t  \in \mathbb{R}$ &  $T = 296$ \\
		\bottomrule
	\end{tabular}
	\label{tab:SIDdataset}
\end{table}

\begin{figure*}[t!]
    \centering		
    \subfloat[Actuator.]{\label{fig:actuator}\includegraphics[width =0.2\textwidth]{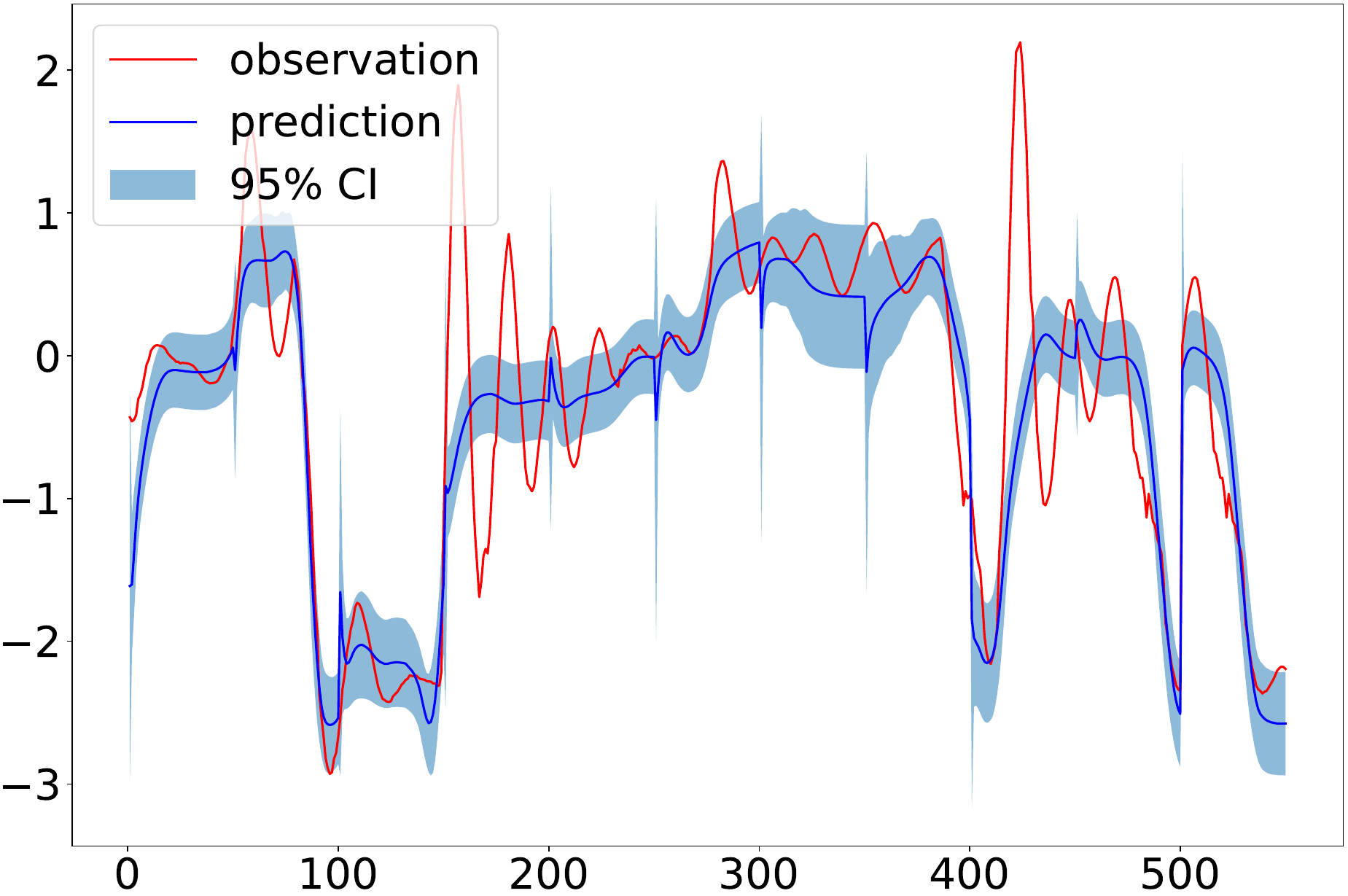}}  
    \subfloat[Ballbeam.]{\label{fig:ballbeam}\includegraphics[width =0.2\textwidth]{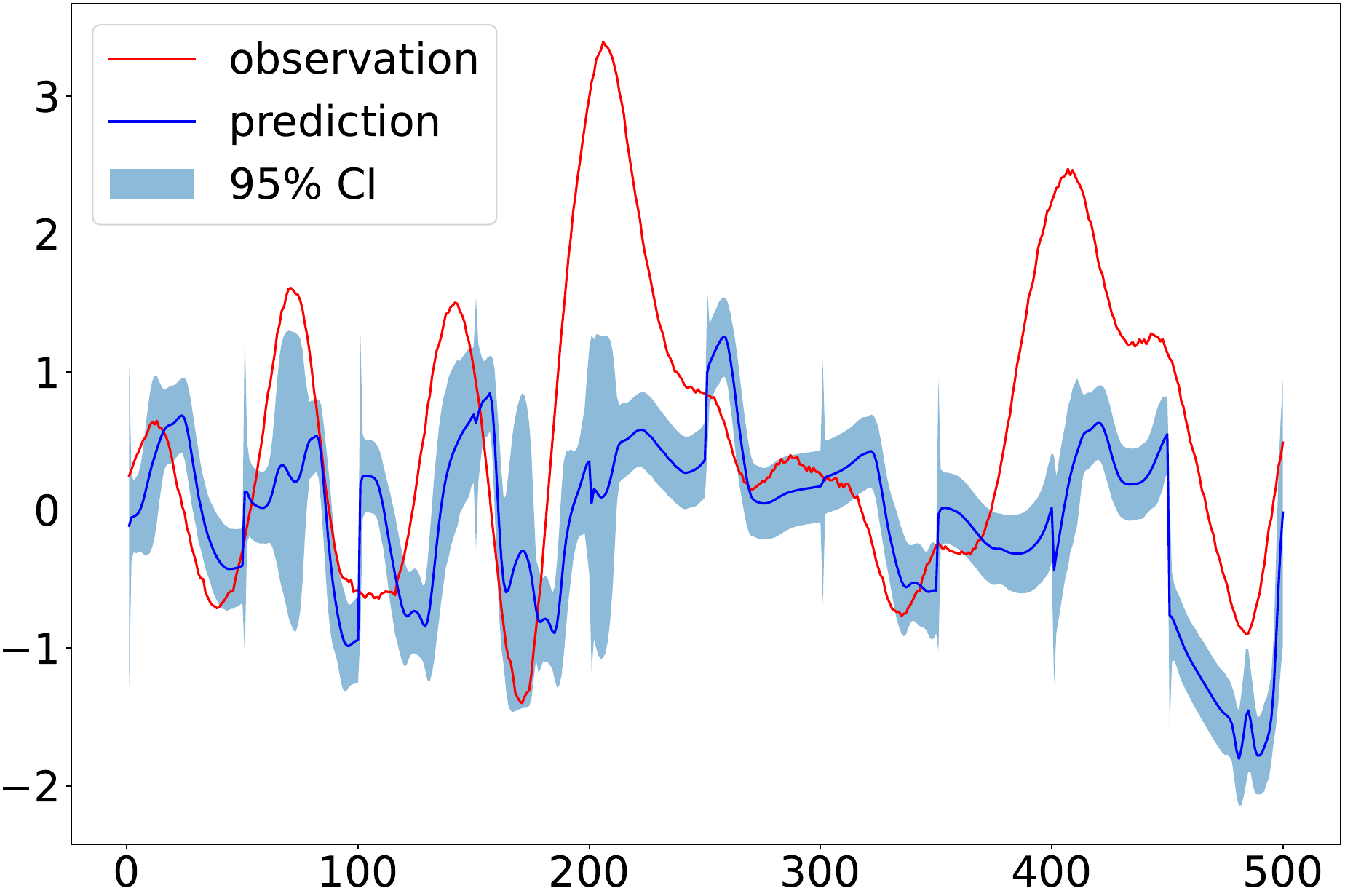}}
    \subfloat[Drive.]{\label{fig:drive}\includegraphics[width =0.2\textwidth]{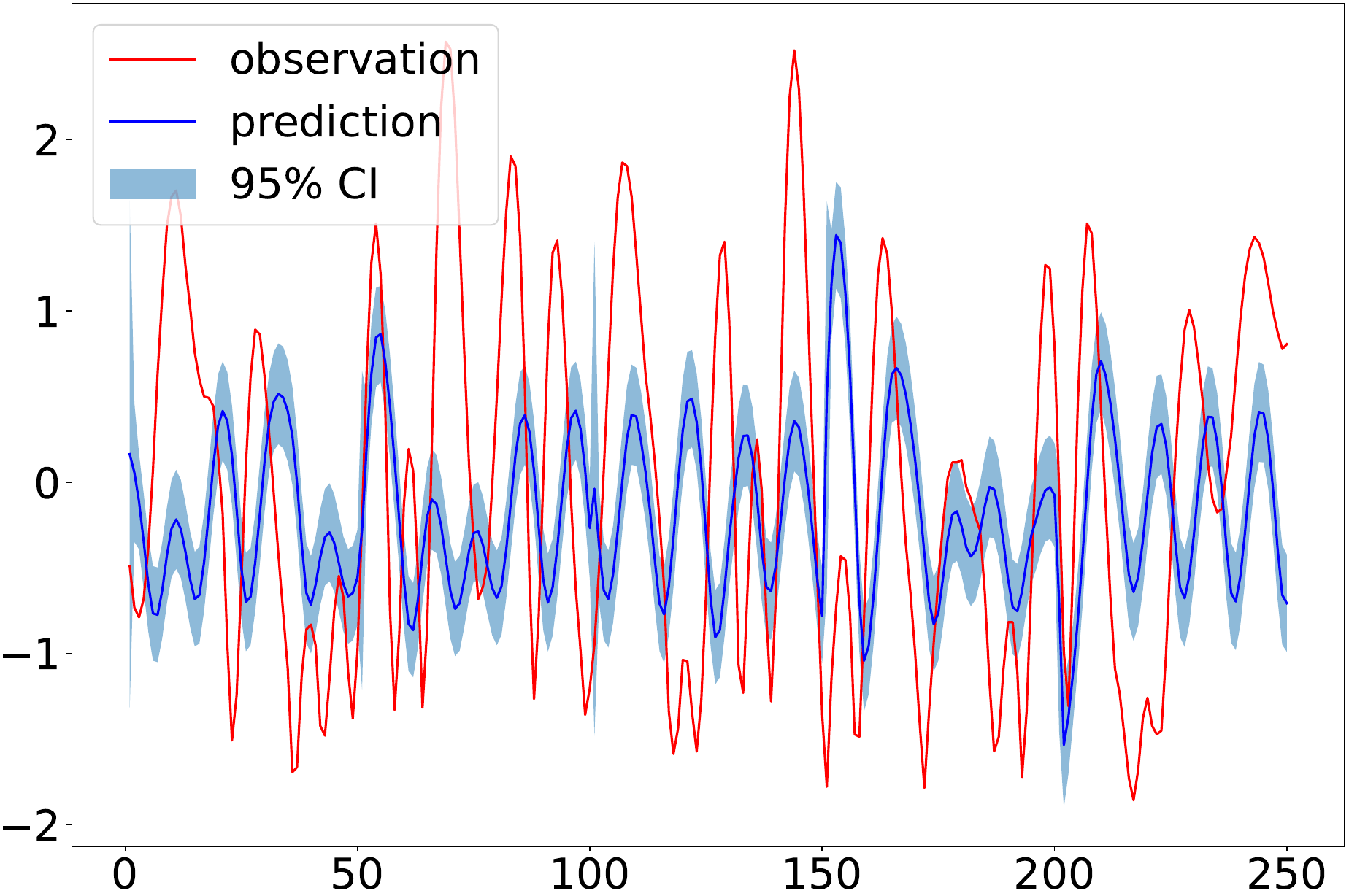}}  
    \subfloat[Dryer.]{\label{fig:dryer}\includegraphics[width =0.2\textwidth]{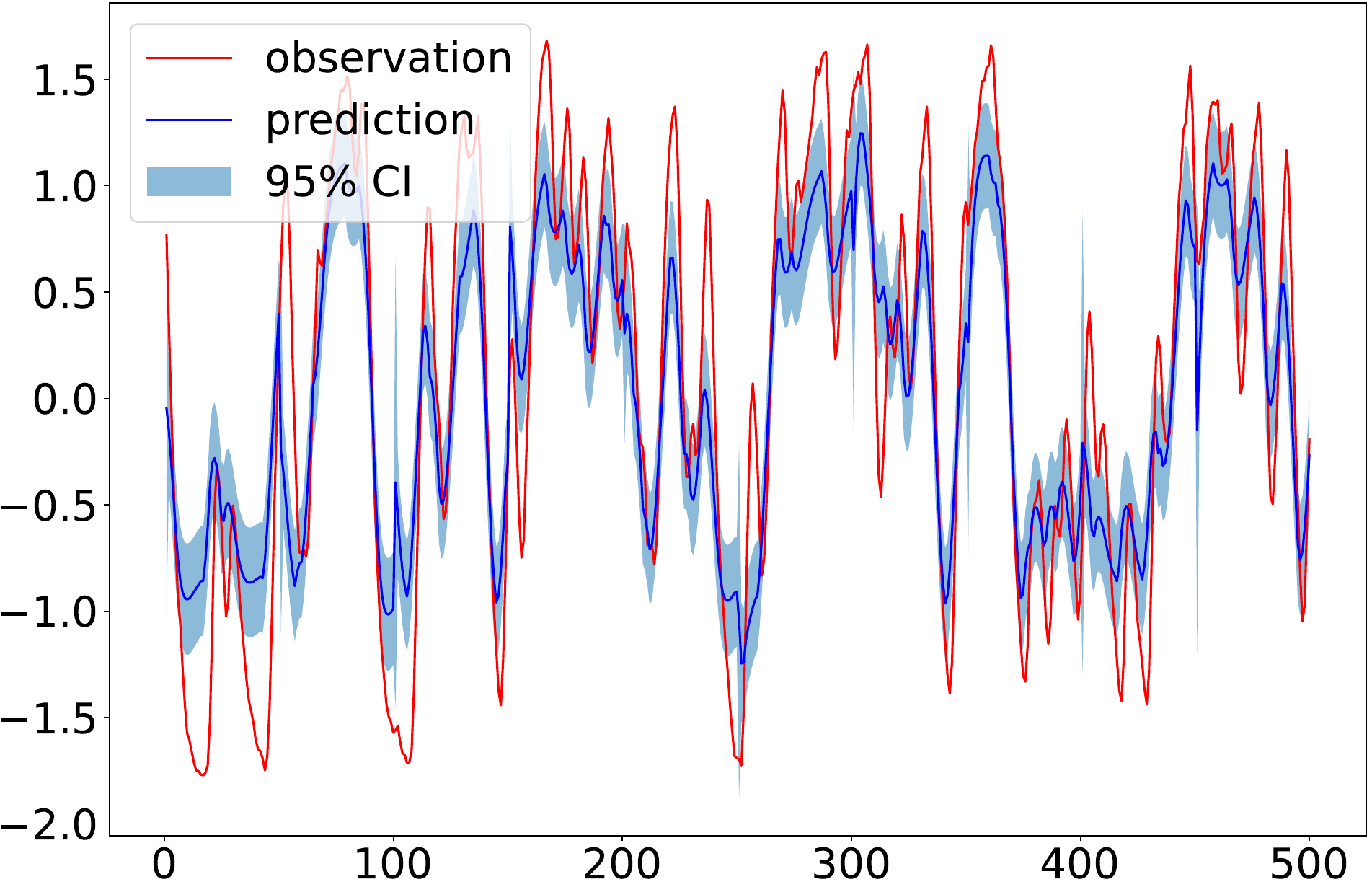}}
    \subfloat[Gas Furnace.]{\label{fig:gasfurnace}\includegraphics[width =0.2\textwidth]{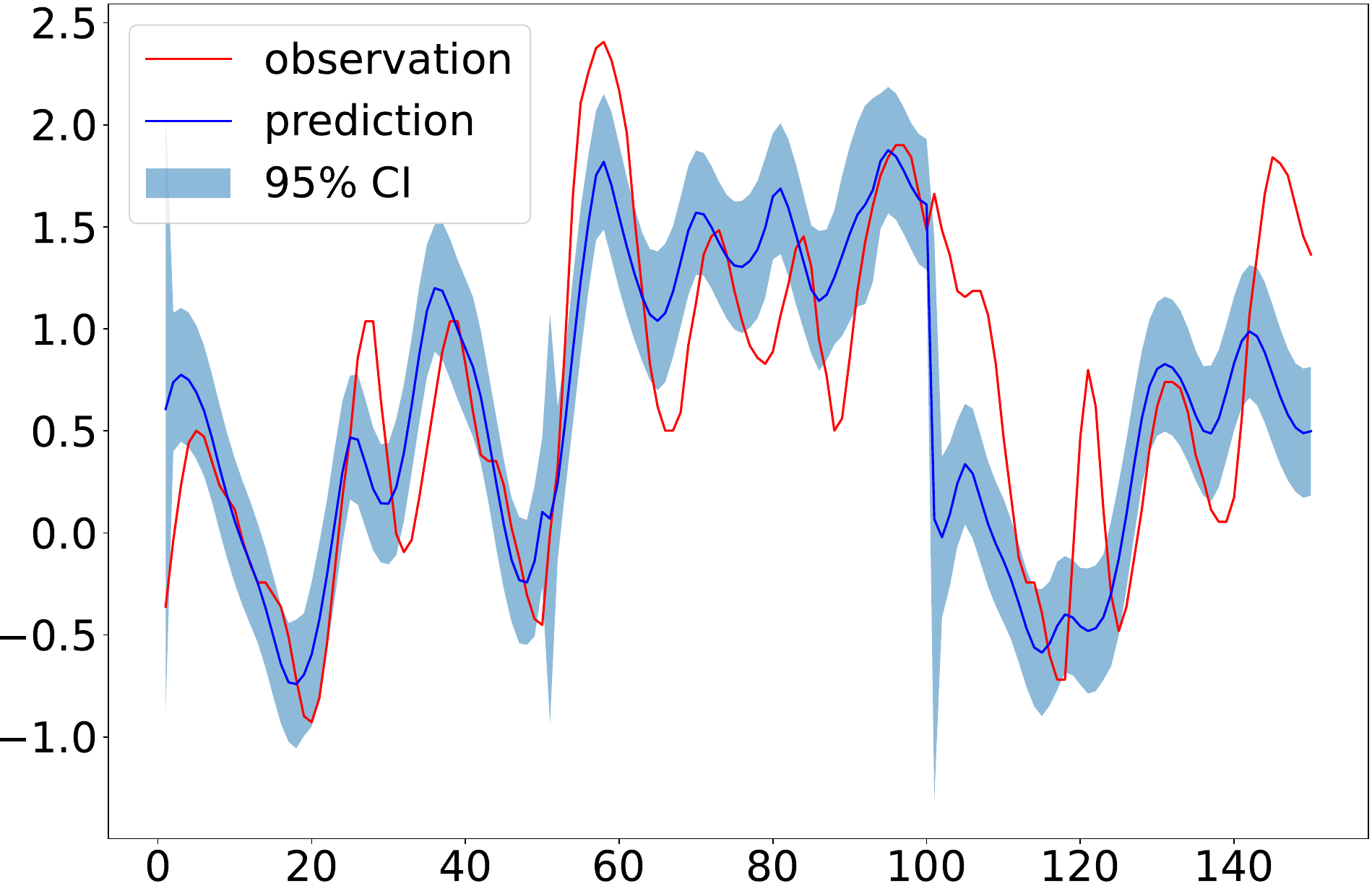}}   
    \caption{System identification datasets: Time series prediction using CO-TGPSSM}
    \label{fig:sysID_dataset_appx}
\end{figure*}

\begin{figure}[t!]
	\centering		
        \subfloat[Ground truth]{\label{fig:latentstate_GT} \includegraphics[width=0.25\textwidth]{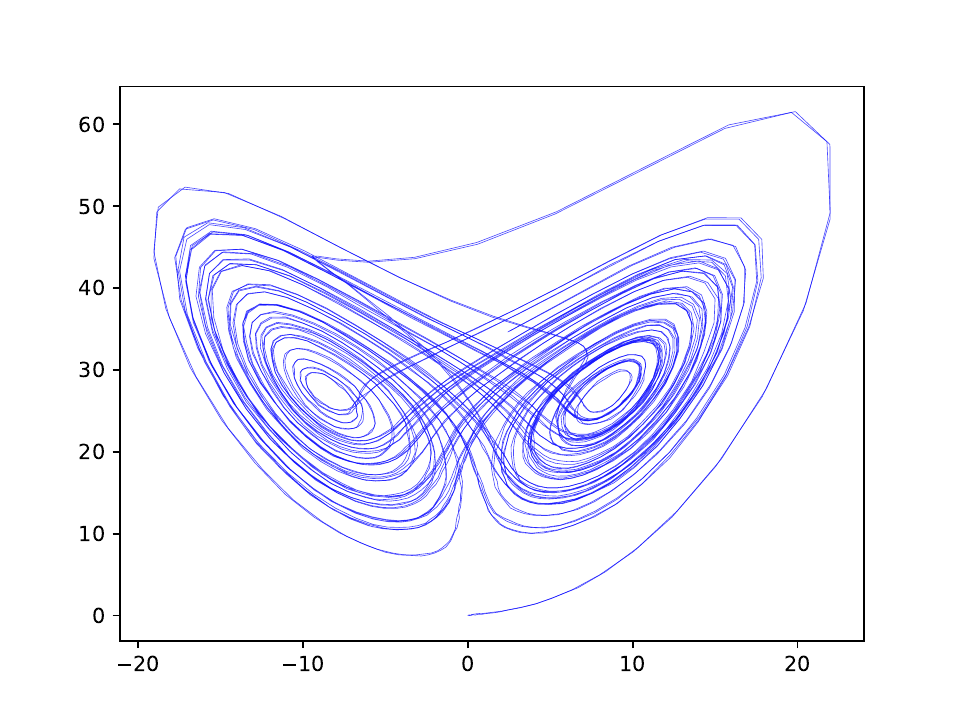}}
	\subfloat[EKF ]{\label{fig:EFK} \includegraphics[width=0.25\textwidth]{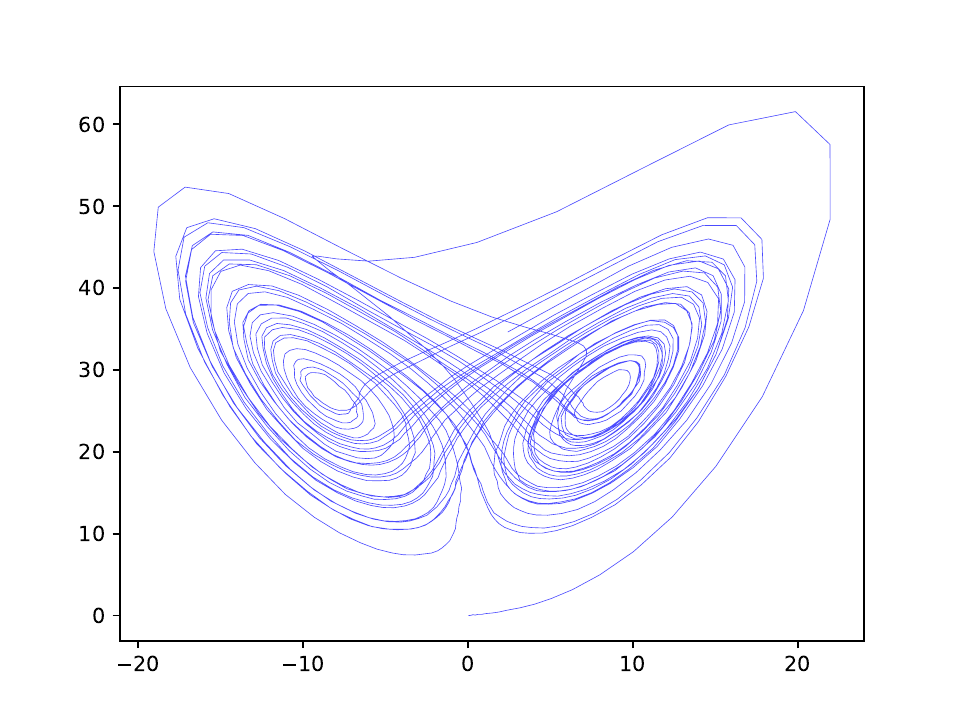}}  
        \subfloat[EKF-M ]{\label{fig:EFK_mismatch} \includegraphics[width=0.25\textwidth]{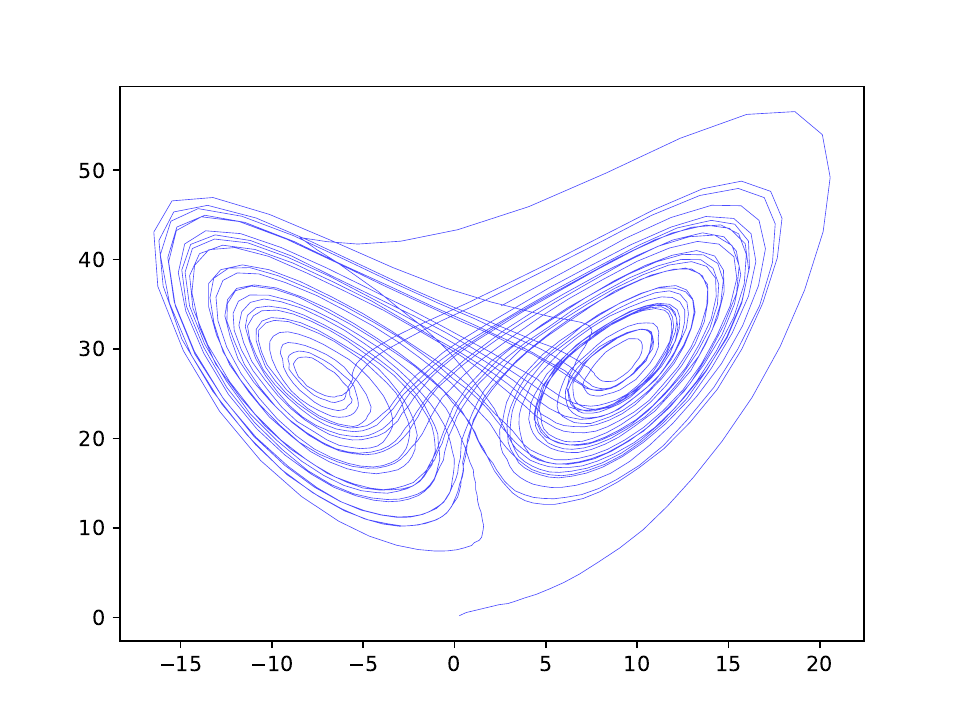}} 
        \vspace{-.1in}
        
        \subfloat[ Observations]{\label{fig:obser} \includegraphics[width=0.25\textwidth]{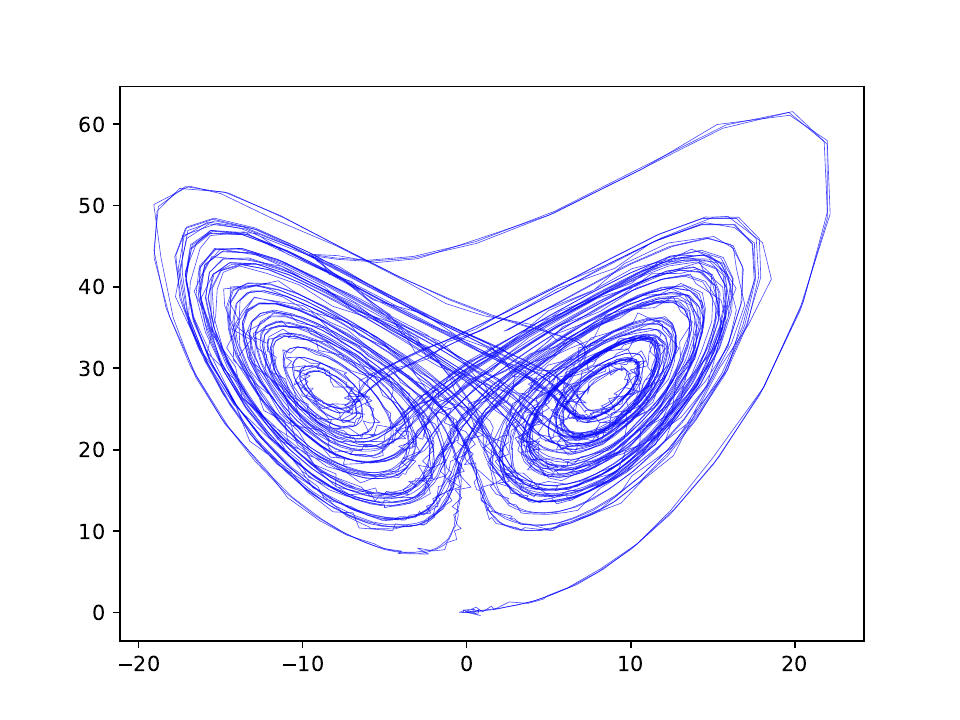}}   
	\subfloat[CO-TGPSSM]{\label{fig:COTGP} \includegraphics[width=0.25\textwidth]{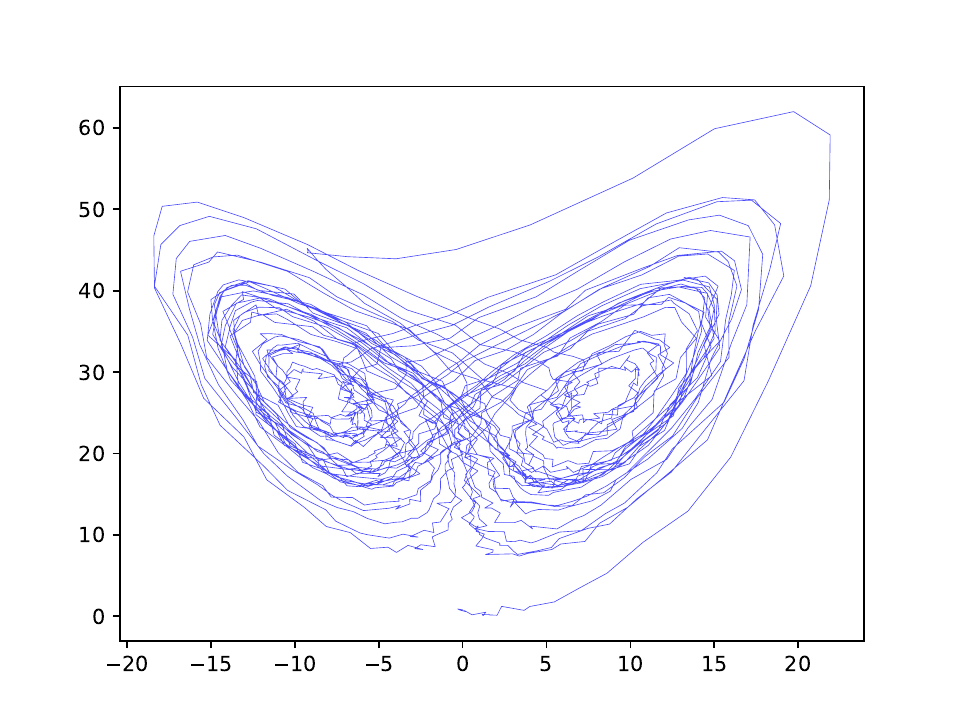}}
	\subfloat[CO-GPSSM]{\label{fig:COGP} \includegraphics[width=0.25\textwidth]{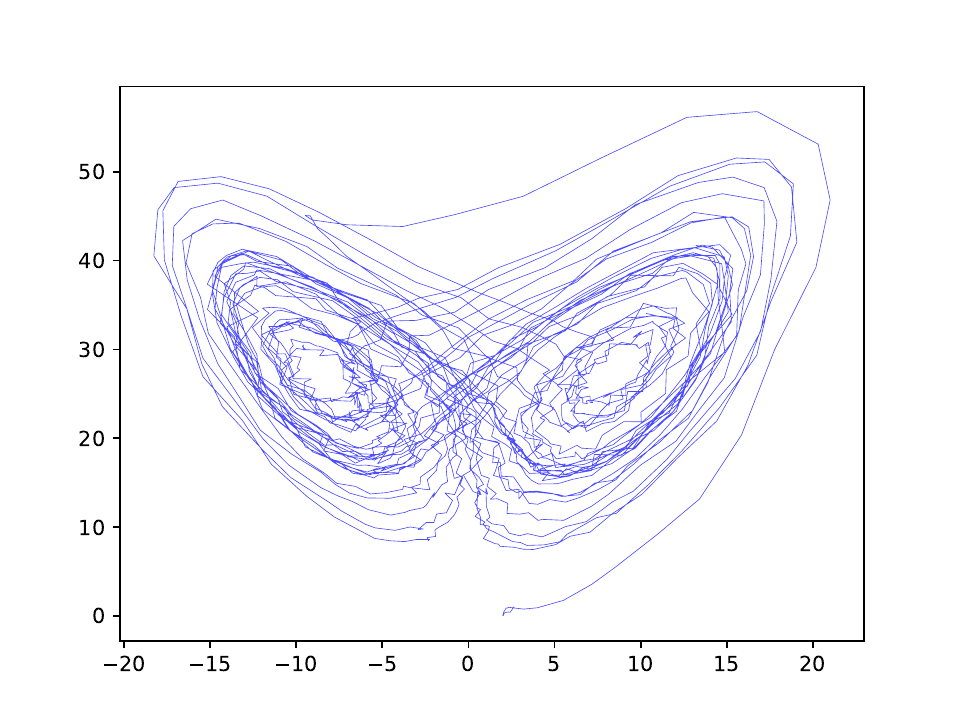}}
	\caption{Comparisons of the inferred state trajectory ($T=2000$). The 2-D plots show the state 1st dimension vs. 3rd dimension.}
	\label{fig:Lorenz_dataresults_SSMs}
\end{figure}

\clearpage
\section{Some Useful Derivations and Materials} 
\label{appendix_II}
\subsection{LOTUS Rule \cite{papamakarios2021normalizing}}
\label{subsec:LOTUS}
The law of the unconscious statistician (LOTUS) is given by 
\begin{equation}
	\E_{p_{\tilde{\x}}} [h(\tilde{\x})] = \E_{p_{\x}} [h(\G(\x))].
	\label{eq:LOTUS}
\end{equation}
That is, expectations w.r.t. the transformed density $p_{\tilde{\x}}$ can be computed without explicitly knowing $p_{\tilde{\x}}$, thus saving the computation of the determinant Jacobian terms. 

\subsection{Variational Expectation} \label{appd:reconst_error}
Assume $p(\x) = \cN(\x \mid \bm{\mu}, V)$,  $q(\x) \sim \cN(\x \mid \bm{m}, \Sigma)$, the expectation $E_{q(\x)} [\log p(\x)]$ has close-form solution and is given by 
\begin{equation}
	\begin{aligned}
		& E_{q(\x)} [\log p(\x)]  = E_{q(\x)} \left[\ln \left(\frac{1}{\sqrt{(2 \pi)^{n}|V|}} \cdot \exp \left[-\frac{1}{2}(\x-\bm{\mu})^{\mathrm{T}} V^{-1}(\x-\bm{\mu})\right]\right)\right]\\
		& = -\frac{1}{2}E_{q(\x)}\left[n \ln (2 \pi) + \ln |V|+(\x-\bm{\mu})^{\mathrm{T}} V^{-1}(\x-\bm{\mu})\right]\\
		& = -\frac{1}{2}E_{q(\x)}\left[n \ln (2 \pi) + \ln |V| + \operatorname{tr}(V^{-1}\x\x^{\mathrm{T}} ) - 2\bm{\mu}^{\mathrm{T}} V^{-1} \x + \bm{\mu}^{T} V^{-1} \bm{\mu} \right]\\
		& =  -\frac{1}{2} \left[n \ln (2 \pi) + \ln |V| + \operatorname{tr}\left(V^{-1} (\bm{m}\bm{m}^{\mathrm{T}}  + \Sigma)\right) - 2\bm{\mu}^{\mathrm{T}} V^{-1} \bm{m} + \bm{\mu}^{T} V^{-1} \bm{\mu} \right]	
	\end{aligned}
\end{equation}

\subsection{Differential Entropy of Multivariate Normal Distribution}\label{appendx:entropy}
\begin{theorem}\label{thm:entropy}
	Let $\mathbf{x}$ follow a multivariate normal distribution
	$$
	\x \sim \mathcal{N}(\bm{\mu}, \Sigma), \quad  \x \in \mathbb{R}^n
	$$
	Then, the differential entropy of $\x$ in nats is
	$$
	\mathrm{H}(\x)\triangleq -\int_{\x} p(\x) \ln p(\x)  \mathrm{d}  \x =\frac{n}{2} \ln (2 \pi)+\frac{1}{2} \ln |\Sigma|+\frac{1}{2} n 
	$$
\end{theorem}
\begin{proof}
	$$
	\begin{aligned}
		\mathrm{H}(\x) &=-\mathrm{E}\left[\ln \left(\frac{1}{\sqrt{(2 \pi)^{n}|\Sigma|}} \cdot \exp \left[-\frac{1}{2}(\x-\bm{\mu})^{\mathrm{T}} \Sigma^{-1}(\x-\bm{\mu})\right]\right)\right] \\
		&=-\mathrm{E}\left[-\frac{n}{2} \ln (2 \pi)-\frac{1}{2} \ln |\Sigma|-\frac{1}{2}(\x-\bm{\mu})^{\mathrm{T}} \Sigma^{-1}(\x-\bm{\mu})\right] \\
		&=\frac{n}{2} \ln (2 \pi)+\frac{1}{2} \ln |\Sigma|+\frac{1}{2} \mathrm{E}\left[(\x-\bm{\mu})^{\mathrm{T}} \Sigma^{-1}(\x-\bm{\mu})\right]
	\end{aligned}
	$$
	The last term can be evaluted as
	$$
	\begin{aligned}
		\mathrm{E}\left[(\x-\bm{\mu})^{\mathrm{T}} \Sigma^{-1}(\x-\bm{\mu})\right] &=\mathrm{E}\left[\operatorname{tr}\left((\x-\bm{\mu})^{\mathrm{T}} \Sigma^{-1}(\x-\bm{\mu})\right)\right] \\
		&=\mathrm{E}\left[\operatorname{tr}\left(\Sigma^{-1}(\x-\bm{\mu})(\x-\bm{\mu})^{\mathrm{T}}\right)\right] \\
		&=\operatorname{tr}\left(\Sigma^{-1} \mathrm{E}\left[(\x-\bm{\mu})(\x-\bm{\mu})^{\mathrm{T}}\right]\right) \\
		&=\operatorname{tr}\left(\Sigma^{-1} \Sigma\right) \\
		&=\operatorname{tr}\left(I_{n}\right) \\
		&=n
	\end{aligned}
	$$
	Therefore the differential entropy is
	$$
	\mathrm{H}(\x)=\frac{n}{2} \ln (2 \pi)+\frac{1}{2} \ln |\Sigma|+\frac{1}{2} n .
	$$
\end{proof}

\subsection{Expected  Log-Gaussian Likelihood}\label{appedx:elgl}
\begin{theorem}
	Let $q(\f) = \mathcal{N}(\f \mid \bm{\mu}, \Sigma)$, and the likelihood $p(\y \mid \f) = \mathcal{N}(\y \mid C\f,   \bm{Q})$, where $C$ is a constant matrix, then we have 
	$$
		\E_{q(\f)}\left[\log p(\y \mid \f)\right] = \log \mathcal{N}(\y \mid C\bm{\mu}, \bm{Q}) - \frac{1}{2}\operatorname{tr}(\bm{Q}^{-1} C \Sigma C^\top) 
	$$
\end{theorem}
\begin{proof}
	Let $\y \in \mathbb{R}^n$, we have
	$$
	\begin{aligned}
		\log p(\y \mid \f) & =  \log \left[ (2\pi)^{-n/2} \det(\bm{Q})^{-1/2} \exp\left(-\frac{(\y - C\f)^\top \bm{Q}^{-1} (\y - C\f)}{2} \right)  \right]\\
		&  = -\frac{n}{2}\log 2\pi - \frac{1}{2}\log \det(\bm{Q}) -\frac{(\y - C\f)^\top \bm{Q}^{-1} (\y - C\f)}{2}\\
		& = -\frac{\operatorname{tr}(C^\top \bm{Q}^{-1} C \f\f^\top ) - 2 \y^\top \bm{Q}^{-1} C \f + \y^\top \bm{Q}^{-1}\y}{2}-\frac{n}{2}\log 2\pi - \frac{1}{2}\log \det(\bm{Q})
	\end{aligned}
	$$
	Therefore, we have the expected value
	$$
	\begin{aligned}
		\E_{q(\f)}\left[\log p(\y \mid \f)\right] &=  -\frac{\operatorname{tr}\left( C^\top \bm{Q}^{-1} C   \E_{q(\f)} \left[ \f\f^\top \right] \right) - 2 \y^\top \bm{Q}^{-1} C \bm{\mu} + \y^\top \bm{Q}^{-1}\y}{2}-\frac{n}{2}\log 2\pi - \frac{1}{2}\log \det(\bm{Q})\\
		& = -\frac{\operatorname{tr}\left[  C^\top \bm{Q}^{-1} C  (\bm{\mu}\bm{\mu}^\top + \Sigma)   \right] - 2 \y^\top \bm{Q}^{-1} C \bm{\mu} + \y^\top \bm{Q}^{-1}\y}{2}-\frac{n}{2}\log 2\pi - \frac{1}{2}\log \det(\bm{Q})\\
		& = -\frac{(\y-C\bm{\mu})^\top \bm{Q}^{-1}(\y-C\bm{\mu}) }{2} -\frac{n}{2}\log 2\pi - \frac{1}{2}\log \det(\bm{Q}) - \frac{1}{2}\operatorname{tr}(\bm{Q}^{-1} C \Sigma C^\top)\\
		& = \log \mathcal{N}(\y \mid C\bm{\mu}, \bm{Q}) - \frac{1}{2}\operatorname{tr}(\bm{Q}^{-1} C \Sigma C^\top)
	\end{aligned}
	$$
\end{proof}
\subsection{Marginal Variational Distribution in Sparse Variational GP \cite{titsias2009variational}}
Suppose that the variational distributions 
$$
q(\f, \u) = q(\u) \ p(\f \mid \u), \ \ q(\u) = \mathcal{N}(\u \mid \bm{\mu}, \Sigma)
$$
and $$
p(\f \mid \u) = \cN(\f \mid \bm{m}_f + \bm{K}_{fu}\bm{K}_{uu}^{-1}(\u - \bm{m}_u) ,  \  \bm{K}_{ff} - \bm{K}_{fu}\bm{K}_{uu}^{-1} \bm{K}_{fu}^\top).
$$
Using the following Lemma \ref{lemma:joint_gaussian}, we have 
$$
\begin{aligned}
	q(\f) &= \int_u q(\u) p(\f\mid \u)  \mathrm{d}  \u\\
	& = \cN\left(\f \mid \bm{m}_f + \bm{K}_{fu}\bm{K}_{uu}^{-1}(\bm{\mu}- \bm{m}_u), \ \bm{K}_{fu}\bm{K}_{uu}^{-1} \Sigma \bm{K}_{uu}^{-1} \bm{K}_{fu}^\top  + \bm{K}_{ff} - \bm{K}_{fu}\bm{K}_{uu}^{-1} \bm{K}_{fu}^\top  \right)\\
	& =  \cN\left[\f \mid \bm{m}_f + \bm{K}_{fu}\bm{K}_{uu}^{-1}(\bm{\mu}- \bm{m}_u), \ 
	\bm{K}_{ff}  - 
	\bm{K}_{fu}\bm{K}_{uu}^{-1} ( \bm{K}_{uu} - \Sigma) \bm{K}_{uu}^{-1} \bm{K}_{fu}^\top   \right]\\
\end{aligned}
$$
\begin{lemma}
	\label{lemma:joint_gaussian}
	(Joint distribution of Gaussian variables \cite{sarkka2013bayesian}) If random variables $\mathbf{x} \in \mathbb{R}^{n}$ and $\mathbf{y} \vert \x \in \mathbb{R}^{m}$ have the Gaussian distributions
	$$
	\begin{aligned}
		\mathbf{x} & \sim \mathrm{N}(\mathbf{m}, \mathbf{P}), \\
		\mathbf{y} \mid \mathbf{x} & \sim \mathrm{N}(\mathbf{H} \mathbf{x}+\mathbf{u}, \mathbf{R}),
	\end{aligned}
	$$
	then the joint distribution of $\mathbf{x}, \mathbf{y}$ and the marginal distribution of $\mathbf{y}$ are given as
	$$
	\begin{aligned}
		\left(\begin{array}{l}
			\mathbf{x} \\
			\mathbf{y}
		\end{array}\right) & \sim \mathrm{N}\left(\left(\begin{array}{c}
			\mathbf{m} \\
			\mathbf{H} \mathbf{m}+\mathbf{u}
		\end{array}\right),\left(\begin{array}{cc}
			\mathbf{P} & \mathbf{P} \mathbf{H}^{\top} \\
			\mathbf{H} \mathbf{P} & \mathbf{H} \mathbf{P} \mathbf{H}^{\top}+\mathbf{R}
		\end{array}\right)\right), \\
		\mathbf{y} & \sim \mathrm{N}\left(\mathbf{H} \mathbf{m}+\mathbf{u}, \mathbf{H} \mathbf{P} \mathbf{H}^{\top}+\mathbf{R}\right) .
	\end{aligned}
	$$ \qed
\end{lemma} 

\subsection{KL Divergence Between Two Multivariate Gaussian Distributions} \label{appedx:KL}
For the two Gaussian distributions $P = \cN(\x \mid \bm{\mu}_1, \Sigma_1)$ and $Q = \cN(\x \mid  \bm{\mu}_2, \Sigma_2)$, their KL divergence is given by
$$
\begin{aligned}
	\mathrm{KL}\left(P|| Q\right)= & \int P(\x) \log \frac{P(\x)}{Q(\x)} \mathrm{d} \x  \\
	= & \frac{1}{2}\left[\left(\bm{\mu}_{2}-\bm{\mu}_{1}\right)^{\top} \Sigma_{2}^{-1}\left(\bm{\mu}_{2}-\bm{\mu}_{1}\right)+\operatorname{tr}\left(\Sigma_{2}^{-1} \Sigma_{1}\right)-\ln \frac{\left|\Sigma_{1}\right|}{\left|\Sigma_{2}\right|}-n\right].
\end{aligned}
$$
Proof can be referred to, e.g.,  \url{https://stanford.edu/~jduchi/projects/general_notes.pdf}

\end{document}